\documentclass{article}

    \usepackage[preprint, nonatbib]{neurips_2023}

\usepackage[utf8]{inputenc} 
\usepackage[T1]{fontenc}    
\usepackage{hyperref}       
\usepackage{url}            
\usepackage{booktabs}       
\usepackage{amsfonts}       
\usepackage{nicefrac}       
\usepackage{microtype}      
\usepackage{xcolor}         

\usepackage{subcaption}  
\usepackage{bm}
\usepackage{amssymb, amsthm, amsmath}
\usepackage[english]{babel}
\usepackage{appendix}
\usepackage{physics}
\usepackage{enumitem}
\usepackage{graphicx}
\usepackage{wrapfig}
\usepackage{thm-restate}

\DeclareMathOperator{\Var}{Var}

\theoremstyle{definition}

\newtheorem{theorem}{Theorem}[section]

\newtheorem{remark}{Remark}[section]
\newtheorem{lemma}{Lemma}[section]

\newcommand{\graphgpt}{GraphART}
\newcommand{\graphrnn}{GraphLSTM}

\newcommand{\ehd}{\text{EHD}}

\title{Let There Be Order: Rethinking Ordering in Autoregressive Graph Generation}

\author{%
  Jie Bu \\
  Department of Computer Science\\
  Virginia Tech\\
  Blacksburg, VA 24060 \\
  \texttt{jayroxis@vt.edu} \\
  \And
  Kazi Sajeed Mehrab \\
  Department of Computer Science\\
  Virginia Tech\\
  Blacksburg, VA 24060 \\
  \texttt{ksmehrab@vt.edu} \\
  \And
  Anuj Karpatne \\
  Department of Computer Science\\
  Virginia Tech\\
  Blacksburg, VA 24060 \\
  \texttt{karpatne@vt.edu} \\
}

\begin{document}

\maketitle

\begin{abstract}
Conditional graph generation tasks involve training a model to generate a graph given a set of input conditions. Many previous studies employ autoregressive models to incrementally generate graph components such as nodes and edges. However, as graphs typically lack a natural ordering among their components, converting a graph into a sequence of tokens is not straightforward. While prior works mostly rely on conventional heuristics or graph traversal methods like breadth-first search (BFS) or depth-first search (DFS) to convert graphs to sequences, the impact of ordering on graph generation has largely been unexplored. This paper contributes to this problem by: (1) highlighting the crucial role of ordering in autoregressive graph generation models, (2) proposing a novel theoretical framework that perceives ordering as a dimensionality reduction problem, thereby facilitating a deeper understanding of the relationship between orderings and generated graph accuracy, and (3) introducing "latent sort," a learning-based ordering scheme to perform dimensionality reduction of  graph tokens. Our experimental results showcase the effectiveness of latent sort across a wide range of graph generation tasks, encouraging future works to further explore and develop learning-based ordering schemes for autoregressive graph generation.
\end{abstract}

\section{Introduction}

We consider the problem of generating graphs given conditional inputs, where there exists precisely one target graph we are interested in generating for every conditional input. We call this problem as \textit{paired} graph generation,  since every input condition is paired with exactly one target graph. This setting is highly relevant in many real-world applications such as predicting the road network of a region given its satellite image \cite{belli2019image}, generating scene graphs with semantic relationships among objects given an image of the objects \cite{lu2021context, yang2022psg, li2022sgtr}, and inferring the underlying circuit graph of an electrical instrument given its sensor observations \cite{he2019circuit}. This is different from the \textit{unpaired} setting of graph generation problems studied by several previous works, where we are interested in generating distributions of graphs rather than a specific target graph.


The problem of graph generation has some similarity to natural language generation \cite{brown2020language, ouyang2022training, bubeck2023sparks}, since they both involve generating outputs with variable sizes. However, a key distinction is that there is no natural ordering of graph components (e.g., nodes and edges) as opposed to the sequential nature of words in a sentence. Graph generation is thus a particularly challenging problem due to the inherent diversity of graphs involving varying sizes and different topologies. Many previous approaches for graph generation \cite{li2018learning, you2018graphrnn, liao2019efficient, belli2019image, liu2019auto} have adopted autoregressive models such as recurrent neural networks \cite{hochreiter1997long, cho2014learning} and auto-regressive transformers \cite{vaswani2017attention, radford2018improving, radford2019language} to handle the variable sizes and topologies in graphs. These models often convert graph components into vector representations or \textit{tokens}, and incrementally generate tokens representing graph components in an ordered sequence. 

However, existing autoregressive models for graph generation face unique challenges in sorting multi-dimensional graph tokens into ordered sequences. As graph components often have no natural ordering \cite{hamilton2017inductive}, there can be $n!$ possible orderings for $n$ graph tokens. It has been shown by \cite{vinyals2015order, chen2021order, liao2019efficient} that finding an optimal ordering of unsorted elements in a set is key to the success of autoregressive models, since certain orderings may be favored by the models leading to better performance. Despite this, existing works on autoregressive graph generation \cite{you2018graphrnn, liao2019efficient, belli2019image, liu2019auto} make arbitrary choices for ordering graph components based on graph traversals such as breadth-first-search (BFS) or depth-first-search (DFS) without fully justifying their choices empirically or theoretically. 

In this work, we present a novel perspective to address this issue by reframing the problem of ordering a set of graph tokens as that of performing dimensionality reduction (DR). Specifically, by learning a generalizable mapping of every possible graph token to a 1-D representation, we theoretically and empirically show that a set of graph tokens can be conveniently ordered by sorting their 1-D values.
Our contributions can be summarized as follows:
\begin{itemize} \vspace{-0.03cm}
    \item 
    We develop novel theoretical frameworks to study the accuracy of DR techniques for sorting unordered tokens in autoregressive models, opening a new line of research in developing learning-based ordering schemes for paired graph generation. 
    \item We propose the \textit{latent sort} algorithm, a novel learning-based ordering scheme employing DR for autoregressive graph generation. We provide theoretical bounds of the errors of latent sort and highlight an  interesting connection between performing latent sort and finding approximate solutions to the shortest path problem. 
    \item We propose a strong backbone network for autoregressive graph generation named Graph Auto-Regressive Transformer (\graphgpt{}) based on a modified GPT-2  architecture \cite{radford2019language}. \graphgpt{} is purely autoregressive in nature and does not use any graph-specific architectural components such as graph neural networks (GNNs), in contrast to  prior works.
    \item We empirically show that different token ordering schemes result in nontrivial performance gaps in autoregressive models such as \graphgpt{}. We also show that Latent Sort is versatile and achieves competitive performance over a wide range of graph generation tasks. 
\end{itemize}

\section{Related Works}

A number of approaches have been developed for graph generation in a probabilistic setting where the goal is 
to generate a distribution of graphs that resemble a target distribution given conditional inputs
\cite{li2018learning, you2018graphrnn, bojchevski2018netgan, liu2018constrained, ma2018constrained, liao2019efficient, yang2019conditional, jo2022score, vignac2022digress}. These works fall in the ``unpaired'' graph generation category, since there is no unique ground-truth target graph that is paired with every input condition. Hence, the evaluation of performance of these works typically focuses on measuring the validity of the generated graphs in terms of statistical similarity with the target distribution, rather than measuring deviation from its corresponding ground-truth graph in a ``paired'' fashion, which is the focus of our work.

One line of work for paired graph generation involves using non-autoregressive approaches to generate graphs \cite{zaheer2017deep, zhang2019deep, kosiorek2020conditional, yang2022psg, li2022sgtr}. These approaches generate entire graphs rather than individual components, which are then matched to the target graphs using bipartite matching algorithms (e.g., Hungarian matching in DETR \cite{carion2020end} for object detection).
However, bipartite matching loss functions can slow down training due to their high computational costs \cite{jonker1988shortest} and can potentially suffer from slow convergence \cite{sun2021rethinking, zhang2022accelerating}. Moreover, these models cannot direclty handle output graphs with varying sizes and have to resort to padding, making them inefficient for real-world applications.

Another line of work involves autoregressive models to generate variable-length sequences of graph tokens. As a necessary preprocessing step, these methods require an approach to sequentialize target graphs into an ordered sequence of tokens, where finding the ideal scheme for token ordering can be non-trivial. Vinyals et al. \cite{vinyals2015order} were the first to reveal the importance of ordering when using sequence-to-sequence models to process sets. They proposed a training algorithm that searches for the optimal order, but scalability issues would arise with large sets, and inexact searches or sampling methods could negatively impact performance.
More recently, Chen et al. \cite{chen2021order} discuss the importance of ordering for unpaired autoregressive graph generation, where they derive the joint probability over the graph for sorting the nodes.
To sequentialize graphs, 
most of the existing works favor traditional graph traversal methods like BFS \cite{you2018graphrnn, liao2019efficient} or DFS \cite{jin2018junction, liu2019auto, liao2019efficient} . 
However, to the best of our knowledge, no prior work has theoretically and empirically compared the performance of autoregressive models for graph generation with varying token ordering schemes, which is one of the contributions of our work. 

\section{Rethinking Sorting As a Dimensionality Reduction Problem}
\label{sec:sorting_as_dr}

Given a set of $M$ unordered points (or vector tokens) $\mathcal{X} \subseteq \mathbb{R}^N$,  $|\mathcal{X}| = M$ in an $N$-dimensional space ($N > 1$), we are interested in finding the ``optimal'' ordering of points in $\mathcal{X}$, denoted by $Y^* \in \mathbb{R}^{M \times N}$, which when used as the target sequence to supervise an autoregressive model yields optimal performance in generating $\mathcal{X}$.
Formally, let us denote the ordered sequence $Y^*$ as the matrix $\left[\bm{x_1}^*, \bm{x_2}^*, ..., \bm{x_M}^* \right]^{\intercal}$, where every row $\bm{x_i}^*$ of $Y^*$ is a point in $\mathcal{X}$ and $i$ denotes its sorted index in the optimal ordering. For now, we assume that such an ordering exists for every set $\mathcal{X}$. Later, in Section \ref{sec:optimal_sort}, we will describe some of the ideal properties of $Y^*$. 

We are interested in finding the optimal ordering $Y^*$ by sorting high-dimensional vector tokens in $\mathcal{X}$.
This is challenging due to the absence of a well-defined comparison operation in the $N$-dimensional space of vector tokens, resulting in a factorial number of possible orderings to be evaluated that is computationally prohibitive. To tackle this issue, we consider sorting points in $\mathcal{X}$ as a dimensionality reduction (DR) problem, as described in the following.

\subsection{DR-based Sorting Algorithms}

    
\begin{wrapfigure}{h}{0.7\textwidth}
    \small
    \vspace{-4ex}
    \centering
    \includegraphics[width=0.7\textwidth]{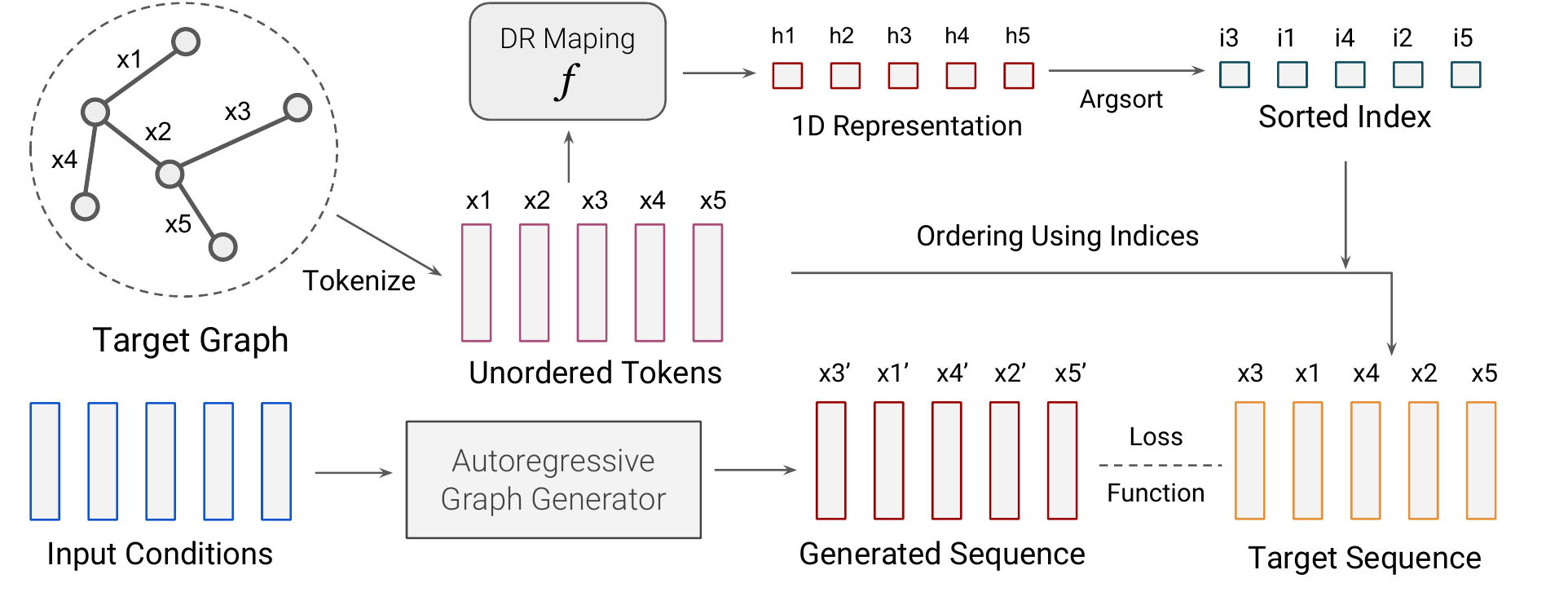}
    \caption{A general pipeline for sorting algorithms using dimensionality reduction (DR) mapping $f$. 
    }
    \label{fig:latent_sort}
    \vspace{-1ex}
    \end{wrapfigure}

Figure \ref{fig:latent_sort} provide an overview of the generic DR-based sorting pipeline for any DR mapping $f$. 
The key insight that we leverage here is that while sorting points is ill-defined in $N$-dimensional space, it is well-defined on a 1-D space (see Remark \ref{remark:sort_1d_well_defined} in Appendix). Hence, if we can learn a generalizable DR mapping from every possible token $\bm{x}$ to a corresponding 1D representation $h$, we can order tokens in any arbitrary set $\mathcal{X}$ by sorting their 1D representations. We can thus define DR-based sorting algorithms as follows.

\begin{restatable}[DR-based Sorting Algorithm]{definition}{defSorting}
\label{def:sorting_as_dr}
 Given a sorting algorithm $s: \mathcal{X} \rightarrow Y$, let $Y = \left[\bm{x_1}, \bm{x_2}, ..., \bm{x_M}\right]^{\intercal}$ be the resulting ordered sequence for the unordered set $\mathcal{X} \subseteq \mathbb{R}^N$. We can then represent the sorting algorithm $s$ by a DR mapping $f: \mathbb{R}^N \rightarrow \mathbb{R}$, such that the ordering in the original $N$-D space is given by sorting their 1-D values, $h_i = f(\bm{x_i}) \in \mathbb{R}$, and $\forall i \leq j \leq M, f(\bm{x_i}) \leq f(\bm{x_j})$.
\end{restatable}

Note that the DR mapping $f$ maps every token $ \bm{x_i}$ to a universal position $h_i$ in the latent space, regardless of the combination of tokens in $\mathcal{X}$ that it appears with. Ordering $\mathcal{X}$ thus reduces to finding a sorted path traversal in the 1D latent space. 
Thus, we formulate the problem of finding the optimal sorting $s^*: \mathcal{X} \rightarrow Y^*$ as finding an optimal DR mapping $f^*: \mathbb{R}^N \rightarrow \mathbb{R}$, $h_i^* = f^*(\bm{x_i}^*)$. 
While sorting itself is challenging to perform in a differentiable way, finding differentiable approaches for DR is relatively easier. This opens up new possibilities for developing learning-based methods to order points in high-dimensional spaces. For simplicity, we assume the 1-D representations of all the points in $\mathcal{X}$, denoted by $\mathcal{H}$, are normalized to $[0, 1]$.

\textbf{Latent Sort Algorithm}: We propose latent sort as our DR-based sorting algorithm, where the DR mapping $f$ is represented by an auto-encoder model consisting of an MLP encoder $f_e$ and an MLP decoder $f_d$. We first train $f_e$ and $f_d$ to reconstruct all tokens in the dataset. We then freeze $f_e$  and plug it in the pipeline shown in Figure \ref{fig:latent_sort} as the DR mapping $f$ for sorting. The property of latent sort will be analyzed in following sections.

\subsection{Analyzing The Errors of DR-based Sorting Algorithms}

To quantify the errors in a sorting algorithm $s$,
we introduce a probability matrix $P \in \mathbb{R}^{M \times M}$ where each entry $p_{ij}$ of $P$ represents the probability of $\bm{x_i}$ being $\bm{x_j}^*$. Thus, we have $\mathbb{E}[Y] = PY^*$, and $P$ can be viewed as a soft permutation matrix that permutes $Y^*$ into $Y$. We use the Frobenius norm of the difference between $\mathbb{E}[Y]$ and $Y^*$ as the expected error in a sorting algorithm, 
\begin{equation}\vspace{-0.05cm}
\mathcal{E}\Big(\mathbb{E}[Y], Y^*\Big) = \big\Vert \mathbb{E}[Y] - Y^* \big\Vert_F^2 = \big\Vert PY^* - Y^* \big\Vert_F^2,
\label{eq:sorting_error_def}
\end{equation}
where $\Vert \cdot \Vert_F$ denotes the Frobenius norm. 
We can see that two factors contribute to the error term $\mathcal{E}$:  $P$ and $Y^*$. It is obvious how $P$ affects the error, as $P$ measures the deviation from the optimal ordering $Y^*$, which reflects the property of an ordering method. We show that minimizing $\mathcal{E}\big(\mathbb{E}[Y], Y^*\big)$ is equivalent to minimizing $\Vert P - I_M \Vert_F^2$ (see Lemma \ref{lemma:optimization_equivalence} in the Appendix). On the other hand, it is less straight-forward to see  how the choice of $Y^*$ affects the error, which we will discuss in Section \ref{sec:optimal_sort}. Here, we give two sources of errors that can lead to suboptimal $P$.

\begin{wrapfigure}{h}{0.5\textwidth}
    \centering
    \small
    \vspace{-2ex}
    \includegraphics[width=0.5\textwidth]{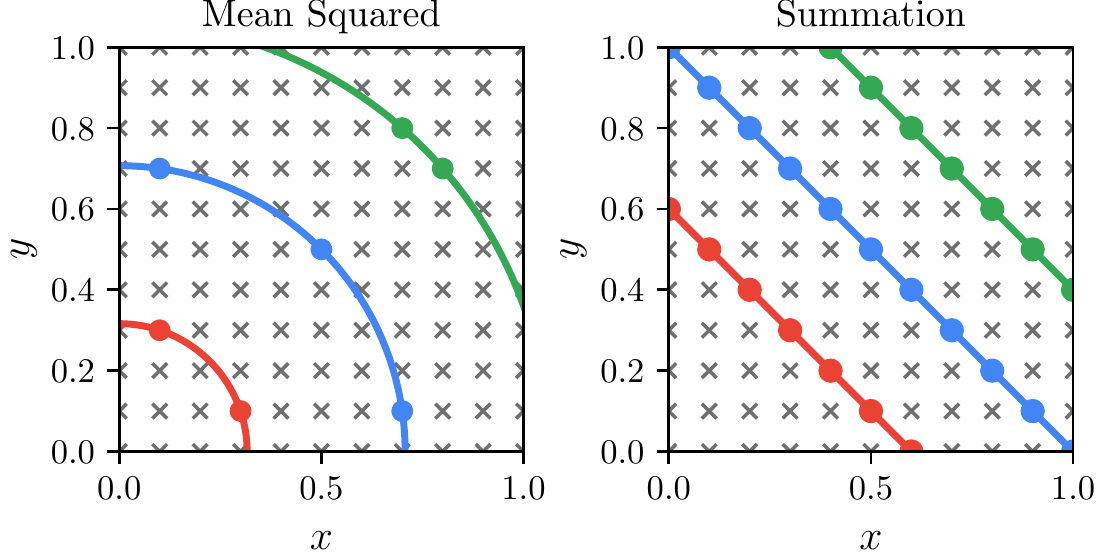}
    \caption{Illustration of the ordering ambiguity on 2D points within $[0, 1]$. Colored contour lines represent points with same 1D latent value, i.e., they belong to the same ordering ambiguity set.}
    \label{fig:sorting_ambiguity}
    \vspace{-1ex}
\end{wrapfigure}

\textbf{Errors from Ordering Ambiguity}: The first source of errors is the ordering ambiguity. An ideal mapping $f^*$ is one that does not result in any ambiguity in sorting, i.e.,  $f^*(\bm{x_i}) = f^*(\bm{x_j})$ iff $i = j$. However, such an ideal mapping only exists for certain conditions (see Remark \ref{remark:bijective_dimensionality_reduction} in Appendix), otherwise $f$ is surjective (non-bijective) even if it has no reconstruction errors. Thus, it is possible that a DR mapping $f$ assigns the same 1-D representation to different points in $\mathcal{X}$, which can cause ambiguity in sorting these points. To describe this ambiguity, we introduce the concept of \textit{ordering ambiguity set}. 

\begin{restatable}[Ordering Ambiguity Set]{definition}{defSortAmbiguity}
\label{def:sort_ambiguity}
Let $\mathcal{A}_i$ be the ordering ambiguity set for a point $\bm{x_i} \in \mathcal{X}$, where $\mathcal{A}_i = \left\{\bm{x_j} | \forall \bm{x_j} \in \mathcal{X}, f(\bm{x_i}) = f(\bm{x_j}) \right\}$.
\end{restatable}

Ordering ambiguity occurs when there exists more than one element in the ordering ambiguity set, i.e., $|\mathcal{A}_i| > 1$. Different DR mappings , can have varying ordering ambiguity patterns. Figure \ref{fig:sorting_ambiguity} shows ordering ambiguity patterns in 2D space for two DR mappings: (i) Mean Squared Sort: $h_i = (x_{i1}^2 + x_{i2}^2) / 2$, (ii) Summation Sort: $h_i = x_{i1} + x_{i2}$. 
 
To quantify the error introduced by ordering ambiguity, we consider the expectation of the ambiguity set $\mathcal{A}_i$ for each point $\bm{x_i} \in \mathcal{X}$. It is reasonable to assume that each row $i$ of $P$ is a uniform distribution over the corresponding ambiguity set $\mathcal{A}_i$ (i.e., $\forall \bm{x_j}^* \in \mathcal{A}_i$, $p_{ij} = {1}/{|\mathcal{A}_i|}$ and $\forall \bm{x_j}^* \not\in \mathcal{A}_i$,  $p_{ij} = 0$. Putting this in Equation (\ref{eq:sorting_error_def}), we get (see Section \ref{sec:ap_sorting_ambiguity} in Appendix for derivation):
\begin{align}\vspace{-0.08cm}
\mathcal{E}\Big(\mathbb{E}[Y], Y^*\Big) &= \sum_{i=1}^M \Bigg\Vert \left(\frac{1}{|\mathcal{A}_i|} \sum_{\bm{x_j} \in \mathcal{A}_i} \bm{x_j} \right) - \bm{x_i}\ \Bigg\Vert^2
\label{eq:error_sorting_ambiguity}
\end{align}

\textbf{Errors From Imperfect Reconstruction In Latent Sort}: Since latent sort trains an autoencoder using reconstruction loss for DR mapping, it is beneficial to discuss the potential errors from imperfect reconstructions (See Section \ref{sec:latent_sort_theory} and \ref{sec:ap_latent_sort_error} in the Appendix for details). Imperfect reconstruction leads to a discrepancy between the reconstructed token $\hat{x}$ and the original token $x$. Assuming the reconstruction follows normal distribution, the reconstruction error will also result in a normal distributed error with bounded mean and variance in the latent space (See Theorem \ref{th:error_latent_normal_distribution} in Appendix). This error can affect the ordering results when two points swap their positions in the sequence due to their 1-D latent values crossing each other. The probability of swap can be quantified by examining the overlap between their 1-D latent distributions. Since the 1-D latent values follow Gaussian distribution, and the overlaps of normal distributions are very small if the mean values of the distributions are far from each other, we only consider the swapping between neighboring points since the non-neighboring points will have larger mean value differences and thus are less likely to be swapped (See Section \ref{sec:ap_minization_of_errors} in Appendix for details).

\subsection{The Shortest Path Property of Desirable Ordering Schemes}
\label{sec:optimal_sort}

In previous theories, we have discussed the error of sorting algorithms using the probability matrix $P$. In this section, we  discuss the desired properties of $Y^*$ that minimize the errors induced by given $P$ matrices. 

From Section \ref{sec:ap_optimal_sort} in the Appendix, we have the simplified equation for error $\mathcal{E}$ as follows:
\begin{align}\vspace{-0.05cm}
\mathcal{E}\Big(\mathbb{E}[Y], Y^*\Big) = \sum_{i=2}^{M-1} \left\| p_{i(i-1)} (\bm{x_{i-1}}^* - \bm{x_{i}}^*) + p_{i(i+1)} (\bm{x_{i+1}}^* - \bm{x_{i}}^*) \right\|_F^2
\label{eq:error_using_neighbors} \vspace{-0.05cm}
\end{align}
Recall that $f$ maps each point in $\mathcal{X}$ to its 1D latent representation independently of other points. Therefore, it is unaware of the global information of the entire sequence. That being said, we expect (\ref{eq:error_using_neighbors}) to be minimized for any three points $\{ \bm{x_{i-1}}^*, \bm{x_{i}}^*, \bm{x_{i+1}}^* \} \subset \mathcal{X}$. To make $\mathcal{E}\big(\mathbb{E}[Y], Y^*\big)$ minimized for all possible $\mathcal{X} \subseteq \mathbb{R}^{N}$, we need to minimize the upper bound of (\ref{eq:error_using_neighbors}), giving by:
\begin{align} \vspace{-0.06cm}
&\ \ \ \ \sum_{i=2}^{M-1} \left\| p_{i(i-1)} (\bm{x_{i-1}^*} - \bm{x_{i}}^*) + p_{i(i+1)} (\bm{x_{i+1}}^* - \bm{x_{i}}^*) \right\|_F^2 \\
&\leq \sum_{i=2}^{M-1} \left\| p_{i(i-1)} (\bm{x_{i-1}}^* - \bm{x_{i}}^*) \right\|_F^2 + \sum_{i=2}^{M-1} \left\| p_{i(i+1)} (\bm{x_{i+1}}^* - \bm{x_{i}}^*) \right\|_F^2, \label{eq:error_using_neighbors_2} \vspace{-0.06cm}
\end{align}
assuming both $p_{i(i-1)}$ and $p_{i(i+1)}$ are nonzero for each $i$. In order to minimize errors from imperfect $P$, the ideal sorting $Y^*$ should minimize $(\bm{x_{i-1}}^* - \bm{x_{i}}^*)$ and $(\bm{x_{i+1}}^* - \bm{x_{i}}^*)$, i.e., minimize the distances between neighboring pairs. This property connects to the shortest path problem or traveling salesman problem (TSP), with the target sorting being the TSP solution on the input point set $\mathcal{X}$. However, TSP is NP-hard and solving it is computationally expensive. Using TSP solutions also induces difficulties for the autoregressive models to capture the ordering rule represented by the TSP solving algorithms. Thus, an approximation to the TSP solution should provide a balanced option between simplicity for learning and error tolerance.

Theorem \ref{th:bounded_data_distance} and \ref{th:bounded_latent_distance} in the Appendix state the correlation between the distance between two points and their 1-D latent distance. Since neighboring points in the sorted sequence have the closest latent values, the distance between neighboring points in the sorted sequence are expected to be small when the decoder is properly regularized. However, our experiments suggest that reconstruction loss and regularization techniques alone may not sufficiently approximate shortest path solutions. To further enhance the shortest path property in the latent sort algorithm, we introduce the \textit{latent gradient penalty} (LGP) loss given by:
\begin{align}
\text{LGP}(\mathcal{X}, f_e) = \min_{f_e} \sum_{i = 2}^{M - 1} \sum_{j = 2}^{M - 1} {g_{ij}^2}, & & \text{where}\ g_{ij} = \begin{cases}
\frac{\|\bm{x_i} - \bm{x_j}\|}{|h_i - h_j| + \beta} - \alpha & \text{if} \ |i - j| = 1, \\
0 & \text{otherwise}
\end{cases}
\end{align}
where $f_e$ is the latent sort encoder, and $\alpha$ and $\beta$ are positive constants. The LGP aims to make the 1-D latent value distance $|h_i - h_j|$ between two neighboring points proportional to their distance $\|\bm{x_i} - \bm{x_j}\|$ in the original space, with proportionality constant constrained by $\alpha$ (we chose $\alpha = 1$). Without extra justification, the latent sort algorithm in the paper are by default trained with LGP.

\section{Experiment Setups}
We briefly summarize the experiment setups in this section (see Appendix  \ref{sec:experiment_details} for complete details of the processing pipelines for full reproducibility). All codes are publicly on Github \footnote{\url{https://github.com/jayroxis/ordering-in-graph-generation}}.

\subsection{Application Tasks and Datasets}

\begin{figure}[ht]
    \centering
    \begin{subfigure}[b]{0.68\textwidth}
        \centering
        \includegraphics[width=\textwidth]{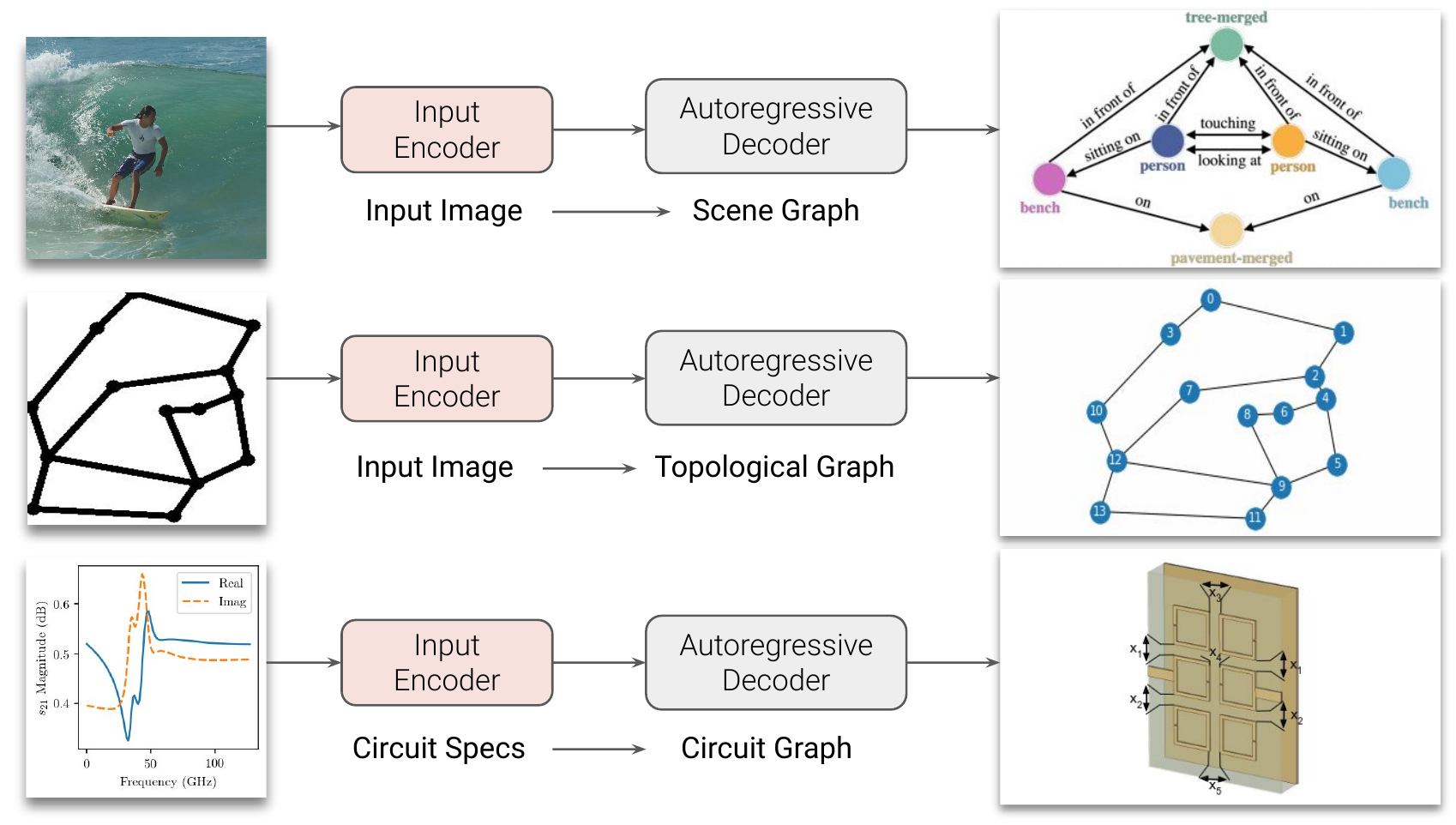}
        \caption{Paired Graph Generation Applications and Pipelines.}
        \label{fig:pipelines}
    \end{subfigure}
    \begin{subfigure}[b]{0.29\textwidth}
        \centering
        \includegraphics[width=\textwidth]{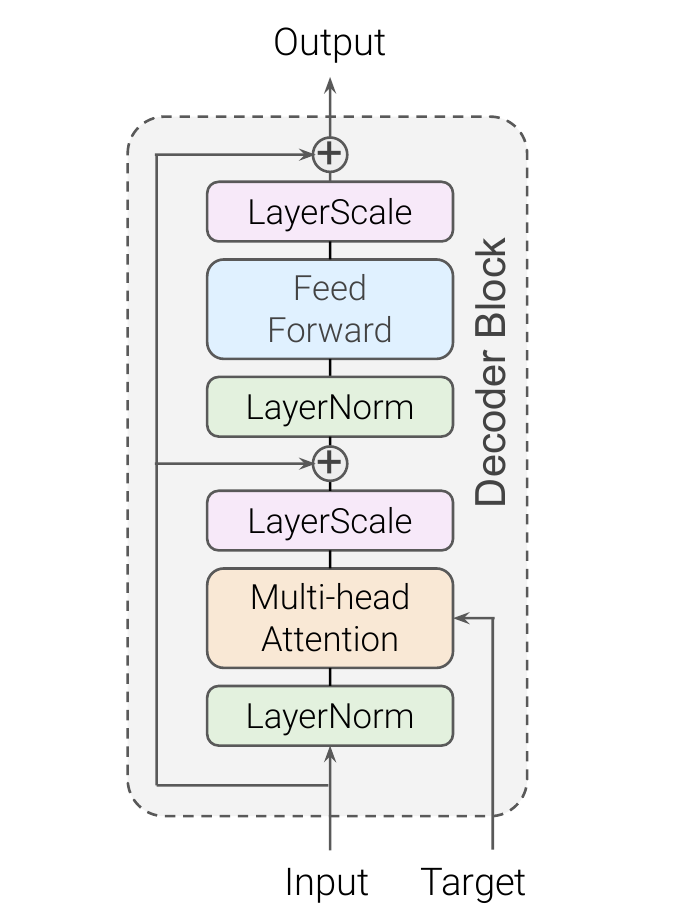}
        \caption{ \graphgpt{} Decoder Block}
        \label{fig:decoder_block}
    \end{subfigure}
    \caption{ Left: schematic views of the three paired graph generation tasks  (from top to bottom): scene graph generation \cite{yang2022psg}, topological graph extraction \cite{belli2019image}, and circuit graph prediction \cite{he2019circuit}. Right: The decoder block of the proposed \graphgpt{} model using multi-head causal attention architecture.}
    \label{fig:applications}
\end{figure} \vspace{-0.08cm}

We consider the following three application tasks for paired graph generation (see Figure \ref{fig:pipelines}).

\textbf{Scene Graph Generation}: The goal of this task is to take an image as input and generate a scene graph as output. We use the OpenPSG dataset \cite{yang2022psg}, which consists of semantic-level triplets including two objects (nodes) and predicates (edges) between them. 

\textbf{Topological Graph Extraction}: The goal here is to predict the spatial location of nodes in a graph and the undirected edges between them, given an image representation of the graph. We use the Toulouse Road Network (TRN) dataset \cite{belli2019image} and our synthetic Planar Graph dataset, which can randomly generate image-graph pairs that are more challenging than the ones in the TRN dataset. We average the results over 10 random runs for every method on the Planar Graph dataset.

\textbf{Circuit Graph Prediction}: For this task, we consider the Terahertz channelizer dataset \cite{he2019circuit}, which comprises 347k circuit graphs. The goal is to predict the ground-truth circuit graph comprises of 3 to 6 resonators, given desirable electromagnetic (EM) properties as inputs. Each resonator is represented by a 9-dimensional vector. Since the scale of some dimensions in the resonator vectors is much larger than others, which may dominate the prediction errors, we normalize the data using min-max normalization per dimension. We report the errors on the normalized data.

\subsection{Models and Pipelines}
\label{sec:models_and_pipelines}

Additional details for the baseline models can be found in Section \ref{sec:experiment_details} in the Supplementary.

\textbf{Backbone Models}: We propose two backbone models for paired graph generation: \graphgpt{} and \graphrnn{}. These models rely solely on their autoregressive nature to sequentially predict graph tokens, without involving graph-specific architecture components such as graph neural networks (GNNs). Both models consist of an encoder to encode the input conditions (e.g., images) and an autoregressive decoder for graph generation. The decoder of \graphgpt{} is based on our modified version of the GPT-2 architecture, as shown in Figure \ref{fig:decoder_block}. It incorporates pre-normalization \cite{xiong2020layer} and LayerScale \cite{touvron2021going} designs. In \graphrnn{}, we simply replace the autoregressive transformer of \graphgpt{} with a long short-term memory (LSTM) model \cite{hochreiter1997long}. These models can be coupled with any sorting algorithm for ordering graph tokens during training.

\textbf{Baseline Sorting Algorithms:} (1) \textit{Mean Squared Sort}: $n$-dimensional tokens are sorted based on the mean squared value of their $n$ dimensions. This heuristic sorts the graph tokens based on their overall magnitudes, from larger to smaller $L^2$ norms. 
(2) \textit{Lexicographical Sort}: We start by sorting tokens based on their first dimension, then the second dimension if there are ties, and so on until all dimensions have been considered. Note that this heuristic assumes a meaningful ordering of dimensions in the data, which may not be true.
(3) \textit{SVD Low-Rank Sort}: We perform a Singular Value Decomposition (SVD) of the tokens in a graph and sort the tokens by projecting them onto the direction of maximum variance, identified using the largest principal component. This serves as a linear baseline for dimensionality reduction, as opposed to the non-linear Latent Sort.
(4) \textit{BFS or DFS Sort}: We use the breadth-first search (BFS) or depth-first search (DFS) traversal on edges, implemented by NetworkX \cite{aric2008networkx}.

\textbf{Graph Tokenization Pipeline}: To address the potential disadvantages of transformers attending different token distributions for each category of graph components (e.g., nodes and edges), as advocated in \cite{touvron2021going}, we implemented a novel edge-based tokenization approach. This approach is similar to, but distinct from, the method used in Graph Transformers \cite{ying2021transformers, kim2022pure}. Specifically, whenever possible, we concatenated the features of each edge with graph-level features and the features of the two connecting nodes to create edge-based tokens. For example, in the topological graph extraction task where only node coordinates are available as features, we constructed edge tokens by combining the coordinates of the two connecting nodes for each edge. In the scene graph generation task, we followed the convention of using one-hot encoded triplets as edge tokens. In the circuit graph prediction task, we utilized a 9-dimensional vector as a token to represent each node (resonator), as employed in \cite{he2019circuit}. 

\textbf{Loss Functions}:
We trained all autoregressive models using a teacher forcing setup. For directed graphs, we calculated the loss between each element in the ground-truth sequence $Y_{gt}$ and its corresponding element in the predicted sequence $Y_{pred}$ using an appropriate choice of loss function $\mathcal{L}(Y_{pred}, Y_{gt})$.
For circuit extraction, we employed the elastic loss function, which is the sum of L1 and MSE loss, as $\mathcal{L}$. In scene graph generation, we used the binary cross-entropy (BCE) loss instead of cross-entropy (CE) loss, following previous works \cite{beyer2020we, wightman2021resnet}.
For the topological graph extraction task involving undirected graphs, we utilized the elastic loss as $\mathcal{L}$, but we considered its undirected variant $\mathcal{L}_{u}$. This variant accounts for swapping the node features in every edge token. The undirected graph loss is defined as follows:$
    \mathcal{L}_{u}(Y_{pred}, Y_{gt}) = \mathcal{L}(Y_{pred}, Y_{gt}) + \mathcal{L}(Y_{pred}, \overline{Y}_{gt}) 
$, 
where $\overline{Y}_{gt}$ is the sequence where the two node features in every edge token are swapped. 

\textbf{Evaluation Metrics}:
During testing, we employed the Edge Mover distance (EMD), Edge Hausdorff distance (EHD), and StreetMover Distance (SMD) from \cite{belli2019image} as evaluation metrics to assess the similarity between predicted graphs and ground-truth graphs.
The EMD calculates the average of the minimum distances between each point in the sequence $Y_{pred}$ and its nearest neighbor in the set $Y_{gt}$. The EHD measures the maximum distance between the closest points of the two sequences. On the other hand, the SMD is based on point cloud distances. Further details on these metrics can be found in Supplementary \ref{sec:experiment_details}.
For scene graph generation, we consider precision, recall, and F-1 scores as the predicted sequences consist of multi-class labels. Additionally, we report the differences in sizes between the predicted graphs and the target graphs as another metric in each application.







\section{Experimental  Results}

\subsection{Comparing Performance on Paired Graph Generation Tasks}

\begin{wraptable}{h}{0.57\textwidth}
    \centering
    \small
    \vspace{-1ex}
\caption{Toulouse Road Network Generation.}
\label{tab:road_graph_results}
\centering
\small
\begin{tabular}{ccc}
\toprule
\textbf{Model} & \textbf{Sorting / Matching} & \textbf{SMD} \cite{belli2019image} \\
\midrule
\graphgpt       & Latent Sort                      & \textbf{0.0075}                     \\
\graphgpt       & Mean Squared Sort                 & \textbf{0.0081}                     \\
\graphgpt       & Lexicographical Sort             & 0.0422                              \\
\graphgpt       & SVD Low-Rank Sort                & 0.0751                              \\
\graphgpt       & BFS Sort                         & 0.0893                              \\
\graphgpt       & DFS Sort                         & 0.0761                              \\
\midrule
GGT \cite{belli2019image} & BFS Sort   & 0.0158                                         \\
GraphRNN \cite{belli2019image, you2018graphrnn} & BFS Sort   & 0.0245                   \\ 
\bottomrule
\end{tabular}
\vspace{-0.05cm}
\end{wraptable}

Tables \ref{tab:road_graph_results} and \ref{tab:planar_graph_results} compare the performance of the Latent Sort algorithm with baseline sorting algorithms and models on the road network generation and planar graph generation tasks, respectively. The top-2 models are highlighted in bold in all tables in this paper. We observe that \graphgpt{} coupled with Latent Sort consistently demonstrates superior performance for topological graph extraction compared to baseline methods such as GGT \cite{belli2019image}, GraphRNN \cite{belli2019image, you2018graphrnn}, and GraphTR. This is noteworthy as GGT and GraphRNN utilize graph-specific architectures to predict adjacency matrices in addition to graph generation, unlike our approach. These results showcase the effectiveness of using purely autoregressive models for graph generation when combined with an appropriate graph token ordering scheme like Latent Sort.

Among the various sorting algorithms used with \graphgpt{}, Mean Squared Sort stands out for its impressive performance on both datasets, despite being a simple heuristic. This could potentially be attributed to its simplicity in sorting based on token magnitudes, which the autoregressive models can effectively capture. It highlights the importance of ordering schemes being "easy to learn" which is commonly neglected. On the other hand, BFS and DFS sort consistently exhibit inferior performance on both datasets for paired graph generation, despite being the de facto convention used in existing methods for probabilistic graph generation. We hypothesize that although BFS and DFS can generate viable distributions of graphs (e.g., molecules) in a probabilistic setting, they perform poorly when an exact match with a target graph is required. This is because they do not consider the values of graph tokens, but only their relative positions in the graph, which can be arbitrarily defined. Moreover, the 

\begin{table}
\centering
\small
\vspace{-3ex}
\caption{Planar Graph Generation.}
\label{tab:planar_graph_results}
\centering
\small
\begin{tabular}{cccccc}
\toprule
\textbf{Model}    & \textbf{Sorting / Matching} & \textbf{EMD}    & \textbf{EHD}    & \textbf{SMD}    & \textbf{Size. Diff.} \\
\midrule
\graphgpt          & Latent Sort                       & \textbf{0.0038}          & \textbf{0.0166}          & \textbf{0.0002}          & \textbf{0.0021}              \\
\graphgpt          & Mean Squared Sort                  & \textbf{0.0031}          & \textbf{0.0151}          & \textbf{0.0002}          & -0.0360              \\
\graphgpt          & Lexicographical Sort              & 0.0141          & 0.0564          & 0.0013          & -0.2434              \\
\graphgpt          & SVD Low-Rank Sort                 & 0.0381          & 0.0951          & 0.0068          & 4.8280               \\ 
\graphgpt          & BFS Sort                          & 0.0217          & 0.0579          & 0.0017          & 1.7296              \\
\graphgpt          & DFS Sort                          & 0.0213          & 0.0572          & 0.0017          & 1.6399               \\ \midrule
\graphrnn{}          & Latent Sort                     & 0.0100          & 0.0293          & 0.0006          & 0.1639               \\
\graphrnn{}          & Mean Squared Sort                & 0.0063          & 0.0211          & 0.0004          & \textbf{0.0060}               \\
\graphrnn{}          & Lexicographical Sort            & 0.0248          & 0.0609          & 0.0030          & 3.9318               \\
\graphrnn{}          & SVD Low-Rank Sort               & 0.0480          & 0.0940          & 0.0069          & 4.5328               \\ 
\graphrnn{}          & BFS Sort                        & 0.0446          & 0.0888          & 0.0059          & 3.6330              \\
\graphrnn{}          & DFS Sort                        & 0.0391          & 0.0738          & 0.0045          & 1.7268               \\\midrule
GraphTR           & Hungarian Matcher                 & 0.0069          & 0.0348          & 0.0003          & -0.0130              \\
\bottomrule
\end{tabular}
\vspace{-2ex}
\end{table}

autoregressive model needs to understand the global structure of the graph to determine the sequence ordering, adding significant complexity in learning the ordering. Additionally, we observe that \graphrnn{} generally does not perform as well as \graphgpt{}. This could be due to the advantage of transformers, which are capable of attending to longer context that help it gain better understanding on global information in graphs compared to LSTMs.

\begin{table}[ht]
\vspace{-1ex}
\caption{Semantic Scene Graph Generation on OpenPSG \cite{yang2022psg} Dataset .}
\label{tab:scene_graph_results}
\centering
\small
\begin{tabular}{cccccc}
\toprule
\textbf{Model}    & {\textbf{Sorting / Matching}} & \textbf{Additional Labels} & \textbf{Precision} & \textbf{Recall} & \textbf{F1-Score} \\
\midrule
\graphgpt      & Latent Sort           & None    & \textbf{0.2537}    & 0.1724           & \textbf{0.2053}   \\
\graphgpt      & Lexicographical Sort  & None    & \textbf{0.2425}    & \textbf{0.1867}  & \textbf{0.2110}   \\
\graphgpt      & SVD Low-Rank Sort     & None    & 0.2496             & 0.1696           & 0.2020            \\ \midrule
\graphrnn{}    & Latent Sort         & None    & 0.0115             & 0.0094           & 0.0104            \\
\graphrnn{}    & Lexicographical Sort  & None    & 0.0110             & 0.0093           & 0.0101            \\ 
\graphrnn{}    & SVD Low-Rank Sort     & None    & 0.0099             & 0.0079           & 0.0088            \\ \midrule
PSGTR \cite{yang2022psg}      & Hungarian Matcher                                     & Panoptic Segmentation         & 0.1810             & \textbf{0.2116} & 0.1920            \\
PSGFormer \cite{yang2022psg}  & Query Matching Block                                  & Panoptic Segmentation         & 0.0438             & 0.0520          & 0.0467     \\
\bottomrule
\end{tabular}
\end{table}
\begin{table}
    \centering
    \small
    \vspace{-2ex}
\caption{Results on Circuit Graph Prediction \cite{he2019circuit}. }
\label{tab:circuit_graph_results}
\centering
\small
\begin{tabular}{ccccc}
\toprule
\textbf{Model} & \textbf{Token Sorting / Matching} & \textbf{EMD} & \textbf{EHD} & \textbf{Size. Diff.} \\
\midrule
\graphgpt       & Latent Sort                       & \textbf{0.0509}       & 0.0856       & 0.6320               \\
\graphgpt       & Mean Squared Sort                  & 0.0558       & 0.0860       & \textbf{0.0794}               \\
\graphgpt       & Lexicographical Sort              & 0.0623       & 0.1080       & 1.4757               \\
\graphgpt       & SVD Low-Rank Sort                 & 0.0590       & 0.0863       & 1.7485               \\
\midrule
\graphrnn{}       & Latent Sort                       & \textbf{0.0496}       & \textbf{0.0757}       & 0.4249               \\
\graphrnn{}       & Mean Squared Sort                  & 0.0576       & 0.0836       & -0.1305              \\
\graphrnn{}       & Lexicographical Sort              & 0.0528       & \textbf{0.0803}       & \textbf{-0.1205}              \\
\graphrnn{}       & SVD Low-Rank Sort                 & 0.0516       & 0.0777       & 0.6112              \\ \midrule
Circuit-GNN \cite{he2019circuit}       & N/A                 & 0.0738       & 0.0943       & (Set to GT)              \\ 
\bottomrule
\end{tabular}
\end{table}

Tables \ref{tab:scene_graph_results} and \ref{tab:circuit_graph_results} present the results for the scene graph generation and circuit graph prediction tasks, respectively. Both tasks involve directed graphs. It is important to note that the mean squared sort cannot be applied to scene graph generation due to the one-hot encoding of graph tokens, which leads to identical mean squared values for all tokens. Additionally, the BFS and DFS sorting algorithms could hardly be applied to both datasets as they mostly consist of disconnected small graphs.
In the scene graph generation task, we observe that \graphgpt{} exhibits higher precision compared to the state-of-the-art (SOTA) baselines PSGTR and PSGFormer for all sorting algorithm choices, where PSGTR \cite{yang2022psg} and PSGFormer \cite{yang2022psg} use additional panoptic segmentation labels for training while we only rely on the scene graph labels. However, it has lower recall than PSGTR. It is worth mentioning that PSGTR has an advantage over our method as it utilizes additional supervision from panoptic segmentation during training and knows the ground-truth sizes of target graphs as user-specified hyperparameters during testing. Despite this, \graphrnn{} shows better F-1 scores than PSGTR for multiple sorting algorithms. Moreover, we observe that lexicographical sort performs the best on this dataset, followed closely by Latent Sort. Using lexicographical sort as a domain-specific heuristic for this problem aligns well with the categorical nature of graph tokens. We include additional results in Section \ref{sec:additional_results} in the Supplementary, including the discussion of label noises in the dataset that lead to size differences between \graphgpt{} generated scene graph vs the ground-truth labels. 
For the circuit graph prediction task, both \graphgpt{} and \graphrnn{} models outperform Circuit-GNN. Among the ordering schemes, Latent Sort exhibits a slight advantage over others.

In summary, while specific heuristics may excel in certain applications (e.g., mean squared sort in topological graph generation and lexicographical sort in scene graph generation), Latent Sort demonstrates versatility and consistently delivers competitive performance across all four datasets.

\subsection{Analyzing Sorting Ambiguity and Shortest Path Property of Latent Sort}
\label{sec:shortest_path_property}

\begin{wrapfigure}{h}{0.65\textwidth}
    \small
    \vspace{-3ex}
    \centering
    \includegraphics[width=0.64\textwidth]{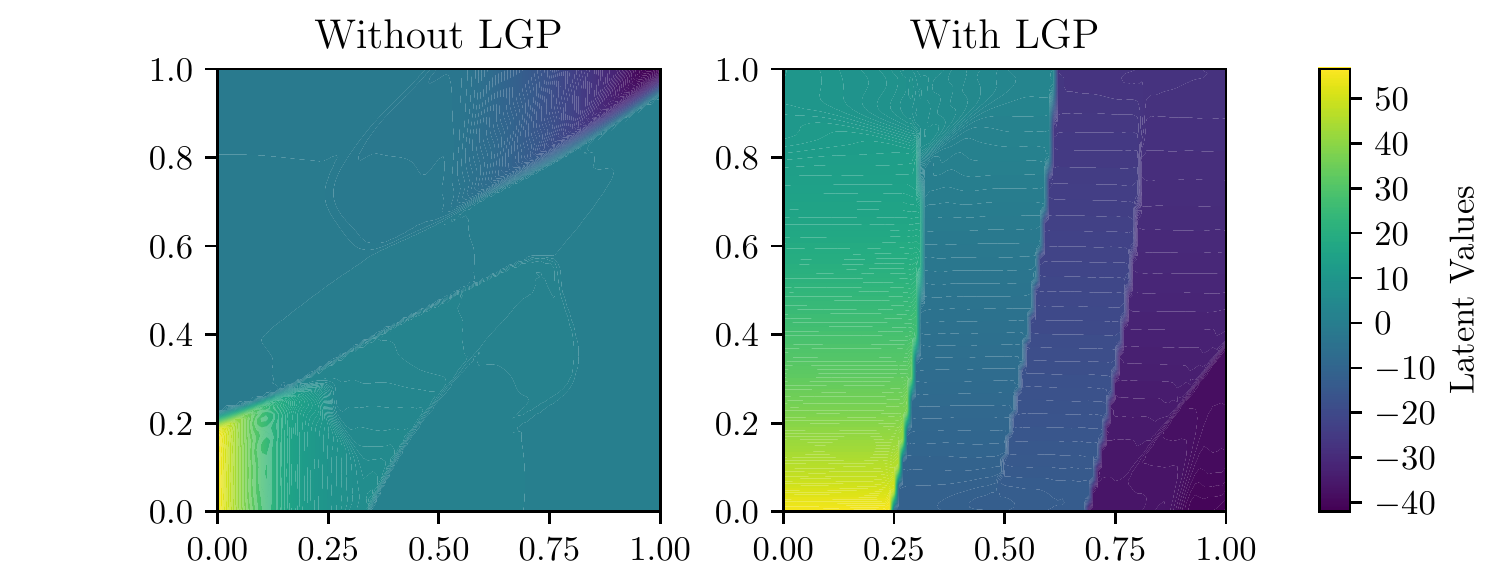}
    \caption{1-D representation of 2-D points learned by latent sort with (right) and without (left) using latent gradient penalty (LGP). White lines are contours showing regions with same latent values.}
    \label{fig:latent_sort_contour}
\vspace{-0.05cm}
\end{wrapfigure}

To analyze the importance of using latent gradient penalty (LGP) in the training of 1-D representations by latent sort, Figure \ref{fig:latent_sort_contour} shows the contour lines in the 1-D space learned by latent sort for a toy 2-D dataset, with and without LGP. The latent space learned without LGP exhibits large regions with identical values, which, according to the theory of ordering ambiguity, can lead to significant errors. In contrast, the latent space learned with LGP demonstrates finer partitioning, resulting in smaller regions sharing the same values. This leads to reduced ambiguity in ordering and, consequently, lower errors due to ordering ambiguity.


\begin{figure}[ht]
    \centering
    \includegraphics[width=1.0\textwidth]{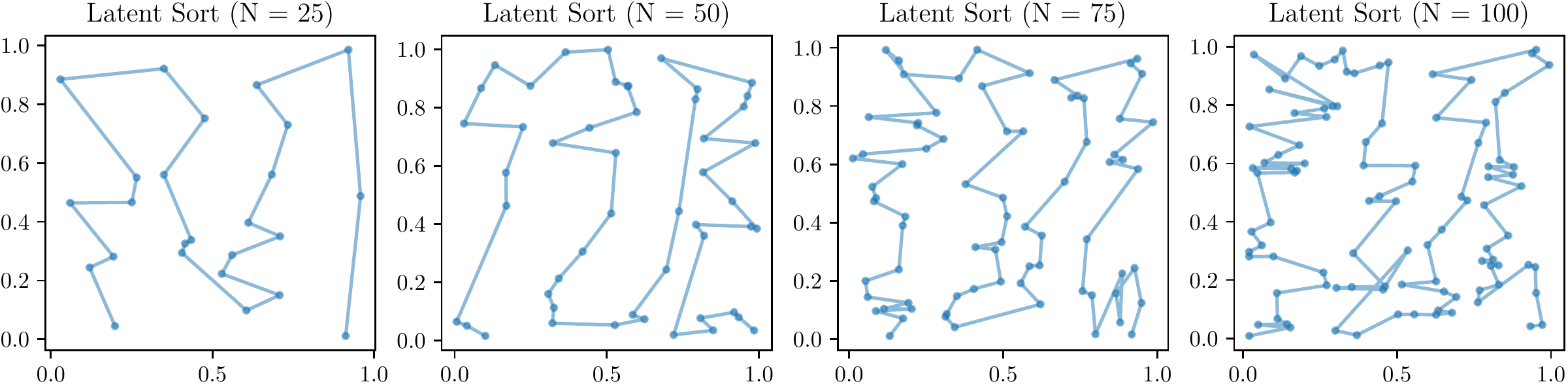}
    \caption{Approximating shortest path solutions using latent sort (with LGP).}
    \label{fig:latent_sort_tsp}
\end{figure}

Moreover, to empirically support the theoretical claim that latent sort finds an approximate solution to the shortest path problem, Figure \ref{fig:latent_sort_tsp} showcases the paths identified by latent sort over different collections of points in 2-D space. 
These paths are determined by sorting the points based on their 1-D values. 
The examples demonstrate that latent sort reasonably approximates the shortest path solution, even for a large number of points. 
For additional quantitative results, please refer to Appendix \ref{sec:additional_results}. These findings highlight an intriguing connection between Latent Sort and shortest path problems, motivating a new use-case for Latent Sort as a fast, low-cost method for approximating shortest path solutions through GPU-based parallelization of neural networks. 
This is particularly beneficial for applications with strict runtime efficiency requirements.

\section{Conclusions}

We presented a novel autoregressive framework for paired graph generation by reframing the task of sorting unordered tokens as a dimensionality reduction problem from a theoretical standpoint. We showed the efficacy of our proposed \textit{latent sort} algorithm on various graph generation tasks. 
However, we also acknowledge certain limitations of our work. While latent sort performs competitively across a wide range of datasets, it does not consistently outperform application-specific heuristics. This suggests the need for future works to explore more sophisticated learning-based ordering schemes. Additionally, our experiments are currently limited to small graphs, emphasizing the importance of investigating the scalability of our approach to larger graphs in future research.

\bibliography{Jie}

\begin{thebibliography}{10}

\bibitem{belli2019image}
Davide Belli and Thomas Kipf.
\newblock Image-conditioned graph generation for road network extraction.
\newblock {\em arXiv preprint arXiv:1910.14388}, 2019.

\bibitem{lu2021context}
Yichao Lu, Himanshu Rai, Jason Chang, Boris Knyazev, Guangwei Yu, Shashank
  Shekhar, Graham~W Taylor, and Maksims Volkovs.
\newblock Context-aware scene graph generation with seq2seq transformers.
\newblock In {\em Proceedings of the IEEE/CVF international conference on
  computer vision}, pages 15931--15941, 2021.

\bibitem{yang2022psg}
Jingkang Yang, Yi~Zhe Ang, Zujin Guo, Kaiyang Zhou, Wayne Zhang, and Ziwei Liu.
\newblock Panoptic scene graph generation.
\newblock In {\em ECCV}, 2022.

\bibitem{li2022sgtr}
Rongjie Li, Songyang Zhang, and Xuming He.
\newblock Sgtr: End-to-end scene graph generation with transformer.
\newblock In {\em Proceedings of the IEEE/CVF Conference on Computer Vision and
  Pattern Recognition}, pages 19486--19496, 2022.

\bibitem{he2019circuit}
Guo Zhang, Hao He, and Dina Katabi.
\newblock Circuit-gnn: Graph neural networks for distributed circuit design.
\newblock In {\em International Conference on Machine Learning}, pages
  7364--7373, 2019.

\bibitem{brown2020language}
Tom Brown, Benjamin Mann, Nick Ryder, Melanie Subbiah, Jared~D Kaplan, Prafulla
  Dhariwal, Arvind Neelakantan, Pranav Shyam, Girish Sastry, Amanda Askell,
  et~al.
\newblock Language models are few-shot learners.
\newblock {\em Advances in neural information processing systems},
  33:1877--1901, 2020.

\bibitem{ouyang2022training}
Long Ouyang, Jeffrey Wu, Xu~Jiang, Diogo Almeida, Carroll Wainwright, Pamela
  Mishkin, Chong Zhang, Sandhini Agarwal, Katarina Slama, Alex Ray, et~al.
\newblock Training language models to follow instructions with human feedback.
\newblock {\em Advances in Neural Information Processing Systems},
  35:27730--27744, 2022.

\bibitem{bubeck2023sparks}
S{\'e}bastien Bubeck, Varun Chandrasekaran, Ronen Eldan, Johannes Gehrke, Eric
  Horvitz, Ece Kamar, Peter Lee, Yin~Tat Lee, Yuanzhi Li, Scott Lundberg,
  et~al.
\newblock Sparks of artificial general intelligence: Early experiments with
  gpt-4.
\newblock {\em arXiv preprint arXiv:2303.12712}, 2023.

\bibitem{li2018learning}
Yujia Li, Oriol Vinyals, Chris Dyer, Razvan Pascanu, and Peter Battaglia.
\newblock Learning deep generative models of graphs.
\newblock {\em arXiv preprint arXiv:1803.03324}, 2018.

\bibitem{you2018graphrnn}
Jiaxuan You, Rex Ying, Xiang Ren, William Hamilton, and Jure Leskovec.
\newblock Graphrnn: Generating realistic graphs with deep auto-regressive
  models.
\newblock In {\em International conference on machine learning}, pages
  5708--5717. PMLR, 2018.

\bibitem{liao2019efficient}
Renjie Liao, Yujia Li, Yang Song, Shenlong Wang, Will Hamilton, David~K
  Duvenaud, Raquel Urtasun, and Richard Zemel.
\newblock Efficient graph generation with graph recurrent attention networks.
\newblock {\em Advances in neural information processing systems}, 32, 2019.

\bibitem{liu2019auto}
Chia-Cheng Liu, Harris Chan, Kevin Luk, and AI~Borealis.
\newblock Auto-regressive graph generation modeling with improved evaluation
  methods.
\newblock In {\em 33rd Conference on Neural Information Processing Systems.
  Vancouver, Canada}, 2019.

\bibitem{hochreiter1997long}
Sepp Hochreiter and J{\"u}rgen Schmidhuber.
\newblock Long short-term memory.
\newblock {\em Neural computation}, 9(8):1735--1780, 1997.

\bibitem{cho2014learning}
Kyunghyun Cho, Bart Van~Merri{\"e}nboer, Caglar Gulcehre, Dzmitry Bahdanau,
  Fethi Bougares, Holger Schwenk, and Yoshua Bengio.
\newblock Learning phrase representations using rnn encoder-decoder for
  statistical machine translation.
\newblock {\em arXiv preprint arXiv:1406.1078}, 2014.

\bibitem{vaswani2017attention}
Ashish Vaswani, Noam Shazeer, Niki Parmar, Jakob Uszkoreit, Llion Jones,
  Aidan~N Gomez, {\L}ukasz Kaiser, and Illia Polosukhin.
\newblock Attention is all you need.
\newblock {\em Advances in neural information processing systems}, 30, 2017.

\bibitem{radford2018improving}
Alec Radford, Karthik Narasimhan, Tim Salimans, Ilya Sutskever, et~al.
\newblock Improving language understanding by generative pre-training.
\newblock {\em OpenAI blog}, 2018.

\bibitem{radford2019language}
Alec Radford, Jeffrey Wu, Rewon Child, David Luan, Dario Amodei, Ilya
  Sutskever, et~al.
\newblock Language models are unsupervised multitask learners.
\newblock {\em OpenAI blog}, 1(8):9, 2019.

\bibitem{hamilton2017inductive}
Will Hamilton, Zhitao Ying, and Jure Leskovec.
\newblock Inductive representation learning on large graphs.
\newblock {\em Advances in neural information processing systems}, 30, 2017.

\bibitem{vinyals2015order}
Oriol Vinyals, Samy Bengio, and Manjunath Kudlur.
\newblock Order matters: Sequence to sequence for sets.
\newblock {\em arXiv preprint arXiv:1511.06391}, 2015.

\bibitem{chen2021order}
Xiaohui Chen, Xu~Han, Jiajing Hu, Francisco~JR Ruiz, and Liping Liu.
\newblock Order matters: Probabilistic modeling of node sequence for graph
  generation.
\newblock {\em arXiv preprint arXiv:2106.06189}, 2021.

\bibitem{bojchevski2018netgan}
Aleksandar Bojchevski, Oleksandr Shchur, Daniel Z{\"u}gner, and Stephan
  G{\"u}nnemann.
\newblock Netgan: Generating graphs via random walks.
\newblock In {\em International conference on machine learning}, pages
  610--619. PMLR, 2018.

\bibitem{liu2018constrained}
Qi~Liu, Miltiadis Allamanis, Marc Brockschmidt, and Alexander Gaunt.
\newblock Constrained graph variational autoencoders for molecule design.
\newblock {\em Advances in neural information processing systems}, 31, 2018.

\bibitem{ma2018constrained}
Tengfei Ma, Jie Chen, and Cao Xiao.
\newblock Constrained generation of semantically valid graphs via regularizing
  variational autoencoders.
\newblock {\em Advances in Neural Information Processing Systems}, 31, 2018.

\bibitem{yang2019conditional}
Carl Yang, Peiye Zhuang, Wenhan Shi, Alan Luu, and Pan Li.
\newblock Conditional structure generation through graph variational generative
  adversarial nets.
\newblock {\em Advances in neural information processing systems}, 32, 2019.

\bibitem{jo2022score}
Jaehyeong Jo, Seul Lee, and Sung~Ju Hwang.
\newblock Score-based generative modeling of graphs via the system of
  stochastic differential equations.
\newblock In {\em International Conference on Machine Learning}, pages
  10362--10383. PMLR, 2022.

\bibitem{vignac2022digress}
Clement Vignac, Igor Krawczuk, Antoine Siraudin, Bohan Wang, Volkan Cevher, and
  Pascal Frossard.
\newblock Digress: Discrete denoising diffusion for graph generation.
\newblock {\em International Conference on Learning Representations (ICLR
  2023)}, 2022.

\bibitem{zaheer2017deep}
Manzil Zaheer, Satwik Kottur, Siamak Ravanbakhsh, Barnabas Poczos, Russ~R
  Salakhutdinov, and Alexander~J Smola.
\newblock Deep sets.
\newblock {\em Advances in neural information processing systems}, 30, 2017.

\bibitem{zhang2019deep}
Yan Zhang, Jonathon Hare, and Adam Prugel-Bennett.
\newblock Deep set prediction networks.
\newblock {\em Advances in Neural Information Processing Systems}, 32, 2019.

\bibitem{kosiorek2020conditional}
Adam~R Kosiorek, Hyunjik Kim, and Danilo~J Rezende.
\newblock Conditional set generation with transformers.
\newblock {\em arXiv preprint arXiv:2006.16841}, 2020.

\bibitem{carion2020end}
Nicolas Carion, Francisco Massa, Gabriel Synnaeve, Nicolas Usunier, Alexander
  Kirillov, and Sergey Zagoruyko.
\newblock End-to-end object detection with transformers.
\newblock In {\em Computer Vision--ECCV 2020: 16th European Conference,
  Glasgow, UK, August 23--28, 2020, Proceedings, Part I 16}, pages 213--229.
  Springer, 2020.

\bibitem{jonker1988shortest}
Roy Jonker and Ton Volgenant.
\newblock A shortest augmenting path algorithm for dense and sparse linear
  assignment problems.
\newblock In {\em DGOR/NSOR: Papers of the 16th Annual Meeting of DGOR in
  Cooperation with NSOR/Vortr{\"a}ge der 16. Jahrestagung der DGOR zusammen mit
  der NSOR}, pages 622--622. Springer, 1988.

\bibitem{sun2021rethinking}
Zhiqing Sun, Shengcao Cao, Yiming Yang, and Kris~M Kitani.
\newblock Rethinking transformer-based set prediction for object detection.
\newblock In {\em Proceedings of the IEEE/CVF international conference on
  computer vision}, pages 3611--3620, 2021.

\bibitem{zhang2022accelerating}
Gongjie Zhang, Zhipeng Luo, Yingchen Yu, Kaiwen Cui, and Shijian Lu.
\newblock Accelerating detr convergence via semantic-aligned matching.
\newblock In {\em Proceedings of the IEEE/CVF Conference on Computer Vision and
  Pattern Recognition}, pages 949--958, 2022.

\bibitem{jin2018junction}
Wengong Jin, Regina Barzilay, and Tommi Jaakkola.
\newblock Junction tree variational autoencoder for molecular graph generation.
\newblock In {\em International conference on machine learning}, pages
  2323--2332. PMLR, 2018.

\bibitem{xiong2020layer}
Ruibin Xiong, Yunchang Yang, Di~He, Kai Zheng, Shuxin Zheng, Chen Xing,
  Huishuai Zhang, Yanyan Lan, Liwei Wang, and Tieyan Liu.
\newblock On layer normalization in the transformer architecture.
\newblock In {\em International Conference on Machine Learning}, pages
  10524--10533. PMLR, 2020.

\bibitem{touvron2021going}
Hugo Touvron, Matthieu Cord, Alexandre Sablayrolles, Gabriel Synnaeve, and
  Herv{\'e} J{\'e}gou.
\newblock Going deeper with image transformers.
\newblock In {\em Proceedings of the IEEE/CVF International Conference on
  Computer Vision}, pages 32--42, 2021.

\bibitem{aric2008networkx}
Aric~A. Hagberg, Daniel~A. Schult, and Pieter~J. Swart.
\newblock Exploring network structure, dynamics, and function using networkx.
\newblock In Ga\"el Varoquaux, Travis Vaught, and Jarrod Millman, editors, {\em
  Proceedings of the 7th Python in Science Conference}, pages 11 -- 15,
  Pasadena, CA USA, 2008.

\bibitem{ying2021transformers}
Chengxuan Ying, Tianle Cai, Shengjie Luo, Shuxin Zheng, Guolin Ke, Di~He,
  Yanming Shen, and Tie-Yan Liu.
\newblock Do transformers really perform badly for graph representation?
\newblock {\em Advances in Neural Information Processing Systems},
  34:28877--28888, 2021.

\bibitem{kim2022pure}
Jinwoo Kim, Dat Nguyen, Seonwoo Min, Sungjun Cho, Moontae Lee, Honglak Lee, and
  Seunghoon Hong.
\newblock Pure transformers are powerful graph learners.
\newblock {\em Advances in Neural Information Processing Systems},
  35:14582--14595, 2022.

\bibitem{beyer2020we}
Lucas Beyer, Olivier~J H{\'e}naff, Alexander Kolesnikov, Xiaohua Zhai, and
  A{\"a}ron van~den Oord.
\newblock Are we done with imagenet?
\newblock {\em arXiv preprint arXiv:2006.07159}, 2020.

\bibitem{wightman2021resnet}
Ross Wightman, Hugo Touvron, and Herv{\'e} J{\'e}gou.
\newblock Resnet strikes back: An improved training procedure in timm.
\newblock {\em arXiv preprint arXiv:2110.00476}, 2021.

\bibitem{velivckovic2017graph}
Petar Veli{\v{c}}kovi{\'c}, Guillem Cucurull, Arantxa Casanova, Adriana Romero,
  Pietro Lio, and Yoshua Bengio.
\newblock Graph attention networks.
\newblock {\em arXiv preprint arXiv:1710.10903}, 2017.

\bibitem{tan2020efficientnet}
Mingxing Tan and Quoc~V. Le.
\newblock Efficientnet: Rethinking model scaling for convolutional neural
  networks, 2020.

\bibitem{he2015deep}
Kaiming He, Xiangyu Zhang, Shaoqing Ren, and Jian Sun.
\newblock Deep residual learning for image recognition, 2015.

\bibitem{russakovsky2015imagenet}
Olga Russakovsky, Jia Deng, Hao Su, Jonathan Krause, Sanjeev Satheesh, Sean Ma,
  Zhiheng Huang, Andrej Karpathy, Aditya Khosla, Michael Bernstein,
  Alexander~C. Berg, and Li~Fei-Fei.
\newblock {ImageNet Large Scale Visual Recognition Challenge}.
\newblock {\em International Journal of Computer Vision (IJCV)},
  115(3):211--252, 2015.

\bibitem{dosovitskiy2021image}
Alexey Dosovitskiy, Lucas Beyer, Alexander Kolesnikov, Dirk Weissenborn,
  Xiaohua Zhai, Thomas Unterthiner, Mostafa Dehghani, Matthias Minderer, Georg
  Heigold, Sylvain Gelly, Jakob Uszkoreit, and Neil Houlsby.
\newblock An image is worth 16x16 words: Transformers for image recognition at
  scale, 2021.

\bibitem{achille2019critical}
Alessandro Achille, Matteo Rovere, and Stefano Soatto.
\newblock Critical learning periods in deep neural networks, 2019.

\end{thebibliography}
\bibliographystyle{unsrt}

\newpage

\appendix

\section{Additional Experiment Results}
\label{sec:additional_results}

In this section, we present additional experiment results from our study, offering further insights into the aspects and factors contributing to our proposed framework.

\subsection{Approximating TSP Solutions Using Latent Sort}

To investigate the performance of latent sort algorithms in approximating the Traveling Salesman Problem (TSP), we present additional results using a synthetic 2-D dataset in $\mathbb{R}^2$. All points in the dataset are generated using a uniform distribution, and the objective is to find the shortest path that traverses all the given points. Table \ref{tab:tsp_results} provides a summary of the experiment outcomes.

\begin{table}[ht]
\caption{Approximating TSP solutions using latent sort algorithm. We report the percentage of all possible paths traversing all points (generated using brute-force enumeration) that are longer than the solution predicted by the latent sort algorithm. The results are reported in ``mean $\pm$ standard deviation'' from 10 random runs of generating $N$ random points and running latent  sort.}
\label{tab:tsp_results}
\centering
\small
\begin{tabular}{ccccc}
\toprule
 & N = 5 & N = 6 & N = 7 & N = 8 \\ \midrule
LS w. LGP & $93.17\% \pm 11.77\%$ & $92.00\% \pm 8.69\%$ & $97.01\% \pm 6.10\%$ & $99.59\% \pm 0.34\%$ \\
LS w.o. LGP & $82.17\% \pm 13.12\%$ & $90.44\% \pm 11.44\%$ & $94.67\% \pm 6.64\%$ & $96.61\% \pm 5.14\%$ \\ \bottomrule
\end{tabular}
\end{table}

\begin{figure}[ht]
    \centering
    \includegraphics[width=0.49\textwidth]{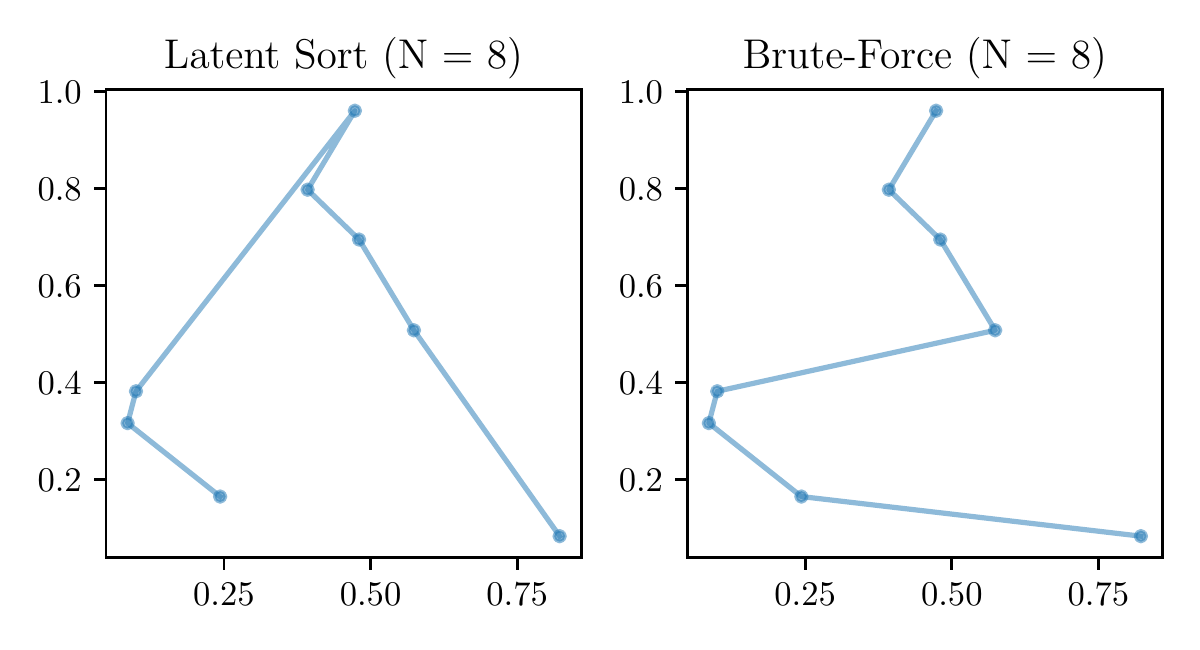}
    \includegraphics[width=0.49\textwidth]{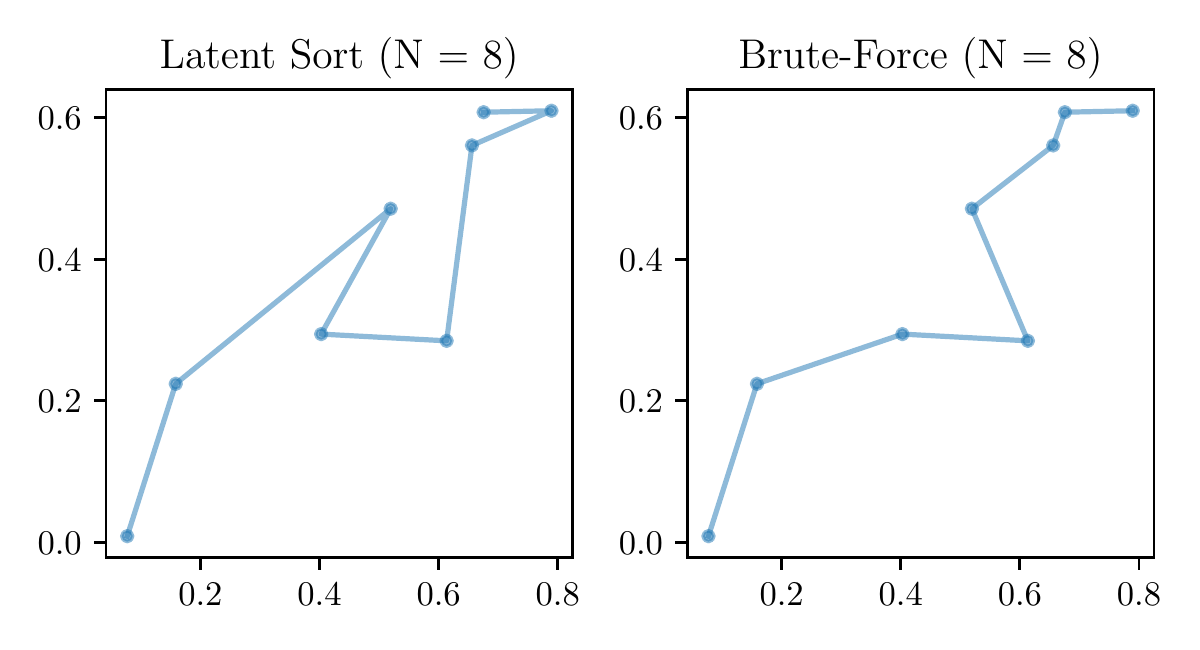}
    \includegraphics[width=0.49\textwidth]{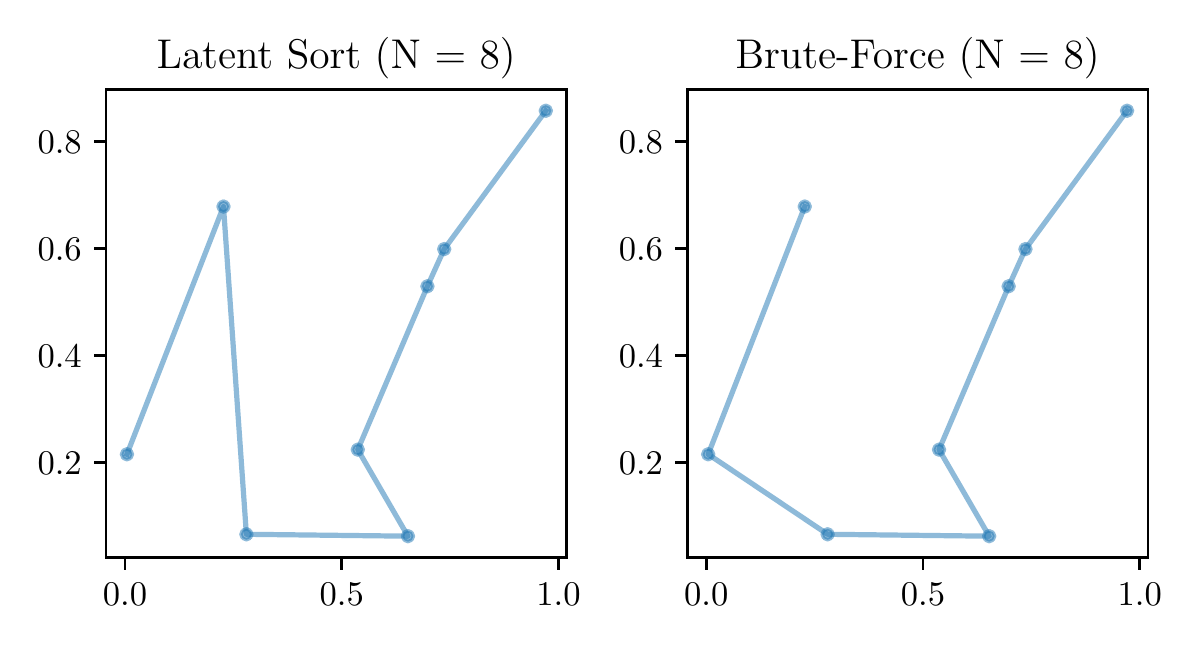}
    \includegraphics[width=0.49\textwidth]{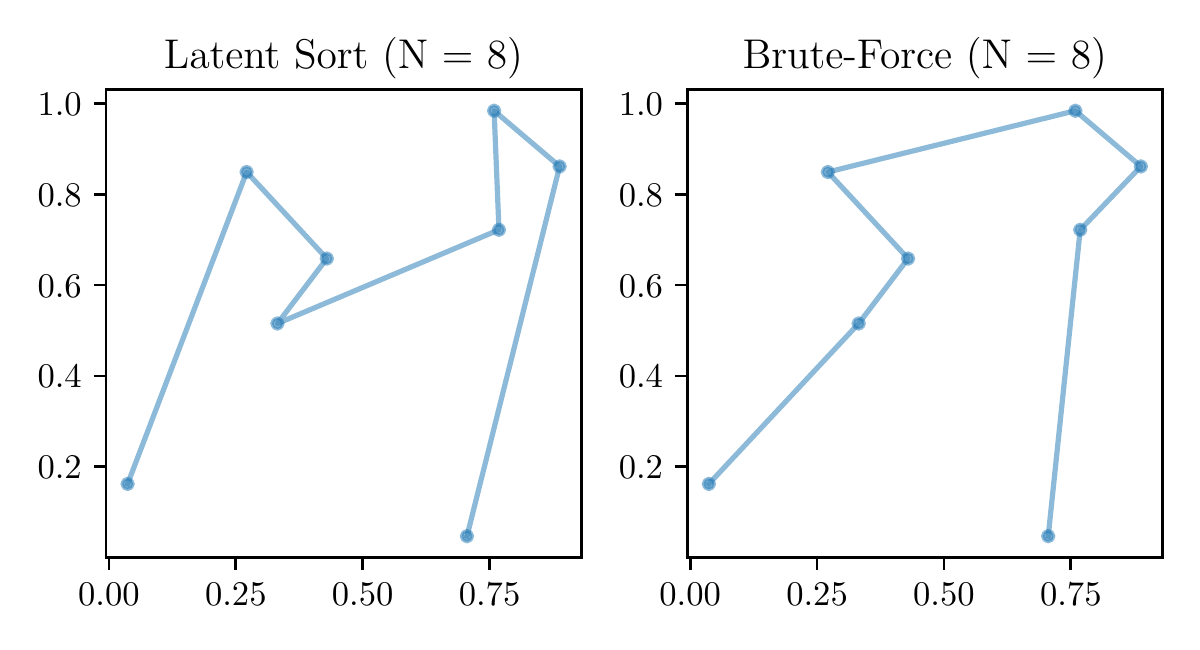}
    \caption{Examples solutions of of TSP problems ($N=8$) with randomly initialized points. It includes approximated solutions using the latent sort algorithm and the exact solutions using brute force search. }
    \label{fig:tsp_approximation}
\end{figure}

The results consistently demonstrate that employing latent sort with the latent gradient penalty (LGP) achieves higher approximation rates compared to using latent sort without LGP. This finding supports the conclusions discussed in Section \ref{sec:shortest_path_property} regarding the effectiveness of LGP in the latent sort algorithm for approximating TSP solutions. Note that the methods in general appear to perform better for a higher number of points, as there are more possible paths in total, leading to a higher percentage of paths being longer than the one predicted by latent sort.

\subsection{Semantic Scene Graph Generation on OpenPSG}

Here we provide some examples of the results (see Figure \ref{fig:additional_scene_graph}) for the semantic scene graph generation task on the OpenPSG dataset. 

\begin{figure}
\centering
\begin{minipage}[b]{\textwidth}
\centering
(a) \includegraphics[width=\textwidth]{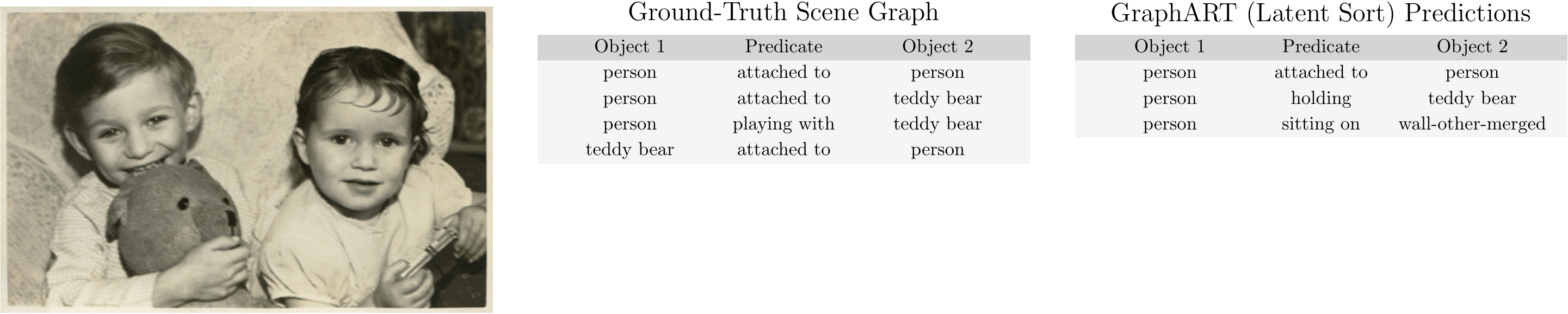}
\end{minipage}
\begin{minipage}[b]{\textwidth}
\centering
(b)\
\includegraphics[width=\textwidth]{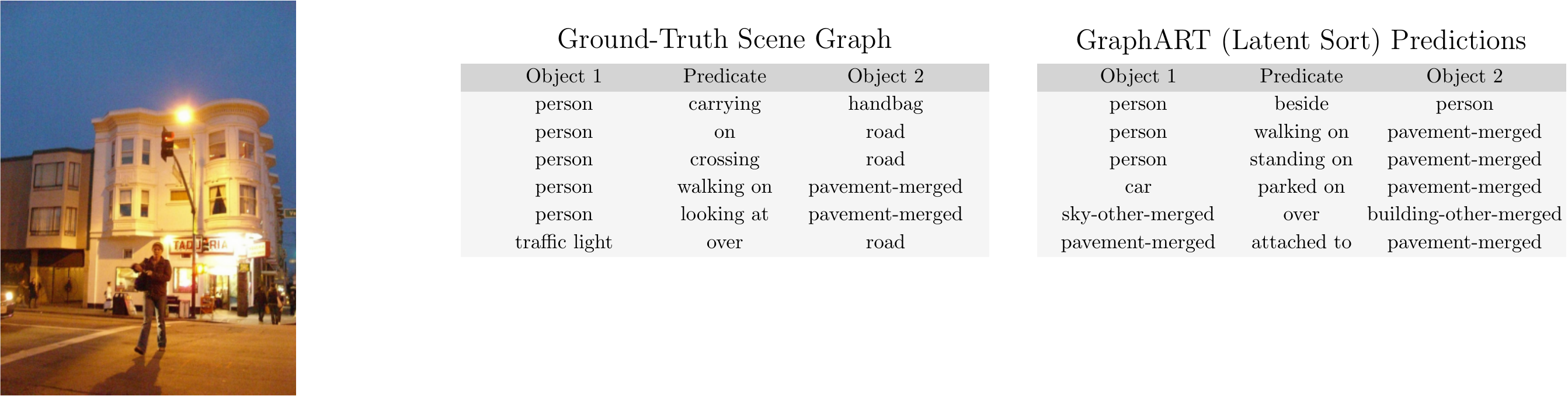}
\end{minipage}
\begin{minipage}[b]{\textwidth}
\centering
(c)\
\includegraphics[width=\textwidth]{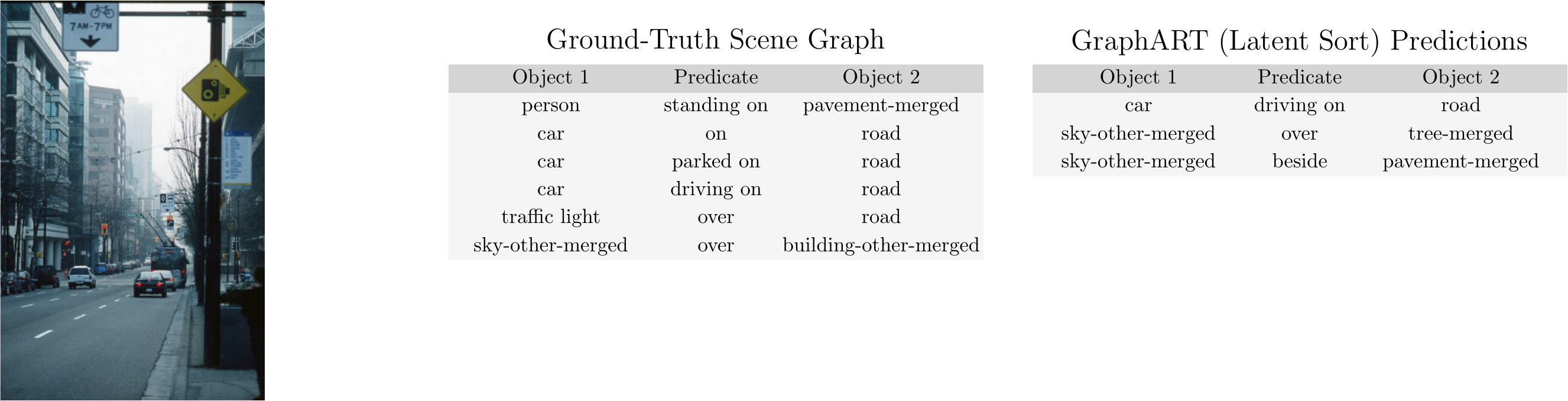}
\end{minipage}
\begin{minipage}[b]{\textwidth}
\centering
(d)\
\includegraphics[width=\textwidth]{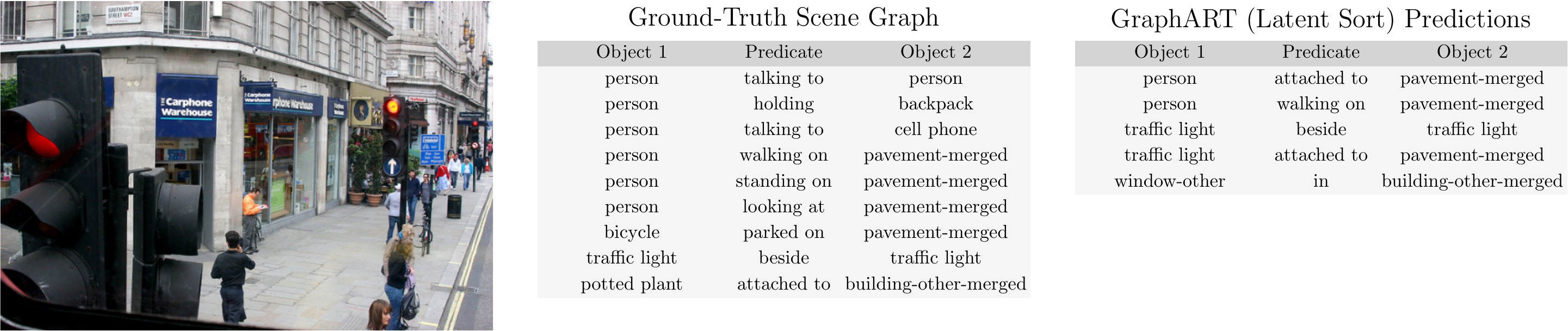}
\end{minipage}
\begin{minipage}[b]{\textwidth}
\centering
(e)\
\includegraphics[width=\textwidth]{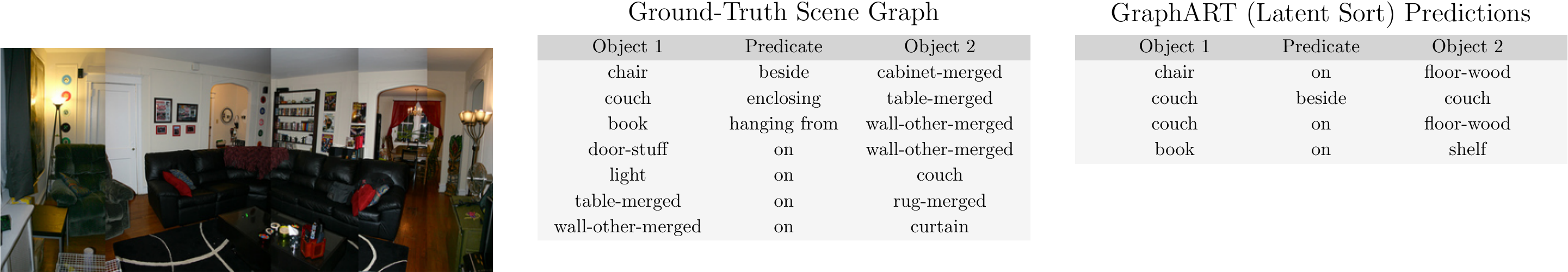}
\end{minipage}
\begin{minipage}[b]{\textwidth}
\centering
(f)\
\includegraphics[width=\textwidth]{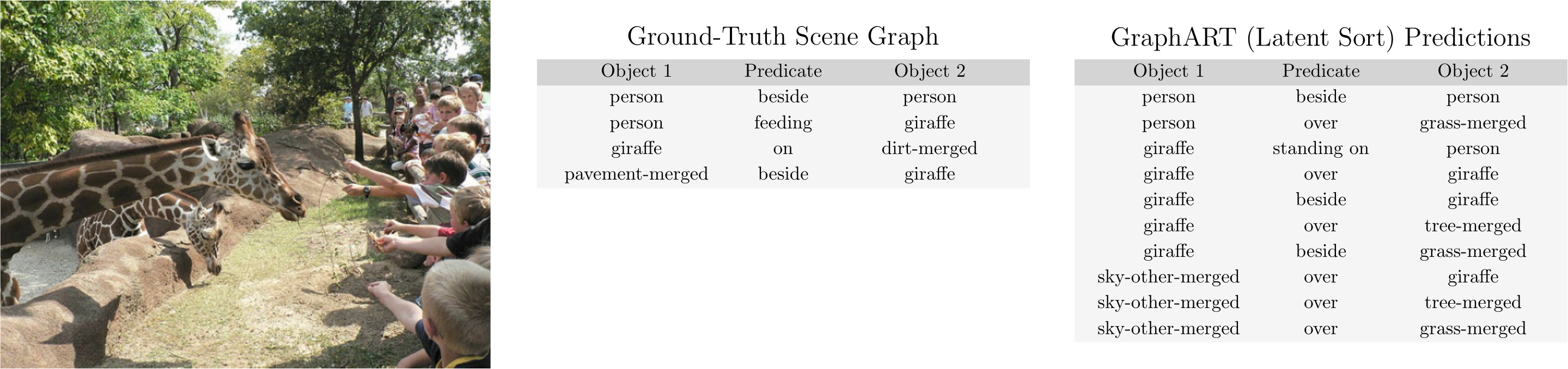}
\end{minipage}
\caption{Randomly sampled results for semantic scene graph generation on the OpenPSG dataset. Each image consists of the input image (left), ground-truth labels (middle), and predictions generated by \graphgpt{} using the latent sort algorithm (right). The samples show the label noises in the dataset.}
\label{fig:additional_scene_graph}
\end{figure}

We can see that \graphgpt{} generally tends to predict slightly smaller scene graphs compared to the ground-truth labels. This model behavior can be attributed to the presence of label noise in the OpenPSG dataset, although the dataset is already known to have less noise compared to other scene graph datasets \cite{yang2022psg}. Due to this issue in the training data, \graphgpt{} becomes cautious in predicting triplets with low confidence. Instead, it shows a preference for the most common objects and predicates, as well as the prevalent patterns present in the dataset. Consequently, the model tends to predict the stop token earlier than expected, resulting in predicted graph sequences that are likely to include only the most common objects and predicates. As a result, the prediction sequences conclude before reaching the size of the ground-truth labels (see Figure \ref{fig:additional_scene_graph}). 

We believe that this issue does not seem to be specific to \graphgpt{} alone, but rather a common problem encountered by autoregressive models when dealing with label noise in the dataset. Interestingly, we observed that the model's predictions are occasionally more reasonable and logically consistent than the ground-truth labels. For example, in Figure \ref{fig:additional_scene_graph}(a), the model predicts \texttt{person - holding - teddy bear}, which is more accurate than \texttt{person - attached to - teddy bear} in the ground-truth labels. In Figure \ref{fig:additional_scene_graph}(d), the model predicts \texttt{window-other - in - building-other-merged}, which is missing from the ground-truth labels, and the labels have a lot of duplicated relations for \texttt{person}. In Figure \ref{fig:additional_scene_graph}(e), the model predicts \texttt{book - on - shelf}, which is more accurate than \texttt{book - hanging from - wall-other-merged} in the ground-truth labels.
We believe that for future research, it is necessary to develop scene graph datasets with cleaner labels in order to enable fair benchmarking and evaluate the performance of different models more accurately.


\section{Additional Information For Experiment Setups}
\label{sec:experiment_details}

In this section, we will provide additional details regarding the experiment setups, as well as provide details of procedures for conducting experiments involving the models and tasks discussed in the paper for full reproducibility of our results. Note that, the amount of details provided here are limited comparing to the \textbf{released code} \footnote{\url{https://anonymous.4open.science/r/ordering-in-graph-gen/}}.

\subsection{Additional Information For Datasets}
\textbf{Scene Graph Generation}: We use the OpenPSG dataset \cite{yang2022psg} which contains 49k annotated images, and we follow the same train-test split as \cite{yang2022psg} , which has around 45k for training and 4k for testing. Our problem setup for scene graph generation is slightly different from most of the existing works on scene graph generation  \cite{lu2021context, li2022sgtr, yang2022psg}.  In particular, we consider semantic-level scene graphs where every node represents a generic object and is not associated to a particular localized (or segmented) object. For example, a ``\texttt{person1 - beside - person2}'' triplet will be converted into ``\texttt{person - beside - person}'' in our problem setup. 
We chose this problem setup because our main focus is to investigate the effect of different ordering methods on graph generation rather than striving for state-of-the-art (SOTA) performances. As such, we avoid complex segmentation settings that could potentially divert our focus. However, this doesn't imply an easier task. In fact, our model must infer the number of objects implicitly without resorting to object-level granularity. This contrasts with methods like PSGTR and PSGFormer used in conventional scene graph generation tasks, where the notion of objects is explicitly provided for training via segmentation masks. 

We adopt the same data preprocessing pipelines used for training PSGTR and PSGFormer, except that we do not use the panoptic segmentation labels for training. We also have a simpler cropping function instead of the sophisticated multi-size cropping used in PSGTR and PSGFormer, and our cropping generates smaller images that result in faster training. More specifically, we pad the input images to $384 \times 384$, where the shorter side will be padded by 0 values to make the image squared (equal aspect ratio). In contrast, the PSGTR and PSGFormer uses $640 \times 384$ or $1333 \times 640$. The difference in input size may be the factor that \graphgpt{} performs worse in identifying small objects. Then, we parse the object-level scene graphs into semantic-level scene graphs involving 133 object classes and 56 predicate classes. While state-of-the-art (SOTA) methods for this problem, such as PSGTR \cite{yang2022psg} and PSGFormer \cite{yang2022psg}, use additional panoptic segmentation labels for training, we do not use this additional supervision in our framework. This is done to demonstrate the capability of autoregressive models to directly generate scene graphs from images.

Note that the central goal of this paper is not to find a superior vision model that achieves state-of-the-art performances in scene graph generation. Thus, we did not fully optimize the pipeline and there are still potentials in \graphgpt{} to reach better performances for this particular task (scene graph generation). We welcome future works to explore the improvement that can be done using \graphgpt{}.

\textbf{Topological Graph Extraction}:  The task is based on the Toulouse Road Network dataset from \cite{belli2019image}, which consists of 99k image-graph pairs. The dataset comprises $64 \times 64$ grayscale images paired with corresponding topological graphs. We follow the same train-test split as described in \cite{belli2019image}, where approximately 80k image-graph pairs are used for training, and 19k pairs are used for testing. The dataset is derived from OpenStreetMap and involves the processing and labeling of raw road network data from the city of Toulouse, France, which was obtained from GeoFabrik and timestamped in June 2017. Each image represents a small cropped patch of the city map, and the coordinates are normalized to [-1, 1]. Upon inspecting sample images from the dataset, we observed that the graph structures in most patches are generally simple, reflecting the inherent simplicity of real-world road networks.

To assess the models' performance under more challenging conditions, we decided to create a synthetic Planar Graph dataset from scratch. This dataset is generated by utilizing Delaunay triangulation and the NetworkX package to generate random planar graphs. The node coordinates are sampled from a uniform distribution $\mathcal{U}(0, 1)^2$. Initially, the graphs contain 15 nodes. Preprocessing involves collapsing nodes that are too close to each other, with a threshold of 0.1. Similarly, edges are collapsed if they form an angle less than 30 degrees. After collapsing nodes and edges, self-loops are removed to obtain the target graph. Subsequently, we utilize the OpenCV package to render the target graph into $256 \times 256$ binary images, with edges represented by lines of random widths. To introduce some variability, random artifacts are added to the images. The graph structures in this dataset are generally more challenging when compared to the Road Network dataset. 
During training, minibatches of images and ground-truth graphs are generated on the fly, resulting in non-repeating dataset (of 40k samples every epoch) and zero generalization gap in the trained model. During testing, we test the performance of the results using 10 random seeds on 1000 generated samples.

\textbf{Circuit Graph Prediction}: We use the Terahertz channelizer dataset\footnote{Please make sure to carefully review and comply with the license permissions and restrictions (if there are any) before using the dataset \cite{he2019circuit}: \url{https://github.com/hehaodele/circuit-gnn}} consisting of 347k circuit graphs. This dataset was introduced in the work by He et al. \cite{he2019circuit}, which presented a novel approach to represent circuits using graphs, where resonators are represented as nodes and electromagnetic couplings as edges. The authors mapped the geometric configuration of a circuit to the graph structure, using node attributes for resonator properties, while edge attributes captured the relative positions, gap lengths, and shifts between resonators. The presence of edges in the graph is determined by a distance threshold based on prior research. We recommend readers to go through the original paper for further details on the dataset, which includes visual examples of the circuits and their graph representations.

In our paper, we consider the goal of solving the ``inverse'' problem, where we are expected to predict the ground-truth circuit graph that consists of 3 to 6 resonators (or nodes), given desirable electromagnetic (EM) properties (e.g., transfer function) as inputs. Each resonator (node) is represented by a 9-dimensional vector. We follow the problem statement of the forward and inverse problems as the original work \cite{he2019circuit}.
Since the scale of some dimensions in the resonator vectors is much larger than others, which may dominate the prediction errors, we normalize the data using min-max normalization per dimension. We report the errors on the normalized data.

\subsection{Model Configurations And Baseline Methods}

\textbf{Baselines in Scene Graph Generation}: For the scene graph generation task (Table \ref{tab:scene_graph_results}), we consider PSGTR \cite{yang2022psg} and PSGFormer \cite{yang2022psg} as state-of-the-art (SOTA) baselines. We ran the evaluation scripts and trained checkpoints as provided by Yang et al \footnote{OpenPSG Github Repo: \url{https://github.com/Jingkang50/OpenPSG}} to obtain the predictions on the test set, where each node in the predicted scene graph is grounded by its pixel-accurate segmentation mask in the image. Then, we convert the segmentation masks based scene graph into semantic scene graph. Since the output graph size of these methods depends on a confidence thresholding hyper-parameter, it is challenging to tune it to match the ground-truth graph sizes. Moreover, these methods are tuned to be biased towards achieving a good recall score, generally producing much larger graphs than the ground-truth (GT) labels. Therefore, we downsampled the output triplets to match the GT numbers and report their average performance over 10 random runs. Note that the PSGTR and PSGFormer utilize panoptic segmentation labels as additional supervision during training, which we do not use in our framework. Additionally, we do not use the GT size of the scene graph as an input parameter in our approach, in contrast to the implementation of PSGTR and PSGFormer.

\textbf{Baselines in Topological Graph Extraction}: For the road network generation task (Table \ref{tab:road_graph_results}), we use the GGT \cite{belli2019image} and GraphRNN \cite{you2018graphrnn} algorithms as baselines. These models are supervised on adjacency matrices using BCELoss, in contrast to our purely autoregressive framework that does not generate graph-specific components. For the planar graph generation task (Table \ref{tab:planar_graph_results}), we designed a non-autoregressive architecture called GraphTR, inspired by the DETR architecture \cite{carion2020end}. GraphTR utilizes bipartite matching loss and does not rely on autoregressive generation.

\textbf{Baselines in Topological Graph Prediction}: Regarding the circuit prediction task (Table \ref{tab:circuit_graph_results}), the original work \cite{he2019circuit} does not directly solve the inverse problem from circuit specifications to circuit graphs. Therefore, we implemented a model adapted from the original design using a graph attention network (GAT) \cite{velivckovic2017graph} and fed it with the ground-truth size of output graphs. We model each circuit as a Graph Attention Network that is initialized to a projection of the input signals. We also initialize the edge attributes to the electromagnetic coupling, which can be seen as a distance measure between the nodes. Our task is to predict the nine ``raw'' node features for each node of the circuit graphs. We focus on only predicting the node features as the original work \cite{he2019circuit} derives both node attributes and edge attributes from these node features. Note that the CircuitGNN in the original work does not have the capability to predict the target graph size (the number of resonators in the circuit) for a given input, therefore we train different models for different graph sizes, specifically one model for one size, as handled by the original work \cite{he2019circuit}. 

\textbf{\graphgpt{} And \graphrnn{}}:
We have different model configurations for each task included in the paper, and \graphrnn{} use the same configurations as \graphgpt{} for all the tasks except the decoders are LSTM models instead of autoregressive transformers. Unless specified, the dropout or drop-path probabilities are set to 0.

\begin{figure}[ht]
    \centering
    \includegraphics[width=\textwidth]{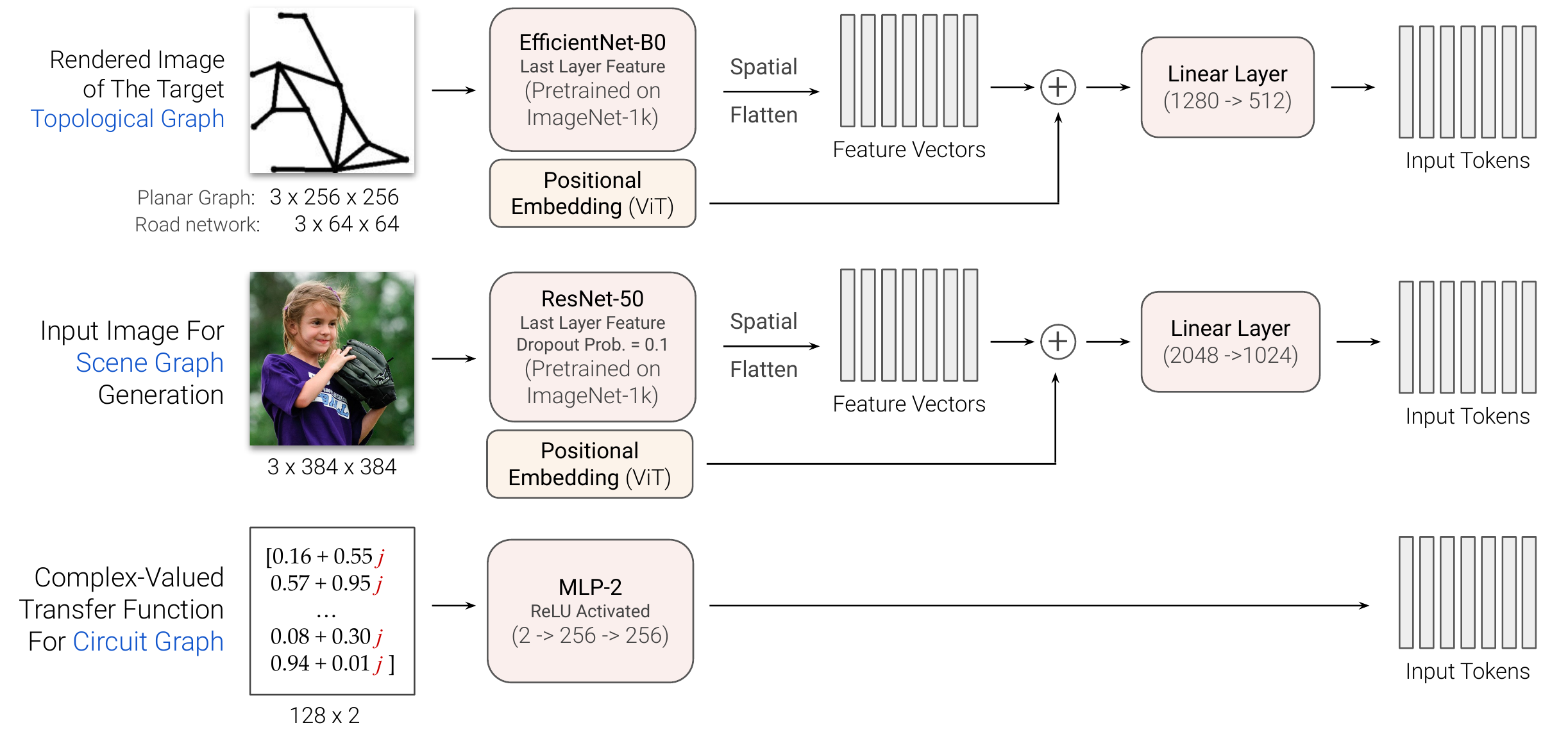}
    \caption{The encoder architecture of the \graphgpt{} and \graphrnn{} for different tasks.}
    \label{fig:ap_encoder_details}
\end{figure}
\begin{figure}[ht]
    \centering
    \includegraphics[width=\textwidth]{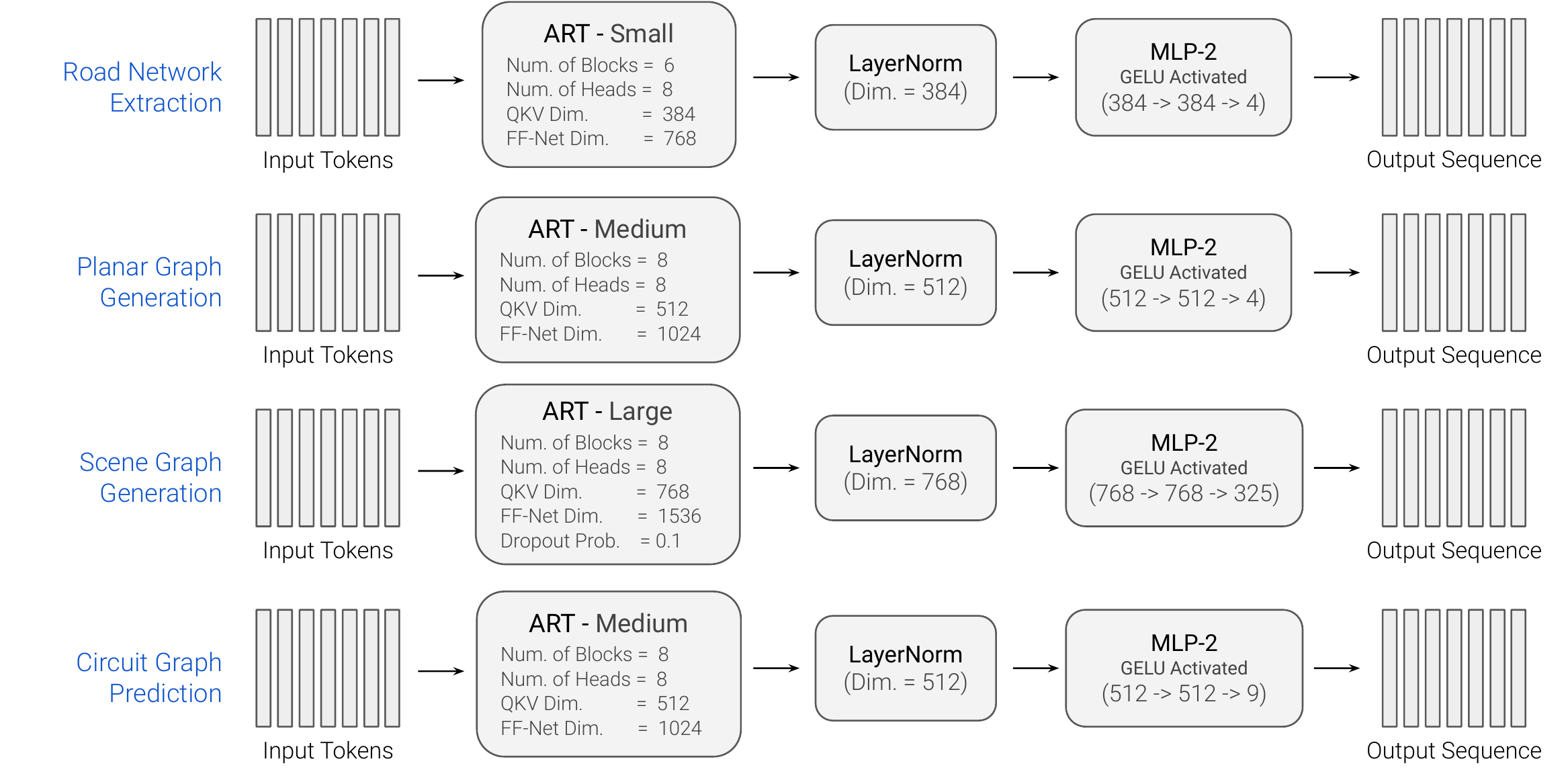}
    \caption{The decoder architecture of the \graphgpt{} for different tasks. The autoregressive transformer (ART) is based on GPT-2 architecture \cite{radford2019language} with our modified decoder block architecture (see Figure \ref{fig:decoder_block}).}
    \label{fig:ap_graphart_decoder}
\end{figure}
\begin{figure}[ht]
    \centering
    \includegraphics[width=\textwidth]{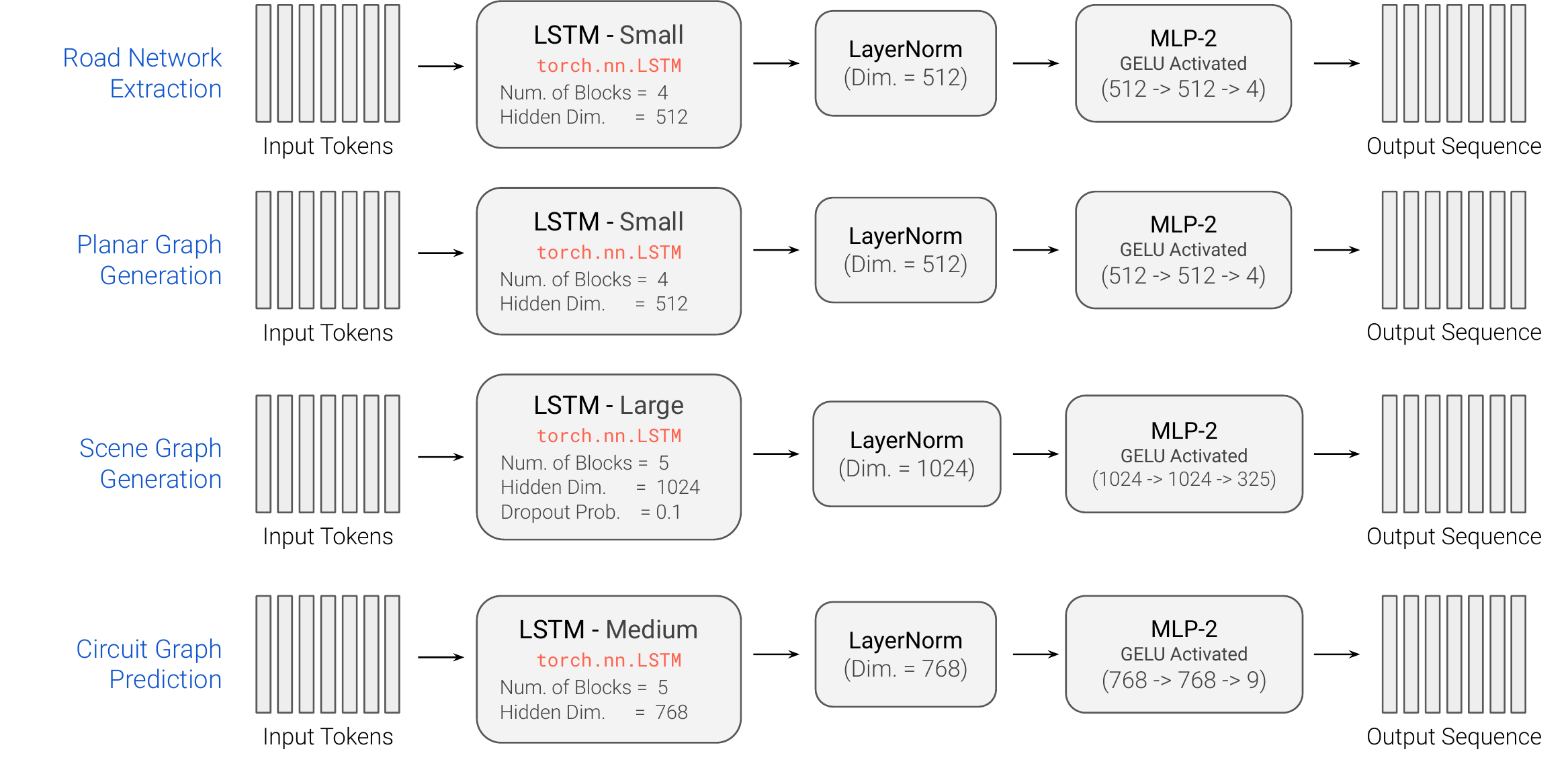}
    \caption{The decoder architecture of the \graphrnn{} for different tasks.}
    \label{fig:ap_graphrnn_decoder}
\end{figure}

The model configurations include the following details:
\begin{itemize}
    \item \textbf{Visual Encoders}: Figure \ref{fig:ap_encoder_details} presents the common encoder architecture for both \graphgpt{} and \graphrnn{}. These models utilize either an Efficient-B0 \cite{tan2020efficientnet} or a ResNet-50 \cite{he2015deep} pretrained on the ImageNet-1k dataset \cite{russakovsky2015imagenet} as the visual backbone network. The input images are encoded into feature vectors after flattening the spatial dimensions of the output feature map. The positional embeddings used in the original Vision Transformer (ViT) \cite{dosovitskiy2021image} are adopted.
    \item \textbf{Autoregressive Decoders}: Figure \ref{fig:ap_graphart_decoder} and \ref{fig:ap_graphrnn_decoder} illustrate the detailed architectures for \graphgpt{} and \graphrnn{}, respectively. The \graphgpt{} model employs a GPT-2-based architecture with stacked transformer decoder blocks, as depicted in Figure \ref{fig:decoder_block}. During training, teacher-forcing is used by feeding the ground-truth sequence to the model. During inference, the model performs iterative next token prediction until a stop criterion is met. For road network extraction, planar graph generation, and circuit graph prediction, stop tokens are represented by vectors with all $-1$ values. If the mean absolute distance of a generated token to the vector of all $-1$ values is less than 0.8, it is considered a stop token, leading to the termination of iterative sequence generation.
\end{itemize}

We try to control the number of parameters to be similar among different models for the same task.
Please refer to the figures for a visual representation of the mentioned architectures.

\subsection{Additional Information For Evaluation Metrics}
To evaluate the performance of graph generation results in relation to the ground truth, it is appropriate to employ metrics that quantify the dissimilarity between two sets. For the topological graph generation tasks, we employ the StreetMover distance as was done in   \cite{belli2019image}. In addition, drawing inspiration from the Earth Mover's distance and Hausdorff distance, we introduce the Edge Mover Distance (EMD) and Edge Hausdorff distance (EHD) metrics specifically tailored for the "edge-based tokenization" scheme used in our models (refer to tokenization in Section \ref{sec:models_and_pipelines}). 

\textbf{StreetMover distance (SMD)}: The SMD is proposed by \cite{belli2019image} as an evaluation metric for evaluating generative models in the context of road networks. It addresses the limitations of existing metrics by jointly capturing the accuracy of the reconstructed graph while being invariant to changes in graph representation, transformations, and size. Unlike pixel-based metrics, SMD considers the global alignment and magnitude of errors in the reconstructed graphs, making it more suitable for evaluating road network generation. It overcomes the drawbacks of other metrics, such as Average Path Length Similarity (APLS), which require post-processing steps and are designed for different purposes. SMD is easily interpretable and computationally efficient, leveraging Sinkhorn iterations for its computation.

\textbf{Edge Mover Distance (EMD) And Edge Hausdorff Distance (EHD)}:

Given two finite length sequences of N-dimensional points $X$ and $Y$ (where $X$ and $Y$ are matrices of N columns), we denote $\mathcal{X}$ and $\mathcal{Y}$ as the corresponding unordered sets of $X$ and $Y$. In other words, $\mathcal{X}$ and $\mathcal{Y}$ are sets that contain the rows of the corresponding matrices, where each row represents an N-dimensional point. Additionally, $\mathcal{X}$ and $\mathcal{Y}$ are sets that contain all the rows, and every element within these sets is a row from the matrices.

The EMD between two sequences $X$ and $Y$ (represented by matrices) can be defined using its unordered sets $\mathcal{X}$ and $\mathcal{Y}$:
\begin{equation}
\text{EMD}(\mathcal{X}, \mathcal{Y}) = \frac{1}{|\mathcal{X}|} \sum_{\bm{x} \in \mathcal{X}} \min_{\bm{y} \in \mathcal{Y}} \rho(\bm{x}, \bm{y}),
\end{equation}
and the EHD can be defined as:
\begin{equation}
    \ehd(\mathcal{X}, \mathcal{Y}) = \max \left[ \sup_{\bm{x} \in \mathcal{X}} \inf_{\bm{y} \in \mathcal{Y}} \rho(\bm{x}, \bm{y}), \sup_{\bm{y} \in \mathcal{Y}} \inf_{\bm{x} \in \mathcal{X}} \rho(\bm{x}, \bm{y}) \right],
\end{equation}
where $|\mathcal{X}|$ is the cardinality of set $\mathcal{X}$ representing the number of points in $\mathcal{X}$. The function $\rho(\bm{x}, \bm{y})$ represents the distance between points $\bm{x}$ and $\bm{y}$. This distance can be calculated using various metrics, such as the Euclidean distance or any other suitable distance metric for the given problem domain.
The first term $\sup_{\bm{x} \in \mathcal{X}} \inf_{\bm{y} \in \mathcal{Y}} \rho(\bm{x}, \bm{y})$ computes the maximum distance from any point in $\mathcal{X}$ to its nearest neighbor in $\mathcal{Y}$, while the second term $\sup_{\bm{y} \in \mathcal{Y}} \inf_{\bm{x} \in \mathcal{X}} \rho(\bm{x}, \bm{y})$ computes the maximum distance from any point in $\mathcal{Y}$ to its nearest neighbor in $\mathcal{X}$. The EMD calculates the average of the minimum distances between each point in set $\mathcal{X}$ and its nearest neighbor in set $\mathcal{Y}$, whereas the EHD measures the maximum distance between the closest points of two sets $\mathcal{X}$ and $\mathcal{Y}$. The EHD is the larger of these two values, representing the maximum discrepancy between the two sets. 

\textbf{Classification Metrics For Sequences And Sets of Different Sizes}:

In traditional classification settings, precision, recall, and F1 score are typically computed on a per-instance basis. Each individual item (or instance) in the prediction set is compared to its corresponding item in the true set, and the counts of true positives, false positives, and false negatives are aggregated over all instances, often using "macro" averaging. This approach works well when the prediction set and true set have an equal number of instances. However, in situations where the prediction set and true set have a different number of instances, the traditional per-instance calculation becomes inapplicable. It is crucial to establish new definitions for these metrics that can appropriately handle the discrepancy in set sizes.

We use the same definition of $\mathcal{X}$ and $\mathcal{Y}$ as unordered sets and  $X$ and $Y$ as predicted and ground-truth sequences. To calculate TP (True Positives), we compute the cardinality of the intersection of $\mathcal{X}$ and $\mathcal{Y}$, which represents the count of matches or true positives. Specifically, we express TP, FP, and FN as follows:
\begin{align*}
TP & = |\mathcal{X} \cap \mathcal{Y}| \\
FP & = |\mathcal{X}| - TP \\
FN & = |\mathcal{Y}| - TP
\end{align*}
where $|\cdot|$ represents the cardinality of a set. Using this new definition of TP, FP and FN, we define Precision, Recall, and F1 Score between two sets of different sizes (cardinality) as follows:

\textbf{Precision} is the proportion of correct predictions (TP) to the total number of predictions in the predicted sets. It reflects the percentage of predictions that are correct. Mathematically, Precision is defined as:
\begin{equation}
Precision = \frac{TP}{TP + FP} = \frac{|\mathcal{X} \cap \mathcal{Y}|}{|\mathcal{X}|}
\end{equation}
    \textbf{Recall} is the proportion of correct predictions (TP) to the total number of elements in the targe set. It reflects the percentage of ground-truth elements that are included in the predictions. Mathematically, Recall is defined as:
\begin{equation}
Recall = \frac{TP}{TP + FN} = \frac{|\mathcal{X} \cap \mathcal{Y}|}{|\mathcal{Y}|}
\end{equation}
Similar to the traditional definition, the \textbf{F1 Score} for sets is the harmonic mean of Precision and Recall. The F1 Score gives equal weight to both Precision and Recall and provides a balanced measure of model performance. High Precision and high Recall will yield a high F1 Score. The F1 Score is defined as:
\begin{equation}
F1 = \frac{2 \times Precision \times Recall}{Precision + Recall} = \frac{2 \times \frac{|\mathcal{X} \cap \mathcal{Y}|}{|\mathcal{X}|} \times \frac{|\mathcal{X} \cap \mathcal{Y}|}{|\mathcal{Y}|}}{\frac{|\mathcal{X} \cap \mathcal{Y}|}{|\mathcal{X}|} + \frac{|\mathcal{X} \cap \mathcal{Y}|}{|\mathcal{Y}|}}
\end{equation}

Using the above definitions, there are precision, recall and F1 scores for each pair of predicted and ground-truth graphs. We report the averaged values over the entire test set.

\subsection{Training Setups For Graph Generation Models}
The configurations for the graph generation models are described as follows:
\begin{itemize}
    \item \textbf{Planar Graph (Table \ref{tab:planar_graph_results}) \& Road Network (Table \ref{tab:road_graph_results})}: We train the models using the Adam (planar graph) or AdamW (road network) optimizer with the following settings: learning rate = $1.0 \times 10^{-4}$, weight decay = 0 (planar graph) or $1.0 \times 10^{-4}$ (road network), and an effective batch size = 300 (150 on 2 GPUs). The warm-up period is set to $10\%$ of the total epochs, and we employ "cosine" annealing. Sequences are padded to the maximum length of sequences plus one within a batch, using a padding value of -1. The stop token is also represented by the value -1. Training 1000 epochs on the planar graph dataset typically takes approximately 16 to 20 hours on two NVIDIA A100 GPUs with 128 CPU cores, using 32-bit precision\footnote{We observed degradation in model accuracy when using 16-bit precision in training.}, and it takes around 6 to 8 hours to train for 200 epochs on the road network dataset using a single NVIDIA A100 GPU with 16 CPU cores.
    \item \textbf{Scene Graph (Table \ref{tab:scene_graph_results})}: We train the models for 20 epochs using the AdamW optimizer with the following settings: learning rate = $1.0 \times 10^{-4}$, weight decay = $5.0 \times 10^{-4}$, and an effective batch size of 100. The warm-up period is set to $10\%$ of the total epochs, and we employ "cosine" annealing. Sequences are padded within a batch to the maximum length of sequences plus one, with the addition of stop tokens. For each object and their predicate that are one-hot encoded, an additional dimension is added for the "stop token" class, resulting in a total of three additional dimensions. Training typically takes approximately 2 hours on a single NVIDIA A100 GPU with 128 CPU cores, using 32-bit precision.
    \item \textbf{Circuit Graph (Table \ref{tab:circuit_graph_results})}: We train the models using the AdamW optimizer with the following settings: learning rate = $1.0 \times 10^{-4}$, weight decay = $1.0 \times 10^{-4}$, and an effective batch size = 300 (150 on 2 GPUs). The warm-up period is set to $10\%$ of the total epochs, and we employ "cosine" annealing. Sequences are padded to the maximum length of sequences plus one within a batch, using a padding value of -1. The stop token is also represented by the value -1. Training 200 epochs typically takes approximately 20 to 24 hours on two NVIDIA A100 GPUs with 64 CPU cores, using 32-bit precision.
\end{itemize}

\subsection{Training Latent Sort Encoders}
We trained 10 latent sort encoders for each experiment and selected the one that yielded the best validation performance for the paper. We observed that the difference in graph generation performance due to different random initializations of the latent sort encoders is less significant than the impact of the hyperparameters used in training the encoders. Based on our observations, here are some key takeaways to guide the design of latent sort encoders:
\begin{itemize}
    \item We found that different random initializations of the latent sort encoders can lead to less than a $5\%$ difference in graph generation performance (measured in terms of EMD values using \graphgpt{}).
    \item Larger latent sort encoders do not necessarily result in better graph generation performances, even if they have lower reconstruction errors. In our experiments, we utilized multilayer perceptrons as encoders and decoders, each consisting of three hidden layers with 512 neurons and tanh activation..
    \item The most important hyperparameters we identified include: the coefficient of LGP in the loss function, the strength of $L_2$ regularization, and the number of training epochs. The trade-off parameter controlling the strength of LGP regularization is the most important hyperparameter, and we used a range of 0.01 to 0.1 in our experiments. We found that $L_2$ regularization was unnecessary (thus set to 0) when using LGP. Generally, longer training epochs do not lead to a degradation in graph generation performance.
    \item Learning rate scheduling plays a crucial role in the performance of latent sort encoders. In our experiments, we employ Adam optimizers and cosine learning rate annealing. The model undergoes a warm-up phase at the beginning of training, lasting $10\%$ of the total training epochs. During this phase, the learning rate is exponentially increased from an initial value of $1.0 \times 10^{-5}$. Subsequently, the learning rate gradually decreases until it reaches a final value of $1.0 \times 10^{-8}$.
    \item We observe that using the undirected graph loss as reconstruction does not produce significant differences when compare to the L1 reconstruction loss for the topological graph generation tasks. For ciruit graph generation, we use L1 loss for reconstruction. For the scene graph generation task, we use binary cross entropy (BCE) loss as used in training the graph generation models.
\end{itemize}

\section{Supplementary Theoretical Results And Discussions}
\label{sec:ap_additional_theory}

This section supplements the theoretical results presented in Section \ref{sec:sorting_as_dr} by providing a more detailed and comprehensive explanation of the overall framework. The structure of this section is as follows:
\begin{enumerate}[label=\textbf{C.\arabic*}]
    \item includes the definition and motivations of sorting algorithms from the perspective of dimensionality reduction (DR), along with other important notations. It also includes a discussion of how sequence orderings affect model behavior.
    \item  introduces the definition of ordering ambiguity for sorting algorithms.
    \item  discusses the expected prediction errors caused by ordering ambiguity.
    \item  introduces the idea of the latent sort algorithm. It also includes two theorems to reveal some of its properties.
    \item  discusses prediction errors from imperfection reconstruction in the latent sort algorithm. First, it connects the reconstruction error (assuming normal distribution) with the error distribution of the 1-D latent values. Theorem \ref{th:error_latent_normal_distribution} shows how to estimate the 1-D latent distribution (in terms of mean and variance) from the distribution of reconstruction errors. Then, it discusses the ordering errors that may be caused by imperfection reconstruction.
    \item  provides a way to simplify the ordering errors from imperfection reconstruction.
    \item  discusses the desirable properties of an optimal ordering to minimize the errors from imperfection reconstruction using the simplified expression of errors from C.6.
\end{enumerate}

\subsection{Sorting As A Dimensionality Reduction Problem}

\textbf{The Ordering Problem in Set Generation}: Since we tokenize each graph into a set of $N$-dimensional tokens ($N > 1$), we denote a graph as a set $\mathcal{X}$ consisting of $M$ unique points (we also refer to these as vector tokens) within an $N$-dimensional space. These points are unordered, which means they don't follow a specific sequence or pattern. 

Our goal is to discover the best possible (or optimal) ordering for these points, which when used as the target sequence to supervise an autoregressive model yields optimal performance in generating $\mathcal{X}$. We denote the sequence obtained by applying the optimal ordering as $Y^*$. To visualize this, imagine $Y^*$ as a matrix where every row represents a point in $\mathcal{X}$. Each point is assigned an index which signifies its position in this optimal order.

Finding $Y^*$ is very challenging since we do not know what orderings the autoregressive models ``prefer'' that yield optimal results. In fact, different model architectures or different applications can have different optimal orderings. 
For example, in an RNN, each output at time $t$ depends on the previous outputs and the current input, forming a sequence of dependencies (i.e., a chain-like structure). This means information has to flow sequentially from one step to the next. In contrast, the transformer model's attention mechanism allows it to have direct dependencies between all pairs of positions in the sequence. Therefore, information can flow directly between any two positions in a sequence, which allows transformers to capture longer dependencies more efficiently.
As a result, the optimal ordering of input elements for an RNN may show temporal or spatial continuity, given its inherent sequential processing nature, while a Transformer may not strictly require such continuity due to its capability to process all elements concurrently and directly. Moreover, the type of application also plays a vital role in determining the optimal ordering. Therefore, our goal isn't simply to find an ordering that fits all models, but rather to discover an ordering that is best suited to the model's architecture and the specific use-case it is being applied to. The optimal orderings could also potentially be influenced by factors like the data distribution, the complexity of the learning task, and even the particular optimization algorithms used during the training process.

\textbf{How About End-to-End Learning Based Ordering?} 

One might naturally consider employing an end-to-end learning approach, in which the process of identifying an optimal ordering is integrated directly into the training of the model. This approach could, for example, involve a Transformer model, known for its permutation invariance at the encoder stage when no positional embeddings are used. However, at the decoder stage, it could be designed to attend to a learnable positional embedding, essentially ``learning'' the optimal sequence ordering as part of the training process. This can be an appealing idea since the model is potentially capable of dynamically adapting to the best ordering for the given data and task.

However, there's a significant challenge when employing such an end-to-end training scheme. Neural networks, including both the ordering model and the autoregressive model, tend to have a critical learning phase \cite{achille2019critical} in the early epochs, where they learn essential features and structures of the data. In an end-to-end setup, the ordering model is consistently changing the sequence ordering to try to best suit the autoregressive model, while the autoregressive model is simultaneously trying to adapt to the new ordering provided by the ordering model. This dual adaptation can create issues with convergence, since we cannot guarantee the two models will ultimately "agree" with each other. This can result in a chaotic learning process where the two models are continuously trying to adapt to each other's changes without reaching a stable state. Indeed, in our experiments, we have observed that such end-to-end learning systems often fail to converge, which motivated us to discontinue them in our work.

\textbf{Why We Opt For Dimensionality Reduction (DR)?}

A natural question to ask is, how is sorting related to Dimension Reduction (DR)? We argue that, although sorting in a differentiable manner is challenging, finding differentiable methods for DR is more straightforward. This offers new opportunities for creating learning-based methods to order points in high-dimensional spaces.

In a simpler scenario, let's say we are working with 1-D representations of all points in $\mathcal{X}$, which we denote as $\mathcal{H}$, and have normalized them to a range between 0 and 1. This allows us to define the problem of finding the optimal ordering as discovering the optimal DR mapping.

Consider a sorting algorithm, $s$, which takes $\mathcal{X}$ and transforms it into $Y$, the ordered sequence of the points. We can represent this sorting algorithm with a DR mapping, $f$, which takes points in the original high-dimensional space $\bm{x}_i$ and reduces it to their 1-D value, $h_i$. 

Note that this DR mapping, $f$, assigns each token $\bm{x_i}$ a specific position, $h_i$, in the latent space, regardless of the combinations of tokens it appears with in $\mathcal{X}$. Therefore, the task of ordering any arbitrary combination of tokens in $\mathcal{X}$ essentially boils down to finding a sorted path traversal in their 1-D latent space representations.

\textbf{Why Local DR Instead of Non-Local DR?}

We broadly classify DR methods that can be used for sorting into two categories: local and non-local. Local DR methods compute the DR mapping based solely on the value of an individual token, independent of other tokens in the set. In contrast, non-local DR methods involve the interaction between different tokens in the set. Non-local methods such as Convolutional Neural Networks (CNNs) or autoregressive Transformer models inherently assume some form of structure or relationships among tokens when performing DR.

For the problem at hand, we do not have any prior knowledge about the best ordering of tokens. Therefore, assuming any form of relationship among tokens, as non-local DR methods do, might not be appropriate. For example, in the case of 1D-CNNs, the 1-D representation depends on the arrangement of neighboring tokens. Since we do not know the correct ordering, deciding what tokens should be considered neighbors can be ambiguous and arbitrary. This ambiguity will then be reflected in the 1-D representations, leading to potential inconsistencies and inaccuracies.
Similarly, using autoregressive Transformers for DR involves an even greater level of complexity since every token can potentially influence every other token. However, unlike standard Transformers, the autoregressive nature of these models limits the flow of information to a directional sequence, which complicates the global interaction of tokens when the correct ordering is unknown. This suggests an implicit assumption of a global sequence structure among all tokens in the set. Yet, without knowing the correct ordering of tokens, determining this structure becomes a chicken-and-egg problem.

There is one notable exception to the aforementioned issue with non-local DR methods: the self-attention mechanism of Transformers without using positional encoding. It treats the token set as a whole, without assuming any specific order or structure, thanks to its full-attention property. This full-attention mechanism, which allows the model to attend to all tokens in the set, is the key to the Transformer's permutation invariance.
However, the full-attention mechanism has one significant drawback when used in the context of autoregressive models. The full-attention mechanism is inherently bidirectional, meaning that it allows for information flow from both "left" and "right" tokens, effectively modeling $P(\bm{x_i}| \bm{x_1}, \bm{x_2}, ..., \bm{x_M})$. This is contrary to the autoregressive models that we use for graph generation, which only attend to "left" side tokens and learn $P(\bm{x_i}| \bm{x_1}, \bm{x_2}, ..., \bm{x_{i-1}})$ \cite{vinyals2015order}.
The bidirectional nature of the full-attention mechanism means that it has access to future information that is unavailable to the autoregressive models. Consequently, the autoregressive models cannot learn the ordering rule represented by the full-attention Transformer, as they are inherently unable to incorporate ``right" side information. This discrepancy between the two types of models poses a significant challenge in aligning the orderings learned by them, and thus limits the utility of using a full-attention Transformer for ordering in the context of training autoregressive models.

In contrast to all the non-local DR methods mentioned above, local DR methods like MLPs do not require any assumptions about relationships among tokens. The 1-D representations are computed based solely on the properties of individual tokens, which are the only certain entities in our context. Therefore, in our scenario where the correct ordering of tokens is unknown, local DR proves to be a more reliable and less assumptive approach for obtaining 1-D representations of tokens.

\textbf{Restated Definition of the Ordering Problem}:
Given a set of $M$ unordered points (or vector tokens) $\mathcal{X} \subseteq \mathbb{R}^N$,  $|\mathcal{X}| = M$ in an $N$-dimensional space ($N > 1$), we are interested in finding the ``optimal'' ordering of points in $\mathcal{X}$, denoted by $Y^* \in \mathbb{R}^{M \times N}$, which when used as the target sequence to supervise an autoregressive model yields optimal performance in generating $\mathcal{X}$.
Formally, let us denote the ordered sequence $Y^*$ as the matrix $\left[\bm{x_1}^*, \bm{x_2}^*, ..., \bm{x_M}^* \right]^{\intercal}$, where every row $\bm{x_i}^*$ of $Y^*$ is a point in $\mathcal{X}$ and $i$ denotes its sorted index in the optimal ordering. For now, we assume that such an ordering exists for every set $\mathcal{X}$, and we will discuss the desired property of it in Section \ref{sec:ap_optimal_sort}. For simplicity, we assume the 1-D representations of all the points in $\mathcal{X}$, denoted by $\mathcal{H}$, are normalized to $[0, 1]$. 

\defSorting*

We can thus formulate the problem of finding the optimal sorting $s^*: \mathcal{X} \rightarrow Y^*$ as finding an optimal DR mapping $f^*: \mathbb{R}^N \rightarrow \mathbb{R}$, $h_i^* = f^*(\bm{x_i}^*)$, and $\forall i \leq j \leq M, f^*(\bm{x_i}^*) \leq f^*(\bm{x_j}^*)$.

\textbf{Modeling the Errors of a Sorting Algorithm:} Assuming we know the target sequence $Y^*$ and we only want to measure the alignment between a sequence $Y$ and $Y^*$, we introduce a probability matrix $P \in \mathbb{R}^{M \times M}$, where each entry $p_{ij}$ within $P$ is designed to represent the probability of $\bm{x_i}$ ($i$-th element in $Y$) being $\bm{x_j}^*$ ($j$-th element in $Y^*$). 
In this way, the expected value of the ordering $Y$, denoted by $\mathbb{E}[Y]$, can be computed as $PY^*$. The matrix $P$  serves as a ``soft" permutation matrix that transforms $Y^*$ into $Y$. This error representation is generic for analyzing different types of errors that may emerge in a sorting algorithm.

For simplicity, we use Frobenius norm along with the $P$ matrix representation: 
\begin{equation}
\mathcal{E}\Big(\mathbb{E}[Y], Y^*\Big) = \big\Vert \mathbb{E}[Y] - Y^* \big\Vert_F^2 = \big\Vert PY^* - Y^* \big\Vert_F^2,
\end{equation}
where $\Vert \cdot \Vert_F$ denotes the Frobenius norm. 
We show that the minimization over $\mathcal{E}\big(\mathbb{E}[Y], Y^*\big)$ can be solved by an equivalent problem of minimizing $\Vert P - I_M \Vert_F^2$ (see Lemma \ref{lemma:optimization_equivalence}). 

\subsection{Definition of The Ordering Ambiguity}

We refer to the problem where multiple points in $\mathcal{X}$ are assigned the same 1-D representation value as \textit{ordering ambiguity}. It is problematic since the ordering of tokens with the same 1-D latent value becomes undefined if no further tie-breaking schemes are defined.

\textbf{Why Ordering Ambiguity Is Almost Inevitable to Local DR?}
As we discussed earlier, local DR methods are more suitable for sorting tokens used for training the autoregressive models. 
Local DR methods are characterized by the computation of dimensionality reduction solely based on the value of an individual token, without considering the context or interactions among other tokens within the set. Since high-to-low-dimensional mappings are mostly surjective (each element in the target space has a pre-image in the domain) and not bijective (one-to-one correspondence), it is very likely that multiple N-D points share the same 1-D representation value (see Remark \ref{remark:bijective_dimensionality_reduction}). This property is a fundamental source of ordering ambiguity, making it an almost inevitable occurrence in local DR methods. Non-local DR methods, on the other hand, can dynamically adapt to alterations in other data points, thereby circumventing collisions in the low-dimensional space. This adaptability can potentially prevent the occurrence of ordering ambiguity. 

The ordering ambiguity of a sorting algorithm can be characterized by an ordering ambiguity set, which is defined as follows:  
\defSortAmbiguity*

\subsection{The Errors From Ordering Ambiguity}
\label{sec:ap_sorting_ambiguity}

In the presence of ordering ambiguity, the resulting sorted sequence $Y$ may not be an accurate representation of the target sorted sequence $Y^*$. We can use the above $P$ matrix to quantify the error introduced by ordering ambiguity. 

Assuming that $Y = \mathbb{E}[Y] = Y^*$ in the absence of ordering ambiguity, we assume the autoregressive model will converge to the expectation of the ambiguity set $\mathcal{A}_i$ for each point $\bm{x_i}$ in $\mathcal{A}_i$, i.e.,
\begin{equation}
    \bm{x_i} = \mathbb{E}_{\bm{x_j} \in \mathcal{A}_i}\left[ \bm{x_j} \right],
\end{equation}
When ordering ambiguity is present, we assume the probabilities are spread  over the members of the corresponding ambiguity set $\mathcal{A}_i$ uniformly, i.e.,
\begin{equation}
\label{eq:p_value_latent_ambiguity}
p_{ij} = \begin{cases}
\frac{1}{|\mathcal{A}_i|} & \text{if} \ \bm{x_j}^* \in \mathcal{A}_i \\
0 & \text{otherwise}
\end{cases}
\end{equation}
Therefore we have:
\begin{equation}
\mathbb{E}[Y] - Y^* = PY^* - Y^* = \begin{bmatrix}
\left(\frac{1}{|\mathcal{A}_1|} \sum_{\bm{x_j}^* \in \mathcal{A}_1} \bm{x_j}^* \right)^\intercal - {\bm{x_1}^*}^\intercal \\
\left(\frac{1}{|\mathcal{A}_2|} \sum_{\bm{x_j}^* \in \mathcal{A}_2} \bm{x_j}^* \right)^\intercal - {\bm{x_2}^*}^\intercal \\
\vdots \\
\left(\frac{1}{|\mathcal{A}_M|} \sum_{\bm{x_j}^* \in \mathcal{A}_M} \bm{x_j}^* \right)^\intercal - {\bm{x_M}^*}^\intercal
\end{bmatrix}
\end{equation}

Now we can calculate the error $\mathcal{E}\big(\mathbb{E}[Y], Y^*\big)$ by taking the squared Frobenius norm of the difference between $PY^*$ and $Y^*$:
\begin{align}
\mathcal{E}\Big(\mathbb{E}[Y], Y^*\Big) &= \Vert PY^* - Y^* \Vert_F^2 \\
&= \sum_{i=1}^M \Bigg\Vert \left(\frac{1}{|\mathcal{A}_i|} \sum_{\bm{x_j}^* \in \mathcal{A}_i} \bm{x_j}^* \right) - \bm{x_i}^*\ \Bigg\Vert^2 \\
&= \sum_{i=1}^M \Bigg\Vert \left(\frac{1}{|\mathcal{A}_i|} \sum_{\bm{x_j} \in \mathcal{A}_i} \bm{x_j} \right) - \bm{x_i}\ \Bigg\Vert^2 \label{eq:ap_error_sorting_ambiguity}
\end{align}

Here, $\Vert \cdot \Vert_F$ denotes the Frobenius norm, which is the square root of the sum of the squared elements of a matrix or vector. Note that we leverage the summation in the formula to make the error term independent from the definition of the target sequence $Y^*$.

\textbf{The Impact of Ordering Ambiguity on Autoregressive Models:} The ordering of tokens within an ordering ambiguity set is determined using a uniform distribution in the above analysis, which might seem confusing, given that the DR methods we employ are deterministic and, theoretically, should always yield a consistent order rather than a ``random" one. However, the randomness can be understood from the perspective of the training data input into the autoregressive models.

In training autoregressive models, our objective is to enable the model to learn generalizable rules automatically from the training data, including those determining the ordering of tokens. Suppose the ordering varies dramatically among different sequences in the training data with similar patterns. In that case, the model may struggle to identify the factors causing these discrepancies in ordering. Instead, it may resort to learning an average of possible orderings (see Figure \ref{fig:sorting_ambiguity_intro}), a phenomenon that is commonly observed when the models have difficulty fitting the data.

The issue of ordering ambiguity compounds this problem. Without a consistent and reliable ordering rule for tokens within the ambiguity set, different samples in the dataset might exhibit varied ordering. Such inconsistency makes it challenging for autoregressive models to effectively learn the underlying ordering rule. As such, the presence of ordering ambiguity in the training data can significantly hamper the performance of autoregressive models in terms of sequence prediction.

\begin{figure}[ht]
    \centering
    \includegraphics[width=\textwidth]{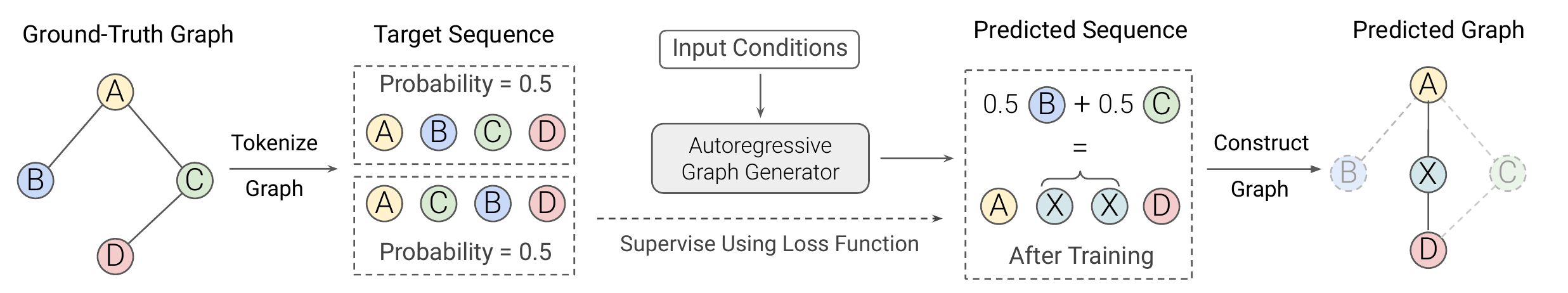}
    \caption{Illustration of Ordering Ambiguity Problem In Autoregressive Graph Generation.}
    \label{fig:sorting_ambiguity_intro}
\end{figure}

\subsection{The Latent Sort Algorithm}
\label{sec:latent_sort_theory}

The main idea of the latent sort algorithm is to use an auto-encoder to learn a dimensionality reduction mapping, which enjoys high degree of flexibility when paired with different loss functions to achieve different goals. 
In our design, the auto-encoder consists of two multi-layer perceptrons (MLPs) as the encoder and decoder, respectively, denoted as $f_e$ and $f_d$. We can use the encoder for sorting $\bm{x_i}$ and the ordering is given by the encodede 1-D latent representation $h_i = f_e(\bm{x_i})$. 

\textbf{The Ordering Ambiguity of Latent Sort:}
Since the auto-encoder uses deterministic neural networks (e.g., MLPs), the encoder and decoder are both surjective (see Remark \ref{remark:surjective_neural_network}). The auto-encoder uses reconstruction loss during training. If the auto-encoder has perfection reconstruction, its encoder and decoder are both bijective (see Remark \ref{remark:bijective_autoencoder}) and there would be no ordering ambiguity. However, we know it cannot be true for most cases (see Remark \ref{remark:bijective_dimensionality_reduction}). Therefore, when there exists no bijective mapping between $\mathcal{X}$ and $\mathcal{H}$, ordering ambiguity in inevitable in latent sort. When multiple points in $\mathcal{X}$ are mapped to the same point in $\mathcal{H}$, we can use Equation (\ref{eq:error_sorting_ambiguity}) to estimate the error due to ordering ambiguity. 

\textbf{Some Theoretical Properties of The Latent Sort Autoencoders:}

For latent sort, since we adopt auto-encoder for dimensionality reduction, there are more theoretical properties due to the fact that the neural networks are often Lipschitz continuous.

Assume that the auto-encoder is initialized using a random distribution with a small variance, and the model is trained with proper regularization, such that after training, the encoder and decoder are Lipschitz continuous with constants $K_e$ and $K_d$ respectively, and $K_e K_d \geq 1$. Assume for all $\bm{x_i} \in \mathcal{X}$ and its reconstructed $\hat{\bm{x_i}} = f_d(f_e(\bm{x_i}))$, $\| \bm{x_i} - \hat{\bm{x_i}} \|$ has a constant upper bound $B$. Assuming $\forall \bm{x_i} \in \mathcal{X}$, there always exist at least one point $h_i^* \in R$ such that $f_d(h_i^*) = \bm{x_i}$, we have the following theorems on the relationship between the latent representation error and the reconstruction error:

\vspace{2ex}

\begin{theorem}[Bounded Original Distance]
\label{th:bounded_data_distance}
Let $\bm{x_i}$ and $\bm{x_j}$ be two distinct points in $\mathcal{X} \subseteq \mathbb{R}^N$, and let $h_i = f_e(\bm{x_i})$ and $h_j = f_e(\bm{x_j})$ be their corresponding latent representations, and $\hat{\bm{x_i}} = f_d(h_i)$ and $\hat{\bm{x_j}} = f_d(h_j)$ be their corresponding reconstructions. If $| h_i - h_j | \leq \epsilon$, where constant $\epsilon > 0$, then the distance $\| \bm{x_i} - \bm{x_j} \|$ has an upper bound $2B + {K_d} \epsilon$.
\end{theorem}
\begin{proof}
From the assumptions, we know that the decoder $f_d$ is $K_d$-Lipschitz continuous. Let $f_d(h_i^*) = \bm{x_i}$ and $f_d(h_j^*) = \bm{x_j}$, where $h_i^*$ and $h_j^*$ are the two 1-D representation values that could result in perfect reconstruction. Since the reconstruction error has a constant upper bound $B$ for all points in $\mathcal{X}$, we get
\begin{align*}
\| \bm{x_i} - \hat{\bm{x_i}} \| &=  \| \bm{x_i} - f_d(h_i) \| \\
&=  \| f_d(h_i^*) - f_d(h_i) \| \\
&\leq K_d | h_i^* - h_i | \leq B\\
\| \bm{x_j} - \hat{\bm{x_j}} \| &=  \| \bm{x_j} - f_d(h_j) \| \\
&=  \| f_d(h_j^*) - f_d(h_j) \| \\
&\leq K_d | h_j^* - h_j | \leq B
\end{align*}
Using $| h_i - h_j | \leq \epsilon$ and triangle inequality we have: 
\begin{align*}
\| \bm{x_i} - \bm{x_j} \| &=  \| f_d(h_i^*) - f_d(h_j^*) \| \\
&\leq K_d | h_i^* - h_j^* | \\
&= K_d ( | h_i^* - h_i  +  h_i - h_j  +  h_j - h_j^* | ) \\
&\leq K_d ( | h_i^* - h_i | + | h_i - h_j | + | h_j - h_j^* | ) \\
&\leq K_d \left( \frac{1}{K_d} \big\| f_d(h_i^*) - f_d(h_i) \big\| + \epsilon + \frac{1}{K_d} \big\| f_d(h_j^*) - f_d(h_j) \big\| \right) \\
&= \| f_d(h_i^*) - f_d(h_i) \| + {K_d} \epsilon + \| f_d(h_j^*) - f_d(h_j) \| \\
&\leq 2B + {K_d} \epsilon, 
\end{align*}
which gives the upper bound of $\| \bm{x_i} - \bm{x_j} \|$.

In addition, we need to examine whether the upper bound exists. Since the encoder is $K_e$-Lipschitz continuous everywhere on $\mathcal{X}$ and we have assumed $K_e K_d \geq 1$, thus
\begin{align}
| h_i - h_j | &\leq {K_e} \| \bm{x_i} - \bm{x_j} \|  \\
 &\leq {K_e} (2B + {K_d} \epsilon )  \\
 &= 2 {K_e} B + {K_e} {K_d} \epsilon 
\label{eq:th1_inequality_condition}
\end{align}
Since $| h_i - h_j | \leq \epsilon$ and $2 K_e B \geq 0$, inequality (\ref{eq:th1_inequality_condition}) is always true for any $\bm{x_i}$ and $\bm{x_j}$ in $\mathcal{X}$. Therefore, the upper bound exists as stated in the theorem.
\end{proof}

Theorem \ref{th:bounded_data_distance} establishes a relationship between the distance in the original space and the distance in the latent representation space learned by an auto-encoder. It states that if two data points have similar latent representations, their distance in the original space is also small, with an upper bound proportional to the Lipschitz constants of the encoder and decoder. 

\vspace{1ex}

\begin{theorem}[Bounded Latent Distance]
\label{th:bounded_latent_distance}
Let $\bm{x_i}$ and $\bm{x_j}$ be two distinct points in $\mathcal{X} \subseteq \mathbb{R}^N$, and let $h_i = f_e(\bm{x_i})$ and $h_j = f_e(\bm{x_j})$ be their corresponding latent representations, and $\hat{\bm{x_i}} = f_d(h_i)$ and $\hat{\bm{x_j}} = f_d(h_j)$ be their corresponding reconstructions. If the distance $\| \bm{x_i} - \bm{x_j} \|$ is greater than or equal to a constant $D$, and $D \geq 2B$, then the difference in their latent representations $| h_i - h_j |$ is larger than a constant lower bound $( D - 2B ) / {K_d} $.
\end{theorem}

\begin{proof}
Let $f_d(h_i^*) = \bm{x_i}$ and  $f_d(h_j^*) = \bm{x_j}$. Since $f_d$ is Lipschitz continuous with constant $K_d$, we know that
\begin{align*}
\| \bm{x_i} - \hat{\bm{x_i}} \| &\leq K_d | h_i^* - h_i | \leq B \\
\| \bm{x_j} - \hat{\bm{x_j}} \| &\leq K_d | h_j^* - h_j | \leq B \\
D  \leq  \| \bm{x_i} - \bm{x_j} \| &\leq K_d | h_i^* - h_j^* |
\end{align*}

From the assumptions, we know that the distance $\| \bm{x_i} - \bm{x_j} \| \geq D$ and $D \geq 2B$. Using reverse triangle inequality, we get
\begin{align}
| h_i - h_j | &= | (h_i^* - h_j^*) - (h_i^* - h_i) + (h_j^* - h_j) | \\
&\geq \Big| | h_i^* - h_j^* | - | (h_i^* - h_i) - (h_j^* - h_j) | \Big|\\
&\geq \Big| | h_i^* - h_j^* | - | h_i^* - h_i | - | h_j^* - h_j | \Big|\\
&\geq \frac{1}{K_d} \Big| D - B - B \Big| \\
&\geq \frac{1}{K_d} ( D - 2B )
\end{align}
Thus we obtain the lower bound for $| h_i - h_j |$.

In addition, we need to examine whether the lower bound exists. Since $f_e$ is $K_e$-Lipschitz continuous and $K_e K_d \geq 1$, we have
\begin{align}
| h_i - h_j | &\leq {K_e} \| \bm{x_i} - \bm{x_j} \| \\
\frac{1}{K_d} ( D - 2B ) &\leq {K_e} \| \bm{x_i} - \bm{x_j} \| \\
\frac{D - 2B}{ K_e K_d} &\leq D \leq \| \bm{x_i} - \bm{x_j} \| \label{eq:th2_inequality_condition}
\end{align}
Inequality (\ref{eq:th2_inequality_condition}) is always true therefore the lower bound exists as stated in the theorem.
\end{proof}
Theorem \ref{th:bounded_data_distance} provides a complementary perspective to Theorem \ref{th:bounded_latent_distance}. While the first theorem bounds the distance in the original space based on the distance in the latent representation space, the second theorem bounds the difference in the latent representations based on the distance in the original space. This is useful to estimate the error from latent ambiguity in latent sort, which we will discuss in the later sections.

Interestingly, it has a connection to ordering ambiguity. If the ordering ambiguity problem is serious and the ambiguity sets are large, then the upper bound $B$ for the reconstruction error is expected to be large as well. In such cases,  $D - 2B$ in the original space can shrink to close to zero, which inversely supports the existence of ordering ambiguity.

\subsection{The Errors From Imperfect Reconstruction}
\label{sec:ap_latent_sort_error}

Previously, when discussing the errors from ambiguity, we assumed that $Y = Y^*$ in the absence of ordering ambiguity, even though it is not realistic since the auto-encoder used in latent sort is not bijective in most cases. To quantify the errors from imperfect reconstruction, we assume no ordering ambiguity in this section.

We model the imperfect reconstruction by assigning a non-zero variance to the reconstructed $\hat{\bm{x_i}}$, and we assume the mean value equals to the target $\bm{x_i}$. We assume $\hat{\bm{x_i}}$ follows Gaussian distribution (see Remark \ref{remark:reconstruction_gaussian}), and larger variances indicate worse reconstruction, and vice versa. 

Figure \ref{fig:theorem_3} shows a schematic plot of an autoencoder that has imperfect reconstruction. The ideal autoencoder that has perfect reconstruction are shown as the dashed lines in the plot, indicating the bijectivity of the autoencoder. However, when the reconstruction is not perfect, we assume that the $\hat{h_i}$ is a sample from the distribution that are centered around $h_i$.

\begin{figure}[ht]
    \centering
    \includegraphics[width=0.75\textwidth]{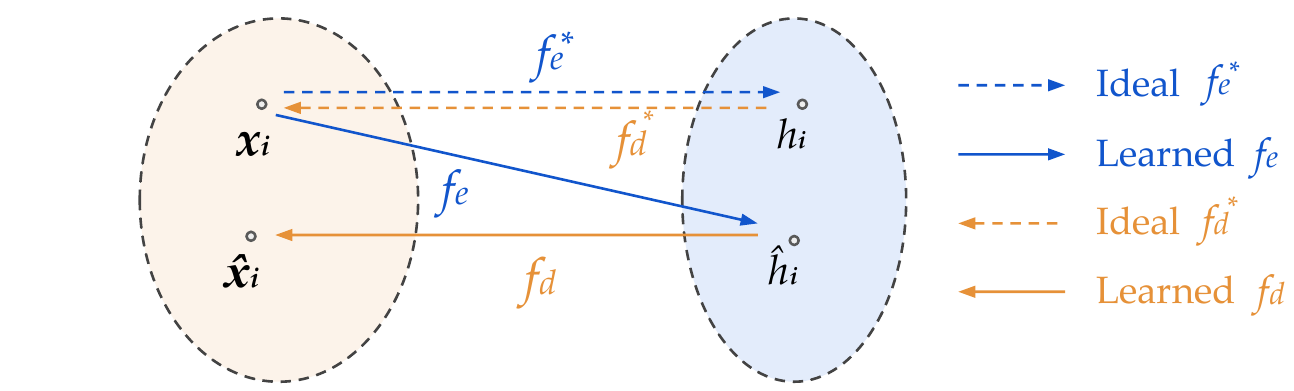}
    \caption{A schematic representation of the autoencoder under imperfect reconstruction.}
    \label{fig:theorem_3}
\end{figure}

Now, we want to understand, when observing the reconstruction error in the original N-D space, can we estimate the distribution of the 1-D latent distribution? This motivates the following theorem.

\vspace{2ex}

\begin{theorem}[Latent Distribution Under Imperfect Reconstruction]
Let $\bm{\sigma_{\hat{\bm{x_i}}}^2} \in \mathbb{R}^N$ be the component-wise variance vector of the reconstructed point $\hat{\bm{x_i}}$, and $\bm{x_i} = f_d(h_i)$, $\hat{\bm{x_i}} = f_d(\hat{h_i})$. Assume that the encoder $f_e$ and decoder $f_d$ are Lipschitz continuous on $\mathcal{X}$ with constants $K_e$ and $K_d$ respectively. If $\hat{h_i}$ also follows a normal distribution, the following bounds hold for the mean $\mu_{\hat{h_i}}$ and variance $\sigma_{\hat{h_i}}^2$ of $\hat{h_i}$:
\begin{align}
\frac{\big\Vert \bm{\sigma_{\hat{\bm{x_i}}}} \big\Vert}{K_d} \leq \big| \mu_{\hat{h_i}} - h_i \big| \leq K_e \big\Vert \bm{\sigma_{\hat{\bm{x_i}}}} \big\Vert \\
\frac{1}{K_d^2}\  \big\Vert \bm{\sigma_{\hat{\bm{x_i}}}} \big\Vert^2 \leq \sigma_{\hat{h_i}}^2 \leq 4 K_e^2\  \big\Vert \bm{\sigma_{\hat{\bm{x_i}}}} \big\Vert^2
\end{align}
\label{th:error_latent_normal_distribution}
\end{theorem}

\begin{proof}
Consider the following chain of inequalities based on the Lipschitz continuity of $f_e$ and $f_d$:
\begin{align}
\|\bm{x_i} - \hat{\bm{x_i}}\| &= \|f_d(h_i) - f_d(\hat{h_i})\| \\
&\leq K_d \big|h_i - \hat{h_i}\big|.
\end{align}

Now, we can also write the Lipschitz continuity for the encoder $f_e$:
\begin{align}
\big|h_i - \hat{h_i}\big| &= |f_e(\bm{x_i}) - f_e(\hat{\bm{x_i}})| \\
&\leq K_e \|\bm{x_i} - \hat{\bm{x_i}}\|.
\end{align}

Combining the two inequalities, we get:
\begin{align}
    \frac{\|\bm{x_i} - \hat{\bm{x_i}}\|}{K_d}  \leq \big|h_i - \hat{h_i}\big| &\leq K_e \|\bm{x_i} - \hat{\bm{x_i}}\| \\
    \text{s.t.\ } \ K_e K_d &\geq 1
\end{align}

Taking the squared expectation on both sides:
\begin{equation}
\frac{\mathbb{E}\Big[ 
 \| \bm{x_i} - \hat{\bm{x_i}} \|^2 \Big]}{K_d^2} \leq \mathbb{E} \Big[ \big|h_i - \hat{h_i} \big|^2 \Big] \leq K_e^2 \cdot \mathbb{E}\Big[ 
 \| \bm{x_i} - \hat{\bm{x_i}} \|^2 \Big]
\end{equation}

Considering that $\mathbb{E}[\hat{\bm{x_i}}] = \bm{x_i}$ we have:
\begin{align}
\bm{\sigma_{\hat{\bm{x_i}}}^2} &= \Var(\hat{\bm{x_i}}) \\
&= \mathbb{E} \Big[ (\hat{\bm{x_i}} - \mathbb{E}\big[\hat{\bm{x_i}}\big])^2 \Big] \\
&= \mathbb{E} \big[ (\hat{\bm{x_i}} - \bm{x_i})^2 \big]
\end{align}

And we have
\begin{align}
    \mathbb{E}\Big[ 
 \| \bm{x_i} - \hat{\bm{x_i}} \|^2 \Big] &= \mathbb{E} \Big[ (\hat{\bm{x_i}} - \bm{x_i})^\intercal (\hat{\bm{x_i}} - \bm{x_i}) \Big] \\
&= \Big\Vert\ \mathbb{E} \big[ (\hat{\bm{x_i}} - \bm{x_i})^2 \big]\  \Big\Vert \\
&= \sum_{j=1}^{N} \sigma_{\hat{\bm{x_i}},j}^2 = \big\Vert \bm{\sigma_{\hat{\bm{x_i}}}} \big\Vert^2
\end{align}
where $\sigma_{\hat{\bm{x_i}},j}^2$ is the $j$-th component of the variance vector $\bm{\sigma_{\hat{\bm{x_i}}}^2}$.
Therefore we have the upper bound and lower bound for the mean value of $\hat{h_i}$:
\begin{align}
\frac{\big\Vert \bm{\sigma_{\hat{\bm{x_i}}}} \big\Vert^2}{K_d^2} &\leq \mathbb{E} \Big[ \big|h_i - \hat{h_i} \big|^2 \Big] \leq K_e^2 \cdot \big\Vert \bm{\sigma_{\hat{\bm{x_i}}}} \big\Vert^2 \\
\frac{\big\Vert \bm{\sigma_{\hat{\bm{x_i}}}} \big\Vert}{K_d} &\leq \big| \mu_{\hat{h_i}} - h_i \big| \leq K_e \big\Vert \bm{\sigma_{\hat{\bm{x_i}}}} \big\Vert \label{eq:latent_mean_bounds}
\end{align}

Now, consider the variance of $\hat{h_i}$. From the Lipschitz continuity of the encoder and decoder, we can use triangle inequality:
\begin{align}
\sigma_{\hat{h_i}}^2 = \Var(\hat{h_i}) &= \mathbb{E} \Big[ (\hat{h_i} - \mathbb{E}\big[\hat{h_i}\big])^2 \Big] = \mathbb{E} \big[ (\hat{h_i} - \mu_{\hat{h_i}})^2 \big]\\
&= \mathbb{E} \Big[ \big| \hat{h_i} - h_i + h_i - \mu_{\hat{h_i}} \big|^2 \Big]\\
&\leq \mathbb{E} \Big[ \big| f_e(\hat{\bm{x_i}}) - f_e(\bm{x_i}) \big|^2 \Big] + \big| \mu_{\hat{h_i}} - h_i \big|^2 \\
&\ \ \ \ + 2\mathbb{E} \Big[ \big| f_e(\hat{\bm{x_i}}) - f_e(\bm{x_i}) \big| \Big] \cdot \big| \mu_{\hat{h_i}} - h_i \big| \\
&\leq K_e^2\ \mathbb{E}\Big[ | \bm{x_i} - \hat{\bm{x_i}} |^2 \Big] + K_e^2 \big\Vert \bm{\sigma_{\hat{\bm{x_i}}}} \big\Vert^2 \\
&\ \ \ \ + 2 K_e^2\ \big\Vert \bm{\sigma_{\hat{\bm{x_i}}}} \big\Vert \cdot \mathbb{E}\Big[ | \bm{x_i} - \hat{\bm{x_i}} | \Big] \\
&= 4 K_e^2\ \big\Vert \bm{\sigma_{\hat{\bm{x_i}}}} \big\Vert^2
\end{align}

Conversely, using the Lipschitz continuity of the decoder, we can derive the lower bound of the variance:
\begin{align}
\sigma_{\hat{h_i}}^2 = \Var(\hat{h_i}) &= \mathbb{E} \Big[ (\hat{h_i} - \mathbb{E}\big[\hat{h_i}\big])^2 \Big] = \mathbb{E} \big[ (\hat{h_i} - \mu_{\hat{h_i}})^2 \big]\\
&= \mathbb{E} \Big[ \big| (\hat{h_i} - h_i) - (\mu_{\hat{h_i}} - h_i ) \big|^2 \Big]
\end{align}

We can apply the Cauchy-Schwarz inequality to above, which gives:
\begin{equation}
\mathbb{E}\left[ \left(\hat{h}_i - h_i\right)\left(\hat{h}_i - \mu_{\hat{h}_i}\right) \right]^2 \leq \mathbb{E}\left[ \left(\hat{h}_i - h_i\right)^2 \right]\mathbb{E}\left[ \left(\hat{h}_i - \mu_{\hat{h}_i}\right)^2 \right]
\end{equation}

We can rearrange the inequality:
\begin{equation}
\mathbb{E}\left[ \left(\hat{h}_i - \mu_{\hat{h}_i}\right)^2 \right] \geq \frac{\mathbb{E}\left[ \left(\hat{h}_i - h_i\right)\left(\hat{h}_i - \mu_{\hat{h}_i}\right) \right]^2}{\mathbb{E}\left[ \left(\hat{h}_i - h_i\right)^2 \right]}
\end{equation}

Using the triangle inequality, we have:
\begin{align}
    \Big[\ \mathbb{E} \big[ (\hat{h_i} - h_i)(h_i - \mu_{\hat{h_i}}) \big] \Big]^2
    & \leq \Big[ \ \mathbb{E} \big[ |\hat{h_i} - h_i| |h_i - \mu_{\hat{h_i}}| \big]  \Big]^2 \\
    & \leq \mathbb{E} \Big[ (\hat{h_i} - h_i)^2 \Big] \mathbb{E} \Big[ (h_i - \mu_{\hat{h_i}})^2 \Big]
\end{align}

Now, we can use this bound in the expression for the lower bound:
\begin{align}
\mathbb{E}\left[ \left(\hat{h}_i - \mu_{\hat{h}_i}\right)^2 \right] &\geq \frac{\mathbb{E}\left[ \left(\hat{h}_i - h_i\right)\left(\hat{h}_i - \mu_{\hat{h}_i}\right) \right]^2}{\mathbb{E}\left[ \left(\hat{h}_i - h_i\right)^2 \right]} \\
&\geq \frac{ \mathbb{E} \Big[ \big(\hat{h_i} - h_i\big)^2 \Big] \mathbb{E} \Big[ \big(h_i - \mu_{\hat{h_i}}\big)^2 \Big] }{\mathbb{E}\left[ \left(\hat{h}_i - h_i\right)^2 \right]} \\
&= \mathbb{E} \Big[ \big(h_i - \mu_{\hat{h_i}}\big)^2 \Big]
\end{align}
Note that $\mathbb{E} \big[ (h_i - \mu_{\hat{h_i}})^2 \big] = \big| \mu_{\hat{h_i}} - h_i \big|^2 \geq \big\Vert \bm{\sigma_{\hat{\bm{x_i}}}} \big\Vert^2 / K_d^2 $. Thus, we have obtained the bounds for the variance of $\hat{h_i}$ as follows:
\begin{equation}
\frac{1}{K_d^2} \cdot \big\Vert \bm{\sigma_{\hat{\bm{x_i}}}} \big\Vert^2  \leq \sigma_{\hat{h_i}}^2 \leq 4 K_e^2\  \big\Vert \bm{\sigma_{\hat{\bm{x_i}}}} \big\Vert^2
\end{equation}
\end{proof}

Theorem \ref{th:error_latent_normal_distribution} provides bounds on the mean and variance of the 1-D representation in terms of the Lipschitz continuity of the encoder and decoder and the component-wise variance of the reconstruction error.
To build theoretical bounds for the error from imperfect reconstruction, we need to consider how the latent distribution affect the sorting results.


\textbf{How Does Imperfect Reconstruction Affect Ordering Results?}

We have discussed the distribution of 1-D representation values when the reconstruction is not perfect. Connecting back to the ordering results, when two points $\bm{x_i}$ and $\bm{x_j}$ follow an ordering where $i < j$, we want the learned 1-D latent to always satisfy $\hat{h_i} < \hat{h_j}$. However, since their 1-D latent representations have non-zero variances, it is possible that $\hat{h_i} \geq \hat{h_j}$, resulting in a swap between $\bm{x_i}$ and $\bm{x_j}$ and leading to an error in the sorted sequence. Therefore we want to quantitatively measure how much error can we expect from imperfect reconstruction.

Assuming for all $\bm{x_i}^*$ in the target sequence $Y^*$, we have $ f_e(\bm{x_i}^*) = \hat{h_i^*} \sim \mathcal{N}(\mu_{\hat{h_i^*}}, \sigma_{\hat{h_i^*}}^2)$  independent of each other (using local DR). For each pair of $\hat{h_i^*}$ and $\hat{h_j^*}$ where $i \ne j$, the probability $P(\hat{h_i^*} > \hat{h_j^*})$ is given by (see Remark \ref{remark:probability_a_larger_than_b}):
\begin{equation}
    \label{eq:latent_ordering_gaussian}
   P(\hat{h_i^*} > \hat{h_j^*}) = 1 - \Phi \left[\ - (\mu_{\hat{h_i^*}} - \mu_{\hat{h_j^*}}) / \sqrt{(\sigma_{\hat{h_i^*}}^2 + \sigma_{\hat{h_j^*}}^2)}\ \right] 
\end{equation}
Here, $\Phi$ represents the cumulative distribution function (CDF) of the standard normal distribution. Since the ordering results in $Y$ is given by the ordering of latent representation, we can use Equation (\ref{eq:latent_ordering_gaussian}) to calculate matrix $P$ as follows.

Denote $\gamma_{ij} = P(\hat{h_i^*} < \hat{h_j^*}) =  \Phi \left[\ - (\mu_{\hat{h_i^*}} - \mu_{\hat{h_j^*}}) / \sqrt{(\sigma_{\hat{h_i^*}}^2 + \sigma_{\hat{h_j^*}}^2)}\ \right] $, we can calculate the value of each entry in $P$ (see Lemma \ref{lemma:comparing_normal_random_variables}) under the assumption of imperfect reconstruction as:
\begin{align}
    p_{ij} & = \sum_{\mathcal{S}_i,\mathcal{T}_i} \prod_{k \in \mathcal{S}_i} \Big[ P(\hat{h_j^*} < \hat{h_k^*}) \Big] \prod_{k \in \mathcal{T}_i} \Big[ P(\hat{h_j^*} > \hat{h_k^*}) \Big] \\
    & = \sum_{\mathcal{S}_i,\mathcal{T}_i} \prod_{k \in \mathcal{S}_i} \Big( \gamma_{jk}  \Big) \prod_{k \in \mathcal{T}_i} \Big( 1 - \gamma_{jk} \Big) \label{eq:p_value_from_imperfect_reconstruction}
\end{align}
where the summation is over all possible partitions of the index set $\{1, 2, ..., M\} \setminus {i}$ into two disjoint subsets $\mathcal{S}_i$ and $\mathcal{T}_i$ with $|\mathcal{S}_i| = i - 1$ and $|\mathcal{T}_i| = M - i$.

\textbf{An Intuitive Interpretation of Equation (\ref{eq:p_value_from_imperfect_reconstruction})}: To have the $j$-th element $\bm{x_j}^*$ from the target sequence $Y^*$ to be the $i$-th element $\bm{x_i}$ in the resulting sequence $Y$, the encoded 1-D representation of $\bm{x_j}^*$ has to be smaller than the 1-D latent of $M - i$ number of points and it has to be larger than the 1-D latent of $i - 1$ number of points. We need to consider all partitions to sum up the probabilities from all possible orderings that has the $j$-th element from $Y^*$ to be placed at $i$-th location in $Y$. 

Note that Equation (\ref{eq:p_value_from_imperfect_reconstruction}) assumes no ordering ambiguity is presented. If we want to get a more accurate estimation of the error for latent sort algorithm, we can add ordering ambiguity in it. The composite $P$ matrix can be estimated by simply calculating a weighted summation over the two $P$ matrices from Equation (\ref{eq:p_value_latent_ambiguity}) and Equation (\ref{eq:p_value_from_imperfect_reconstruction}), then normalize the summed matrix to be a probability matrix.

\subsection{Simplifying The Errors From Imperfect Reconstruction}
\label{sec:ap_minization_of_errors}

Equation (\ref{eq:p_value_from_imperfect_reconstruction}) gives an estimation of the errors of latent sort algorithm using the matrix $P$. However, without simplification, it is intractable to calculate the probability values due to the factoral number of possible partitions $\mathcal{S}_i$ and $\mathcal{T}_i$.  Here we will first provide a way to simplify Equation (\ref{eq:p_value_from_imperfect_reconstruction}), then we will use it to reveal the shortest path property of latent sort algorithm.

The simplification leverages the sparsity in the summation in Equation (\ref{eq:p_value_from_imperfect_reconstruction}), based on the property that if any term in the product is zero, the product becomes zero. 

Let $\Gamma_j = \prod_{k \in \mathcal{S}_i} \left( \gamma_{jk}  \right) \prod_{k \in \mathcal{T}_i} \left( 1 - \gamma_{jk} \right)$. In a partitions that yields $\mathcal{S}_i$ and $\mathcal{T}_i$ among all possible partitions, if $\exists k \in \mathcal{S}_i, \gamma_{jk} \to 0$ or $\exists k \in \mathcal{T}_i, \gamma_{jk} \to 1$, the value of $\Gamma_j \to 0$. 

From Theorem \ref{th:error_latent_normal_distribution}, we can see that the minimization on reconstruction loss, empirically measured by $\Vert \bm{\sigma_{\hat{\bm{x_i}}}^2} \Vert$, can lead to the collapse of the distribution of $\hat{h_i^*}$ to a sinlge value $h_i$ (see Remark \ref{remark:squeeze_theorem_error_bound}). Thus, when we train the auto-encoder with reconstruction loss, we can assumed the $\Vert \bm{\sigma_{\hat{\bm{x_i}}}^2} \Vert$ is properly minimized (meaning it is bounded by a small constant), and the value of $\mu_{\hat{h_i^*}}$ will not deviate from $h_i$ too much, such that we can assume $\mu_{\hat{h_i^*}}$ follows ascending order as $h_i$ does, i.e., $\mu_{\hat{h_i^*}} \leq \mu_{\hat{h_j^*}} \iff i \leq j$. 

In addition, from Theorem \ref{th:bounded_latent_distance} and \ref{th:bounded_data_distance}, we know that neighboring points in the original space $\mathbb{R}^N$ tend to have similar 1-D latent values. Moreover, from the property of $\Phi$ as the CDF of normal distribution, the value of $\Phi(x)$ quickly becomes infinitesimal when $x$ starts to deviate from zero towards negative infinite, similarly $1 - \Phi(x)$ quickly becomes infinitesimal when $x$ starts to deviate from zero towards positive infinite. As $\Vert \bm{\sigma_{\hat{\bm{x_i}}}^2} \Vert$ getting smaller in training, the value of $\sqrt{(\sigma_{\hat{h_i^*}}^2 + \sigma_{\hat{h_j^*}}^2)}$ is also getting smaller, thus the value  of $\gamma_{ij}$ will be more sensitive to the difference of mean values $\mu_{\hat{h_i^*}} - \mu_{\hat{h_j^*}}$, pushing more and more values to zero and one. 

Therefore, we can replace the value of $y_{ij}$ to be 0 or 1 when $i$ and $j$ are not neighbors, and we can approximate $\gamma_{ij}$ as follows:
\begin{equation}
\label{eq:simplified_gamma_value_from_imperfect_reconstruction}
\gamma_{ij} \approx \begin{cases}
1, & \text{if} \ i < j - 1 \\
\Phi \left[\ - (\mu_{\hat{h_i^*}} - \mu_{\hat{h_j^*}}) / \sqrt{(\sigma_{\hat{h_i^*}}^2 + \sigma_{\hat{h_j^*}}^2)}\ \right], & \text{if}\ j - 1 \leq i \leq j + 1 \\
0, & \text{if} \ i > j + 1
\end{cases}
\end{equation}

Thus using Equation (\ref{eq:simplified_gamma_value_from_imperfect_reconstruction}) and the definition of $\mathcal{S}_i$ and $\mathcal{T}_i$, we can simplify Equation (\ref{eq:p_value_from_imperfect_reconstruction}) as follows:
\begin{align}
    p_{ij}
    & = \sum_{\mathcal{S}_i,\mathcal{T}_i} \prod_{k \in \mathcal{S}_i} \Big( \gamma_{jk}  \Big) \prod_{k \in \mathcal{T}_i} \Big( 1 - \gamma_{jk} \Big) \\
    & = \begin{cases}
        \big[ \gamma_{j, (j-1)} \big] \big[ 1 - \gamma_{j, (j + 1)} \big] + \big[ \gamma_{j, (j+1)} \big] \big[ 1 - \gamma_{j, (j - 1)} \big], & i = j\\
        \big[ \gamma_{j, (j-1)} \big] \big[ \gamma_{j, (j+1)} \big], & i = j - 1\\
        \big[ 1 - \gamma_{j, (j-1)} \big] \big[ 1 - \gamma_{j, (j+1)} \big], & i = j + 1\\
        0, & i \notin [j - 1, j + 1]
    \end{cases}
    \label{eq:simplified_p_value_from_imperfect_reconstruction}
\end{align}
Equation (\ref{eq:simplified_p_value_from_imperfect_reconstruction}) gives a simplification of $P$ matrix in the case of imperfect reconstruction, where $P$ becomes a tridiagonal matrix with non-zero values on the main diagonal and the diagonals immediately above and below it, which looks like:
\begin{equation}
    P = \begin{pmatrix}
        p_{11} & p_{12} & 0 & \cdots & 0 \\
        p_{21} & p_{22} & p_{23} & \cdots & 0 \\
        0 & p_{32} & p_{33} & \cdots & 0 \\
        \vdots & \vdots & \vdots & \ddots & \vdots \\
        0 & 0 & 0 & \cdots & p_{MM} \\
        \end{pmatrix}
        \label{eq:simplified_p_matrix}
\end{equation}
We can do a quick check to verify that $P$ is still a probability matrix in the simplied form (see Remark \ref{remark:simplified_p_is_probability_matrix}). 

Now we can use Equation (\ref{eq:simplified_p_value_from_imperfect_reconstruction}) and Equation (\ref{eq:simplified_p_matrix}) to estimate the sorting error $\mathcal{E}$. Lemma \ref{lemma:optimization_equivalence} states the equivalence in minimization between $\mathcal{E}\Big(\mathbb{E}[Y], Y^*\Big)$ and $\| P - I_M \|_F^2$, thus we can calculate the difference between $P$ and $I_M$ using the simplified equation by summing the main diagonal difference and the off-diagonal difference:
\begin{align}
    \| P - I_M \|_F^2 &= \sum_{i=1}^{M} (p_{ii} - 1)^2 + \sum_{i=1}^{M-1} \Big[ p_{i(i+1)}^2 + p_{(i+1)i}^2 \Big] \\
    &= \sum_{i=1}^{M} (p_{ii} - 1)^2 + \sum_{i=2}^{M} p_{(i-1)i}^2 + \sum_{i=1}^{M-1}  p_{(i+1)i}^2 
\end{align}
Where the main diagonal:
\begin{align}
    (p_{ii} - 1)^2 &= \left\{ \big[ \gamma_{i, (i-1)} \big] \big[ 1 - \gamma_{i, (i + 1)} \big] + \big[ \gamma_{i, (i+1)} \big] \big[ 1 - \gamma_{i, (i - 1)} \big] - 1\right\}^2
\end{align}
And the diagonals immediately above and below the main diagonal:
\begin{align}
    p_{(i-1)i}^2 &= \big[ \gamma_{i, (i-1)} \big]^2 \big[ \gamma_{i, (i+1)} \big]^2 
    \label{eq:p_i-1_i} 
    \\
    p_{(i+1)i}^2 &= \big[ 1 - \gamma_{i, (i-1)} \big]^2 \big[ 1 - \gamma_{i, (i+1)} \big]^2
    \label{eq:p_i+1_i}
\end{align}
Based on the above simplfication, now we can state the following theorem.

\vspace{2ex}

\begin{theorem}
\label{th:minima_of_error_imperfect_reconstruction}
    Define $P$ using Equation (\ref{eq:simplified_p_value_from_imperfect_reconstruction}). The minima of the error term $\| P - I_M \|_F^2$ occur at $\gamma_{i, (i-1)} = 0$ and $\gamma_{i, (i+1)} = 1$ for all $i \in \{2, 3, ..., M -1\}$.
\end{theorem}

\begin{proof}
    Let $x = \gamma_{i, (i-1)}$ and $y = \gamma_{i, (i+1)}$.
    To find the minima of the error term $\| P - I_M \|_F^2$, we need to differentiate the error term with respect to $x$ and $y$, and find the values that minimize the error term. Here is the error term we are working with:
    \begin{align}
    \| P - I_M \|_F^2 = \sum_{i=1}^{M} (p_{ii} - 1)^2 + \sum_{i=2}^{M} p_{(i-1)i}^2 + \sum_{i=1}^{M-1}  p_{(i+1)i}^2
    \end{align}
    
    Using Equation (\ref{eq:simplified_p_matrix}) and exclude the cases that makes $p_{ij} = 0$, we have:
    \begin{align}
        \begin{cases}
            p_{ii} = x(1 - y) + y(1 - x), \\
            p_{(i-1)i} = xy, \\
            p_{(i+1)i} = (1 - x)(1 - y).
        \end{cases}
        \label{eq:p_value_exclude_zero}
    \end{align}

    Now, let's differentiate the error term with respect to $x$ and $y$. First, we will differentiate the error term with respect to $x$:
    \begin{align}
    \frac{\partial \| P - I_M \|_F^2}{\partial x} &= \sum_{i=1}^{M} 2(p_{ii} - 1)\frac{\partial p_{ii}}{\partial x} + \sum_{i=2}^{M} \Big[ 2p_{(i-1)i}\frac{\partial p_{(i-1)i}}{\partial x}  \Big] \\
    &\ \ \ \ + \sum_{i=1}^{M-1} \Big[ 2p_{(i+1)i}\frac{\partial p_{(i+1)i}}{\partial x} \Big]
    \end{align}

    Now, we differentiate $p_{ii}$, $p_{i(i+1)}$, and $p_{(i+1)i}$ with respect to $x$:
    \begin{align}
        \begin{cases}
            \frac{\partial p_{ii}}{\partial x} = (1 - y) - y = 1 - 2y, \\
            \frac{\partial p_{(i-1)i}}{\partial x} = y, \\
            \frac{\partial p_{(i+1)i}}{\partial x} = y - 1.
        \end{cases}
    \end{align}
    
    Plugging these derivatives back into the expression for the partial derivative with respect to $x$, we get:
    \begin{align}
    \frac{\partial \| P - I_M \|_F^2}{\partial x} &= \sum_{i=1}^{M} \Big[ 2(p_{ii} - 1) (1 - 2y)  \Big] + \sum_{i=2}^{M} \Big[ 2p_{(i-1)i} (y) \Big] \\
    &\ \ \ \  + \sum_{i=1}^{M-1} \Big[ 2p_{(i+1)i} (y-1) \Big] \\
    &= \sum_{i=1}^{M} \Big[ 2(x + y - 2xy - 1) (1 - 2y)  \Big] \\
    &\ \ \ \  + \sum_{i=2}^{M} \Big[ 2 xy^2 \Big] + \sum_{i=1}^{M-1} \Big[ 2 (1 - x - y + xy) (y-1) \Big]
    \end{align}
    Since $x$ and $y$ are between $[0, 1]$, for the above partial derivative to be zero, either $x = 0, y = 1$ or $x = 1, y = 0$.
    
    Next, we will differentiate the error term with respect to $y$. Notice that symmetry in Equation (\ref{eq:p_value_exclude_zero}), thus similarly we have:
    \begin{align}
    \frac{\partial \| P - I_M \|_F^2}{\partial y} &= \sum_{i=1}^{M} \Big[ 2(x + y - 2xy - 1) (1 - 2x)  \Big] \\
    &\ \ \ \  + \sum_{i=2}^{M} \Big[ 2 yx^2 \Big] + \sum_{i=1}^{M-1} \Big[ 2 (1 - x - y + xy) (x-1) \Big]
    \end{align}
    Again, to make the above derivative zero, we have $x = 0, y = 1$ and $x = 1, y = 0$ are two possible solutions. Thus, by setting the two partial derivatives to zero, we can see that the two minima of $\| P - I_M \|_F^2$ are $x = 0, y = 1$ and $x = 1, y = 0$.
    
    However, recall the definition $x = \gamma_{i, (i-1)}$ and $y = \gamma_{i, (i+1)}$, and the definition of  $\gamma_{ij} = P(\hat{h_i^*} < \hat{h_j^*})$. Since Equation (\ref{eq:simplified_gamma_value_from_imperfect_reconstruction})(\ref{eq:simplified_p_value_from_imperfect_reconstruction})(\ref{eq:simplified_p_matrix}) are defined over every row of $P$, which correspond to every position in the sorted sequence, we can see that the two minima represent two different orderings in the sorted sequence, where $x = 0, y = 1$ gives ascending order and $x = 1, y = 0$ gives descending order. Since in our discussion we only consider ascending order in the target sequence, $x = 0, y = 1$ is the only minimum that satisfy our definitions.
    \vspace{1ex}
\end{proof}

Theorem \ref{th:minima_of_error_imperfect_reconstruction} implies that in order to minimize $\mathcal{E}$, we can minimize $\| P - I_M \|_F^2$ instead (see Lemma \ref{lemma:optimization_equivalence}) by minimizing an equivalent objective $\gamma_{i, (i-1)} - \gamma_{i, (i+1)}$ for all $i = 2, ..., M - 1$, where both $\gamma_{i, (i-1)}$ and $\gamma_{i, (i+1)}$ are between 0 and 1. Use the definition of $\gamma_{ij}$, rewrite this alternative minimization problem on $\forall i \in \{2, 3, ..., M - 1 \}$ as:
\begin{align}
    \min_{f}&\ {\gamma_{i, (i-1)} - \gamma_{i, (i+1)}} \\
    =\min_{f}&\ {\Phi \left[\ - \frac{\mu_{\hat{h_i^*}} - \mu_{\hat{h_{i-1}^*}}}{\sqrt{(\sigma_{\hat{h_i^*}}^2 + \sigma_{\hat{h_{i-1}^*}}^2)}}\ \right]} - {\Phi \left[\ - \frac{\mu_{\hat{h_i^*}} - \mu_{\hat{h_{i+1}^*}}}{\sqrt{(\sigma_{\hat{h_i^*}}^2 + \sigma_{\hat{h_{i+1}^*}}^2)}}\ \right]} \\
    =\min_{f}&\ {\Phi \left[\ \frac{\mu_{\hat{h_i^*}} - \mu_{\hat{h_{i+1}^*}}}{\sqrt{(\sigma_{\hat{h_i^*}}^2 + \sigma_{\hat{h_{i+1}^*}}^2)}}\ \right]} - {\Phi \left[\ \frac{\mu_{\hat{h_i^*}} - \mu_{\hat{h_{i-1}^*}}}{\sqrt{(\sigma_{\hat{h_i^*}}^2 + \sigma_{\hat{h_{i-1}^*}}^2)}}\ \right]} 
    \label{eq:minimization_of_error_imperfect_recon}
\end{align}

Assuming the 1-D latents are normalized to [0, 1]. From the property of CDF $\Phi$, we know that (\ref{eq:minimization_of_error_imperfect_recon}) reaches minima -1 only when the reconstruction is perfect such that the variances $\sigma_{\hat{h_{i-1}^*}}, \sigma_{\hat{h_i^*}}, \sigma_{\hat{h_{i+1}^*}}$ are zero.

\subsection{Optimal Sorting From The Perspective of Shortest Path Problem}

\label{sec:ap_optimal_sort}

In previous theories, we have discussed the error of $Y$ to the target sorting $Y^*$ using the probability matrix $P$. However, the target sorting is not yet defined. The optimal sorting need to have properties to minimize the errors induced by imperfect $P$ matrices. Let us first discuss the desired properties of $Y^*$.

Recall that $\mathbb{E}[Y] = P Y^*$. From Equation (\ref{eq:sorting_error_def}) and (\ref{eq:simplified_p_matrix}), we have
\begin{align}
& \ \ \ \ \mathcal{E}\Big(\mathbb{E}[Y], Y^*\Big) \\
&= \Vert PY^* - Y^* \Vert_F^2 \\
&= \sum_{i=1}^{M} \left\| \sum_{j=1}^{M} p_{ij} \bm{x_{j}}^* - \bm{x_{i}}^* \right\|_F^2 \\
&= \sum_{i=1}^{M} \left\| p_{i(i-1))} \bm{x_{i-1}}^* + (p_{ii} - 1) \bm{x_{i}}^* + p_{i(i+1))} \bm{x_{i+1}}^* \right\|_F^2 \\
&= \sum_{i=1}^{M} \Big\| p_{i(i-1))} \bm{x_{i-1}}^* + \big[- p_{i(i-1))} - p_{i(i+1))} \big] \bm{x_{i}}^* + p_{i(i+1))} \bm{x_{i+1}}^* \Big \|_F^2 \\
&= \sum_{i=1}^{M} \left\| p_{i(i-1))} (\bm{x_{i-1}}^* - \bm{x_{i}}^*) + p_{i(i+1))} (\bm{x_{i+1}}^* - \bm{x_{i}}^*) \right\|_F^2
\label{eq:ap_error_using_neighbors}
\end{align}
Recall that $f$ is a local DR that maps each point in $\mathcal{X}$ to its 1-D latent representation independent from the other elements in the sequence. That being said, we expect (\ref{eq:ap_error_using_neighbors}) to be minimized for any three points $\{ \bm{x_{i-1}}^*, \bm{x_{i}}^*, \bm{x_{i+1}}^* \} \subset \mathcal{X}$. To make $\mathcal{E}\Big(\mathbb{E}[Y], Y^*\Big)$ minimized for all possible $\mathcal{X} \subseteq \mathbb{R}^{N}$, we need to minimize the upper bound of (\ref{eq:ap_error_using_neighbors}), giving by:
\begin{align}
&\ \ \ \ \sum_{i=1}^{M} \left\| p_{i(i-1))} (\bm{x_{i-1}} - \bm{x_{i}}) + p_{i(i+1))} (\bm{x_{i+1}} - \bm{x_{i}}) \right\|_F^2 \\
&\leq \sum_{i=1}^{M} \left\| p_{i(i-1))} (\bm{x_{i-1}}^* - \bm{x_{i}}^*) \right\|_F^2 + \left\| p_{i(i+1))} (\bm{x_{i+1}}^* - \bm{x_{i}}^*) \right\|_F^2 
\label{eq:ap_error_using_neighbors_2}
\end{align}

Assuming the matrix $P$ is already given for a sorting algorithm using a DR mapping $f$ and both $p_{i(i-1))}$ and $p_{i(i+1))}$ are not zero for each $i$. From (\ref{eq:ap_error_using_neighbors_2}), we can see that the ideal sorting $Y^*$ need to have minimal $(\bm{x_{i-1}}^* - \bm{x_{i}}^*)$ and $(\bm{x_{i+1}}^* - \bm{x_{i}}^*)$. An intuitive interpretation is that the ideal sorting $Y^*$ need to have the property of minimal distance between each neighboring pairs, so that the errors from imperfect $P$ will be minimized.

This property has an interesting connection to the shortest path problem or \textit{traveling salesman problem} (TSP), since a sorting an be seen as a path that traverse all the points in 
a given set. The desired target sorting is given by the solution from solving TSP on the input point set $\mathcal{X}$. However, TSP is NP-hard and requires expensive algorithms to solve it. Furthermore, the autoregressive models faces significant difficulties to learn the ordering rule represented by the solution of TSP.

\textbf{Why Autoregressive Models Struggle to Find The Shortest Path?}

The TSP is inherently a problem that requires a global understanding of the entire point set to solve. An optimal solution necessitates considering all the points simultaneously, understanding the distance and relationship between each pair of points, and calculating a path that minimally covers all points.

On the contrary, autoregressive models are inherently local and sequential in their processing. They generate output one token at a time and in a specific order, each time only attending to previous tokens in the sequence. They do not consider future tokens or have an overview of the entire set of tokens at each generation step.

Thus, the local, one-sided nature of autoregressive models is at odds with the global problem-solving requirement of the TSP, leading to significant challenges when attempting to learn the ordering rule represented by the solution of TSP.

\section{Supplementary Lemmas}

In this section, we provide some lemmas that help to understand the theoretical results discussed in Section \ref{sec:ap_additional_theory}.

\begin{lemma}
\label{lemma:optimization_equivalence}
Let $Y^*$ be an $M \times N$ matrix with full column rank, and let $P$ be an $M \times M$ probability matrix, i.e., the summation of every row and column in $P$ is one. If we minimize $\| P - I_M \|_F^2$, we can solve the minimizing problem of $\min_P \| PY^* - Y^* \|_F^2$.
\end{lemma}

\begin{proof}
We want to show that by minimizing $\| P - I_M \|_F^2$, we can minimize $\| PY^* - Y^* \|_F^2$. Here we know $P$ is probability matrix, so the set of all possible $P$ is a compact set in the bounded space, i.e., $P \in [0, 1]^{M \times M}$, and the objective function is continuous in that space. In this case, the Weierstrass theorem ensures the existence of a minimum.

First, consider the objective function we want to minimize:
\begin{equation}
\min \| PY^* - Y^* \|_F^2
\end{equation}

Now, let's expand the Frobenius norm:
\begin{equation}
\| PY^* - Y^* \|_F^2 = \operatorname{tr}((PY^* - Y^*)^T(PY^* - Y^*))
\end{equation}

Expanding this expression, using the property of the trace $\operatorname{tr}(AB) = \operatorname{tr}(BA)$, and regrouping the terms, we have:
\begin{equation}
\operatorname{tr}((PY^* - Y^*)^T(PY^* - Y^*)) = \operatorname{tr}((P - I_M)Y^{*T} Y^*P^T + (I_M - P)Y^*Y^{*T})
\end{equation}

Notice that the expression above is minimized when $(P - I_M)Y^{*T} Y^*P^T = 0$ and $(I_M - P)Y^*Y^{*T} = 0$. Both of these conditions are satisfied when $P = I_M$. This is because, when $P = I_M$, we have:

\begin{equation}
(I_M - I_M)Y^{*T} Y^*I_M^T = 0
\end{equation}

and

\begin{equation}
(I_M - I_M)Y^*Y^{*T} = 0
\end{equation}

Thus, by minimizing $\| P - I_M \|_F^2$, we can solve the minimizing problem of $\min \| PY^* - Y^* \|_F^2$.
\end{proof}

\vspace{2ex}

\begin{lemma}
\label{lemma:comparing_normal_random_variables}
Given a set of $M$ random variables $x_1, x_2, ..., x_M \in \mathcal{X}$ that follow normal distributions with potentially different means and variances, let $y_1 \le y_2 \le \cdots \le y_M$ represent the sorted sequence of these random variables. Define $c_{ij} = P(x_i > x_j)$ for all $i, j \in \{1, 2, ..., M\}$ with $i \ne j$. Then, the probability distribution of the $k$-th element, $y_k$, in the sorted sequence can be expressed as:
\begin{equation}
P(y_k = x_i) = \sum_{\mathcal{S},\mathcal{T}} \prod_{j \in \mathcal{S}} (1 - c_{ij}) \prod_{j \in \mathcal{T}} c_{ij}
\end{equation}
where the summation is over all possible partitions of the set $\{1, 2, ..., M\} \setminus {i}$ into two disjoint subsets $\mathcal{S}$ and $\mathcal{T}$ with $|\mathcal{S}| = k - 1$ and $|\mathcal{T}| = M - k$.

\end{lemma}

\begin{proof}
Let's denote the sorted sequence of the independent random variables as $[y_1, y_2, ..., y_M]$, where $y_1 \le y_2 \le \cdots \le y_M$. We want to find the probability distribution of the $k$-th element, $y_k$, in this sorted sequence.

First, note that for any two random variables $x_i$ and $x_j$, the probability that $x_i > x_j$ is given by $c_{ij}$. The complementary probability, that $x_i \le x_j$, is given by $1 - c_{ij}$.

Let's consider the probability that a specific random variable, $x_i$, is the $k$-th element in the sorted sequence. For this to happen, there must be exactly $k - 1$ random variables that are less than or equal to $x_i$ and $M - k$ random variables that are greater than $x_i$.

Since the random variables are independent from each other, the probability of $k - 1$ random variables being less than or equal to $x_i$ can be calculated as the product of probabilities $1 - c_{ij}$ for all $j \ne i$ and $j \in \mathcal{S}$, where $\mathcal{S}$ is a subset of $\{1, 2, ..., M\} \setminus \{i\}$ with $|\mathcal{S}| = k - 1$, where $\setminus$ represents relative complement or set difference. The probability of $M - k$ random variables being greater than $x_i$ can be calculated as the product of probabilities $c_{ij}$ for all $j \ne i$ and $j \in \mathcal{T}$, where $\mathcal{T}$ is a subset of $\{1, 2, ..., M\} \setminus \{i\}$ with $|\mathcal{T}| = M - k$.

Therefore, the probability that $x_i$ is the $k$-th element in the sorted sequence is:
\begin{equation}
P(y_k = x_i) = \sum_{\mathcal{S},\mathcal{T}} \prod_{j \in \mathcal{S}} (1 - c_{ij}) \prod_{j \in \mathcal{T}} c_{ij}
\end{equation}
where the summation is over all possible partitions of the set $\{1, 2, ..., M\} \setminus \{i\}$ into two disjoint subsets $\mathcal{S}$ and $\mathcal{T}$ with $|\mathcal{S}| = k - 1$ and $|\mathcal{T}| = M - k$.
\end{proof}


\section{Supplementary Remarks and Comments}

In this section, we provide some remarks and comments that are useful for readers to understand some claims that are used in Section \ref{sec:ap_additional_theory}.

\begin{remark}
\label{remark:sort_1d_well_defined}
Sorting on 1-D space $\mathbb{R}$ is well-defined.
\end{remark}

\textit{Explanation}: A total order $\leq$ on a set S is a binary relation that satisfies the following properties for all $a, b, c \in \mathbb{R}$:

\begin{itemize}
    \item Reflexivity: $a \leq a$.
    \item Antisymmetry: If $a \leq b$ and $b \leq a$, then $a = b$.
    \item Transitivity: If $a \leq b$ and $b \leq c$, then $a \leq c$.
    \item Totality (or connexity): Either $a \leq b$ or $b \leq a$.
\end{itemize}

In the 1-D space, i.e., on the real numbers $\mathbb{R}$, the usual order $\leq$ is a total order that satisfies these properties. Therefore, for any two real numbers $a$ and $b$, we can unambiguously compare them using the order relation $\leq$. This implies that sorting on 1-D space is well-defined, as the total order allows us to determine the correct order of any pair of real numbers.

Furthermore, the well-ordering property of the integers $\mathbb{Z}$ and the dense nature of the rational numbers $\mathbb{Q}$ within the real numbers $\mathbb{R}$ provide a solid foundation for sorting algorithms on 1-D space, as these properties allow us to efficiently find the desired order of elements and guarantee the existence of a unique sorted sequence.

\vspace{0.5cm}

\begin{remark}
\label{remark:probability_a_larger_than_b}
For two random variables $a \sim \mathcal{N}(\mu_a, \sigma_a^2)$ and $b \sim \mathcal{N}(\mu_b, \sigma_b^2)$, the probability $P(a > b)$ monotonically increases as $\mu_a - \mu_b$ increases, and $P(a > b)$ monotonically decreases as $\sigma_a^2 + \sigma_b^2$ increases.
\end{remark}

\textit{Explanation}: When two random variables $a$ and $b$ are from normal distributions and are independent from each other, the difference between the two variables can be described by another normal distribution. Let $c = a - b$. Since $a$ and $b$ are independent, the mean and variance of $c$ can be computed as follows:
\begin{align}
    \text{Mean of }c:& \mu_c = \mu_a - \mu_b \\
    \text{Variance of }c:& \sigma_c^2 = \sigma_a^2 + \sigma_b^2
\end{align}

Thus, $c$ follows a normal distribution with mean $\mu_c$ and variance $\sigma_c^2$.

Now, to find the probability that $a > b$ is equivalent to finding the probability that $c > 0$. We can compute this using the CDF of the normal distribution.

Let $Z$ be the standard normal variable, such that $Z = (c - \mu_c) / \sigma_c$. The probability we want to find is $P(c > 0)$, which is equivalent to finding $P(Z > -\mu_c / \sigma_c)$, since $Z$ is standardized.

Using the standard normal CDF ($\Phi$), we can find the probability as:
\begin{align}   
P(a > b) & = P(Z > -\mu_c / \sigma_c) \\
         & = 1 - \Phi(-\mu_c / \sigma_c) \\
         & = 1 - \Phi \left[\ - (\mu_a - \mu_b) / \sqrt{(\sigma_a^2 + \sigma_b^2)}\ \right] 
\end{align}

Here, $\Phi$ represents the cumulative distribution function (CDF) of the standard normal distribution. 
\vspace{0.5cm}

\begin{remark}
\label{remark:bijective_dimensionality_reduction}
For two sets $\mathcal{X} \subset \mathbb{R}^N$ and $\mathcal{H} \subset \mathbb{R}^M$ where $N > M$, there may and not always exists a bijective mapping between them. 
\end{remark}

\textit{Explanation}: We can think from the perspective of the cardinality, Lebesgue measure, and intrinsic dimension of the two sets $\mathcal{X}$ and $\mathcal{H}$. 
\begin{itemize}
    \item Finite sets: If both $\mathcal{X}$ and $\mathcal{H}$ are finite, there exists a bijective mapping between them if and only if they have the same cardinality (number of elements). In formal terms, a bijection $f: \mathcal{X} \rightarrow \mathcal{H}$ exists if $|\mathcal{X}| = |\mathcal{H}|$.

    \item Countably infinite sets: If one or both of the sets are countably infinite, there exists a bijective mapping between them if both sets have the same cardinality, which is the cardinality (aleph number) of the set of natural numbers, $\aleph_0$. In this case, a bijection $f: \mathcal{X} \rightarrow \mathcal{H}$ exists if $|\mathcal{X}| = |\mathcal{H}| = \aleph_0$.

    \item Uncountably infinite sets: If one or both of the sets are uncountably infinite, we need to consider the cardinality, Lebesgue measure, and intrinsic dimension of the sets. 
    
    - Cardinality: A bijection between the sets can exist if both sets have the same cardinality. For example, if both sets have the cardinality of the continuum ($\aleph_1$ or $2^{\aleph_0}$), there exists a bijection $f: \mathcal{X} \rightarrow \mathcal{H}$.

    - Lebesgue measure: When comparing sets with different Lebesgue measures, it is more challenging to find a bijection that preserves the local properties of the spaces. For instance, a bijection between a 3D cube with finite volume and a 2D plane with infinite area might not be meaningful in terms of preserving the local properties of the spaces.

    - Intrinsic dimension: If $\mathcal{X}$ is an $N$-D manifold and $\mathcal{H}$ is an $M$-D manifold, with $M < N$, a bijective mapping between them can exist if the intrinsic dimension of the spaces is the same. Specifically, if the intrinsic dimension of $\mathcal{X}$ is equal to the intrinsic dimension of $\mathcal{H}$, it is possible to find a continuous, bijective function $f: \mathcal{X} \rightarrow \mathcal{H}$ that preserves the local properties of the spaces. [Lexicographical sort dimension 1 space-filling curve]
\end{itemize}
In summary, we can see that under certain conditions, there exists a bijective mapping between $\mathcal{X}$ and $\mathcal{H}$.
\vspace{0.5cm}

\begin{remark}
    \label{remark:surjective_neural_network}
    A deterministic neural network can be represented by a surjective mapping.
\end{remark}

\textit{Explanation}: A deterministic neural network on real-space can be expressed by $f: \mathcal{X} \rightarrow \mathcal{Y}$ that maps an input space $\mathcal{X}$ to an output space $\mathcal{Y}$ using a fixed set of weights and biases, and no stochastic components are involved. 

Assuming $\mathcal{Y}$ is the complete set for input set $\mathcal{X}$. For every element $y \in \mathcal{Y}$, there exists at least one element $x \in \mathcal{X}$ such that $f(x) = y$. One the other hand, for each $x \in \mathcal{X}$, there cannot be more than one $y \in \mathcal{Y}$ such that $y = f(x)$, unless there are stocasticity in the model.
\vspace{0.5cm}

\begin{remark}
    \label{remark:bijective_autoencoder}
    If an autoencoder has perfection reconstruction, both the encoder and decoder of it are bijective.
\end{remark}

\textit{Explanation}: A bijective mapping is both injective (one-to-one) and surjective (onto).

Let $E: \mathcal{X} \rightarrow \mathcal{Z}$ be the encoder function and $D: \mathcal{Z} \rightarrow \mathcal{X}$ be the decoder function. The autoencoder achieves perfect reconstruction if for any input $x \in \mathcal{X}$, we have $D(E(x)) = x$.

\begin{itemize}
    \item Injectivity:

(a) Encoder: We want to show that if $x_1, x_2 \in \mathcal{X}$ and $x_1 \neq x_2$, then $E(x_1) \neq E(x_2)$. Suppose, for the sake of contradiction, that $E(x_1) = E(x_2)$. Then, we have:
\begin{align*}
D(E(x_1)) &= D(E(x_2)) \\
x_1 &= x_2
\end{align*}

This contradicts the assumption that $x_1 \neq x_2$. Therefore, $E(x_1) \neq E(x_2)$, and the encoder is injective.

(b) Decoder: We want to show that if $z_1, z_2 \in \mathcal{Z}$ and $z_1 \neq z_2$, then $D(z_1) \neq D(z_2)$. Since the encoder is injective, for $z_1 \neq z_2$, there exist distinct inputs $x_1, x_2 \in \mathcal{X}$ such that $E(x_1) = z_1$ and $E(x_2) = z_2$. Then, we have:
\begin{align*}
D(z_1) &= D(E(x_1)) \\
D(z_2) &= D(E(x_2)) \\
x_1 &\neq x_2
\end{align*}

Hence, $D(z_1) \neq D(z_2)$, and the decoder is injective.

    \item Surjectivity:

(a) Encoder: We want to show that for every point $z \in \mathcal{Z}$, there exists an input $x \in \mathcal{X}$ such that $E(x) = z$. Since the autoencoder achieves perfect reconstruction, for every input $x \in \mathcal{X}$, we have $D(E(x)) = x$. Let $z = E(x)$ for some $x \in \mathcal{X}$. Then, the encoder covers the entire latent space, and it is surjective.

(b) Decoder: We want to show that for every point $x \in \mathcal{X}$, there exists a point $z \in \mathcal{Z}$ such that $D(z) = x$. Since the autoencoder achieves perfect reconstruction, for every input $x \in \mathcal{X}$, we have $D(E(x)) = x$. Let $z = E(x)$ for some $x \in \mathcal{X}$. Then, the decoder covers the entire input space, and it is surjective.
\end{itemize}

In conclusion, if an autoencoder achieves perfect reconstruction, both the encoder and decoder functions are bijective mappings. The perfect reconstruction property ensures that both functions are injective and surjective.
\vspace{3ex}

\begin{remark}
    \label{remark:squeeze_theorem_error_bound}
    In Theorem \ref{th:error_latent_normal_distribution}, when $\| \bm{\sigma_{\hat{\bm{x_i}}}^2} \| \to 0$, we have $\hat{h_i} \to h_i$.
\end{remark}

\textit{Explanation}: As $\| \bm{\sigma_{\hat{\bm{x_i}}}^2} \| \to 0$, we have $\big\Vert \bm{\sigma_{\hat{\bm{x_i}}}} \big\Vert \to 0$. Let $\sigma = \big\Vert \bm{\sigma_{\hat{\bm{x_i}}}} \big\Vert$, we can write:
\begin{align}
\lim_{\sigma \to 0} \frac{\sigma}{K_d} = 0 \\
\lim_{\sigma \to 0} K_e \sigma = 0 \\
\lim_{\sigma \to 0} \frac{1}{K_d^2}\  \sigma^2 = 0 \\
\lim_{\sigma \to 0} 4 K_e^2\  \sigma^2 = 0
\end{align}

Now, using the sandwich theorem (squeeze theorem), we can show that:
\begin{align}
&0 \leq \lim_{\sigma \to 0} \big| \mu_{\hat{h_i}} - h_i \big| \leq 0 \\
&0 \leq \lim_{\sigma \to 0} \sigma_{\hat{h_i}}^2 \leq 0
\end{align}

Since the only value satisfying both inequalities is 0, we can conclude that when $\sigma \to 0$, we have 
\begin{align}
    \lim_{\sigma \to 0}{\mu_{\hat{h_i}}} = h_i\\
    \lim_{\sigma \to 0}{\sigma_{\hat{h_i}}^2} = 0
\end{align}
Thus, the distribution of $\hat{h_i}$ is collapsing to a single value $h_i$.
\vspace{3ex}

\begin{remark}
    The simplified matrix given by Equation (\ref{eq:simplified_p_value_from_imperfect_reconstruction}) and Equation (\ref{eq:simplified_p_matrix}) is a probability matrix.
    \label{remark:simplified_p_is_probability_matrix}
\end{remark}

\textit{Explanation}: To prove that $P$ is a probability matrix, we need to show that:
\begin{enumerate}
    \item Each entry $p_{ij}$ is between 0 and 1.
    \item The sum of the elements in each row and each column is equal to 1.
\end{enumerate}

Denote $x = \gamma_{j, (j-1)}$ and $y = \gamma_{j, (j+1)}$ and we know $x$ and $y$ are both between 0 and 1. It is clear that each entry will be between 0 and 1, satisfying the first condition.

Now, let's consider the sum of the elements in each row. Let We will sum over $i$ for a fixed value of $j$:
\begin{align}
\sum_{i=1}^{M} p_{ij} & = p_{j-1, j} + p_{j, j} + p_{j+1, j} \\
& = xy + x(1 - y) + y(1 - x) + (1 - x)(1 - y) \\
& = x(1 - y) + y(1 - x) + xy + 1 - x - y + xy \\
& = 1 - x - y + 2xy + x + y \\
& = 1
\end{align}

Now, let's check the sum of the elements in each column. We will sum over $j$ for a fixed value of $i$:
\begin{align}
\sum_{j=1}^{M} p_{ij} = p_{i, i-1} + p_{i, i} + p_{i, i+1}
\end{align}

Note that the expression for $p_{ij}$ is symmetric with respect to $i$ and $j$. Thus, the sum over columns is the same as the sum over rows, which we have already shown to be 1.

Since we have proven that each entry $p_{ij}$ is between 0 and 1, and the sum of elements in each row and each column is 1, we can conclude that $P$ is a probability matrix.

\begin{remark}
\label{remark:reconstruction_gaussian}
    We assume the reconstructed $\hat{\bm{x_i}}$ follows a Gaussian distribution, where larger variances indicating worse reconstruction.
\end{remark}

\textit{Explanation}:
This assumption is plausible when employing the Mean Squared Error (MSE) as a measure of reconstruction loss. This Gaussian assumption aligns with the nature of MSE. The MSE is an estimator that measures the average squared differences between estimated and true values, essentially quantifying variance around the mean. Assuming a Gaussian distribution of the reconstructed $\hat{\bm{x_i}}$ aligns with the statistical properties of MSE, as the Gaussian distribution is parametrized by the mean and variance, mirroring the way MSE operates.

Hence, larger variances in the Gaussian distribution of $\hat{\bm{x_i}}$ would indicate a worse reconstruction due to higher dispersion from the mean (true) value, and conversely, smaller variances suggest a better reconstruction. This perspective offers a statistical rationale for evaluating the quality of the reconstruction process.

\end{document}